%% file: main.tex
\documentclass{article}
\newif\ifarxivversion
\arxivversiontrue
\usepackage{iclr2027_conference,times}

\ifarxivversion
  \iclrfinalcopy
\fi

\usepackage{microtype}
\usepackage{graphicx}
\usepackage{subcaption}
\usepackage{booktabs}
\usepackage{longtable}
\usepackage{diagbox}
\usepackage{multirow}
\usepackage{threeparttable}
\usepackage{float}
\usepackage{placeins}
\usepackage{xcolor}
\usepackage{enumerate}
\usepackage{comment}

\usepackage{algorithm}
\usepackage{algorithmic}

\usepackage{amsmath}
\usepackage{amssymb}
\usepackage{amsfonts}
\usepackage{mathtools}
\usepackage{amsthm}
\usepackage{nicefrac}
\usepackage{thmtools}
\usepackage{thm-restate}

\input{math_commands.tex}

\usepackage{hyperref}
\usepackage{url}

\usepackage[capitalize,noabbrev]{cleveref}

\graphicspath{{./}{figures/}}
\def\iclrsysrs{}

\theoremstyle{plain}
\newtheorem{theorem}{Theorem}[section]
\newtheorem{proposition}[theorem]{Proposition}
\newtheorem{lemma}[theorem]{Lemma}
\newtheorem{corollary}[theorem]{Corollary}

\makeatletter

\providecommand*{\toclevel@proposition}{2}
\providecommand*{\toclevel@lemma}{2}
\providecommand*{\toclevel@corollary}{2}
\providecommand*{\toclevel@algorithm}{2}
\makeatother

\theoremstyle{definition}

\theoremstyle{remark}

\DeclareRobustCommand{\edit}[1]{#1}

\title{Efficient Cost-Aware LLM Evaluation via Bayesian Bandit Gittins Indices}

\author{
\setcounter{footnote}{1}
\makebox[\dimexpr\textwidth-2\tabcolsep-3pt\relax][c]{Qian Xie$^{1,2}$\thanks{\hangindent=1.8em\hangafter=1 Corresponding authors: Qian Xie
(qx66@cornell.edu); Nairen Cao (caonairen@sufe.edu.cn).}\qquad
Yueli He$^{3}$\qquad
Nairen Cao$^{3,4}$\footnotemark[2]}\\[10pt]
\normalfont
\makebox[\dimexpr\textwidth-2\tabcolsep-3pt\relax][c]{$^{1}$Cornell University \quad $^{2}$Columbia University}\\
\makebox[\dimexpr\textwidth-2\tabcolsep-3pt\relax][c]{$^{3}$New York University \quad $^{4}$Shanghai University of Finance and Economics}
}

\begin{document}

\maketitle

\ifarxivversion
  \lhead{Preprint.}
\fi

\begin{abstract}
\input{abstract}
\ifarxivversion
\par\smallskip
\noindent\textbf{Code:}
\href{https://github.com/QianJaneXie/BanditGittinsEval}{\texttt{github.com/QianJaneXie/BanditGittinsEval}}
\fi
\end{abstract}

\input{introduction}

\input{problem_setup}

\input{method}

\input{experiment}
\input{conclusion}
\input{statements}

\ifarxivversion
\section*{Acknowledgments}
QX thanks Theodore Brown, Ziv Scully, and Alexander Terenin for a prior collaboration
from which the translational-invariance reduction and FFT-based approach to efficient
Gittins-index computation originated. QX also thanks Kyuseong Choi for introducing QX
to the area of efficient LLM evaluation and sharing related work, Tianyi Peng for
inviting QX to a related project and discussions connecting LLM evaluation with
multi-armed bandits, and Jin Peng Zhou and Ruihan Wu for sharing the BanditEval data
and answering questions about the datasets.
\fi

\bibliographystyle{iclr2027_conference}
\bibliography{BanditGittinsEval}

\appendix

\input{appendices/theory_details_iclr}
\input{appendices/anytime_recommendation}
\input{appendices/dataset_details}
\input{appendices/experiment_setup}
\input{appendices/experiment_results}

\end{document}

%% file: math_commands.tex
\usepackage{amsmath,amsfonts,bm}

\def\eqref#1{equation~\ref{#1}}

\def\1{\bm{1}}

\DeclareMathAlphabet{\mathsfit}{\encodingdefault}{\sfdefault}{m}{sl}
\SetMathAlphabet{\mathsfit}{bold}{\encodingdefault}{\sfdefault}{bx}{n}



%% file: abstract.tex
Exhaustively evaluating every candidate LLM configuration on every benchmark item to identify a high-performing one is costly. We formulate configuration selection as a cost-aware Bayesian bandit problem and propose \emph{GittinsEval}, \edit{inspired by Bayesian-optimal Gittins-index policies for Markov chain selection. We augment this policy with a separate anytime recommendation rule} over both fully and partially evaluated configurations, using an LCB-style score to account for posterior uncertainty. GittinsEval is computationally efficient, requiring only lightweight online updates after offline precomputation. \edit{Across tasks with varying numbers of candidates and examples}, GittinsEval is consistently competitive, with particularly strong gains over the Bayesian optimization baselines on few-candidate large-example benchmarks and over cost-unaware frequentist baselines on large-candidate tasks. Crucially, GittinsEval often attains near-zero simple regret using only \(1\%\)–\(2\%\) of the exhaustive-evaluation cost; it also offers a stopping rule that typically triggers at \(1\%\)–\(10\%\) \edit{with an appropriate cost--quality tradeoff parameter}.

%% file: introduction.tex
\section{Introduction}

Large language models (LLMs) are increasingly used for diverse tasks, including question answering, translation, and mathematical reasoning. \edit{For a given task, practitioners often need to select a single high-performing LLM configuration, such as a combination of model, prompting technique, and temperature, before deployment. A configuration's performance is typically measured by averaging its scores across benchmark examples (e.g., the fraction of questions answered correctly on GSM8K, PIQA, or MMLU~\citep{cobbe2021training,bisk2020piqa,hendrycks2020measuring}), and evaluating it on a subset of examples provides partial information about this average.} Exhaustively evaluating every candidate configuration on every benchmark example can be costly. \edit{However, evaluating only a small fraction of these configuration--example pairs may already suffice to identify a high-performing configuration. Adaptive evaluation seeks to exploit this opportunity by using observations collected so far to allocate further evaluations, potentially identifying a high-performing configuration at a lower evaluation cost. At any point, the evaluator can recommend a configuration based on the available evidence, which the practitioner may adopt as the final selection upon ending evaluation.} This motivates the question: \emph{which configuration should be evaluated next, and how can uncertainty be incorporated to make reliable recommendations?}

Many existing approaches to efficient LLM evaluation~\citep{polo2024efficient,shi2024efficientbestarm,lyu2026cutting}, including pinoeering work BanditEval~\citep{zhou2024speeding}, operate under a prespecified budget measured by the number of evaluations (e.g., $10\%$ of all configuration--example pairs), without explicitly accounting for heterogeneous evaluation costs. \edit{However, the same number of evaluations can incur different cumulative costs depending on which configurations are evaluated. Choosing a suitable evaluation budget in advance can also be difficult, motivating guidance on when to stop. Moreover, a configuration with a high observed score may still have substantial uncertainty if evaluated on only a few examples. Bayesian approaches offer a natural way to account for such uncertainty, while Gittins indices inherently incorporate evaluation costs and provide a principled allocation policy and stopping rule~\citep{gittins2011multi,scully2025gittins}.}

We introduce \textsc{GittinsEval}, a cost-aware Bayesian bandit framework for LLM configuration selection (Figure~\ref{fig:gittins-eval-loop}). \edit{Each configuration is treated as an arm, and evaluations on batches of benchmark examples provide information about its performance.} \textsc{GittinsEval} uses Gittins indices \mbox{\citep{gittins2011multi,scully2025gittins}} \edit{to determine which configuration to evaluate next and to suggest when to stop, accounting for uncertainty, future observations, and evaluation costs}. \edit{However, the classical Gittins-index optimality result relevant to our setting assumes \emph{required completion}, where a configuration must be fully evaluated before terminal recommendation~\citep{dumitriu2003playing,scully2025gittins}.} Since practical LLM evaluation need not require full evaluation, we introduce an \edit{\emph{optional-completion} setting and a separate LCB-style recommendation rule that provides an anytime recommendation over both partially and fully evaluated configurations}. This recommendation rule \edit{does not affect Gittins-based allocation or stopping}.

\begin{figure*}[t]
  \centering
  \includegraphics[width=\textwidth]{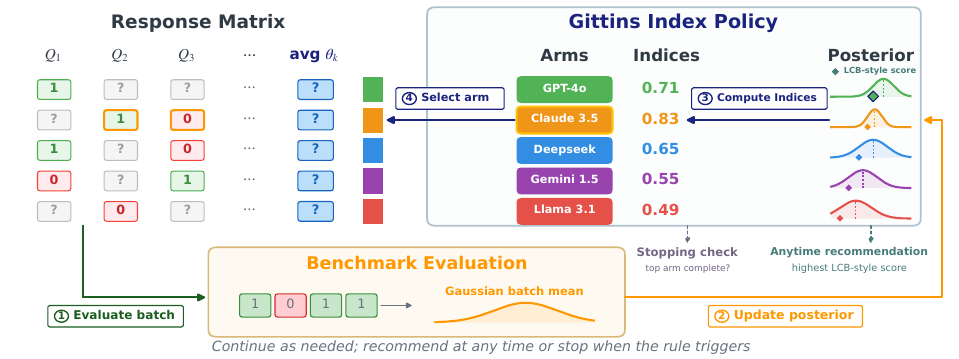}
  \caption{
  Overview of \textsc{GittinsEval}. Each iteration evaluates a fresh batch from the selected configuration, uses a Gaussian approximation to its mean to update the posterior, recomputes the empirical-target Gittins indices, and selects the highest-index unfinished configuration for the next batch. The separate LCB-style recommendation rule yields an anytime recommendation based on the posterior mean minus one posterior standard deviation. The policy can return its current recommendation at any time or, under adaptive stopping, terminate once the highest-index configuration is fully evaluated. Highlighted cells mark the new batch; model names and values are illustrative. }
  \label{fig:gittins-eval-loop}
\end{figure*}

\edit{Practical use of Gittins indices has been limited in part by the cost of dynamic programming}~\citep{scully2025gittins}. PBGI applies Pandora's-box Gittins indices to \edit{inspect-once} Bayesian optimization~\citep{xie2024cost}. \edit{Repeated partial evaluations entail \emph{multi-stage} decisions. We efficiently implement Bayesian bandit Gittins indices for allocation and stopping under the Markov chain selection formulation, with overhead comparable to UCB-E in our experiments} (Appendix~\ref{app:computational-overhead}).

Across GSM8K, PIQA, AlpacaEval, and MMLU, \textsc{GittinsEval} is competitive and often reaches near-zero simple regret at a small fraction of exhaustive-evaluation cost. It is particularly strong against the Bayesian optimization baselines on the few-candidate large-example benchmarks and against cost-unaware frequentist baselines on the large-candidate MMLU tasks. Paired ablations show that our LCB-style rule can reduce low-budget regret without changing allocation or stopping (Figure~\ref{fig:recommendation-rule-ablation}).

Our core contributions are:
\begin{itemize}
    \item We develop \textsc{GittinsEval}, a cost-aware Bayesian framework that supports \edit{\emph{optional completion}} in LLM configuration selection. It \edit{separates Gittins-index allocation and stopping from an LCB-style recommendation rule}, enabling recommendations over both partially and fully evaluated configurations at any stage.
    \item \edit{We provide an efficient computational instantiation of the \emph{multi-stage} Gittins policy for LLM evaluation by deriving how the Gaussian posterior evolves after each batch evaluation and using established index-computation machinery, with algorithmic overhead comparable to UCB-E in our experiments (Appendix~\ref{app:computational-overhead}).}
    \item We demonstrate cost-efficient recommendation across diverse LLM evaluation settings with varying numbers of candidate configurations and benchmark examples. We also investigate key modeling and algorithmic design choices through extensive ablation studies.
\end{itemize}

\subsection{\texorpdfstring{\edit{Related Work}}{Related Work}}
\label{sec:related-work}
\input{related_work}

%% file: related_work.tex
\paragraph{Efficient LLM evaluation.} Recent work has studied how to reduce the cost of evaluating many LLM configurations by adaptively allocating benchmark queries. BanditEval uses UCB-E for best-arm identification and introduces UCB-E-LRF (hereafter LRF), which exploits low-rank structure in the response matrix \citep{zhou2024speeding}. TRIPLE combines successive halving with a predictive model of final arm performance \citep{shi2024efficientbestarm}, while SySRs exploits similarities among model response vectors within a successive-rejection framework \citep{lyu2026cutting}. PromptEval uses configuration covariates to improve multi-prompt evaluation and best-arm identification \citep{polo2024efficient}.
\edit{FormatSpread uses Beta--Bernoulli Thompson sampling to estimate performance ranges across equivalent prompt formats~\citep{sclar2024quantifying}. These approaches generally use prespecified evaluation budgets, whereas \textsc{GittinsEval} additionally provides a cost-aware adaptive stopping signal.}
Complementary approaches include
low-rank prediction for best-model identification \citep{tolochinsky2026valid},
selective human auditing of LLM judges \citep{ao2026bestarm}, and budgeted
multi-agent judging \citep{harrasse2026debate}. \edit{In concurrent work,
\citet{xue2026cost} share our focus on cost-aware LLM configuration evaluation. They study multi-objective bandits, while we consider Bayesian allocation, adaptive stopping, and anytime recommendation under a scalar performance objective.}

\paragraph{Bayesian optimization.}
Bayesian optimization uses surrogate models to guide configuration search, with cost-aware variants accounting for heterogeneous evaluation costs \citep{snoek2012practical}. \edit{Multi-fidelity approaches additionally exploit cheaper, lower-fidelity evaluations \citep{klein2017fast,kandasamy2017multi,wu2020practical,falkner2018bohb}. In our setting, each example provides a noisy observation of a configuration's latent mean; evaluating more examples improves estimation precision, providing a natural notion of fidelity. \textsc{GittinsEval} uses independent per-configuration posteriors to guide cost-aware batch allocation and adaptive stopping.}

\paragraph{Gittins indices and terminology.} Gittins indices are most familiar from the classical Bayesian multi-armed bandit, which maximizes expected discounted cumulative reward over an infinite sequence of arm pulls \citep{gittins1979bandit,gittins2011multi}. Following \citet{scully2025gittins}, we use \emph{Markov chain selection} for the broader structure that unifies Gittins-index applications: at each step, the decision maker chooses one of several independent Markov chains to advance. This name separates the indexable structure from the particular reward objective. Our method uses a terminal-selection objective---benchmark observations are costly measurements, and reward comes from the final recommended configuration---and is closely related to terminal-state problems such as the golf problem of \citet{dumitriu2003playing}. Pandora's-box problems provide a related inspect-once setting that has been used for cost-aware Bayesian optimization and stopping \citep{xie2024cost,xie2025cost}. We retain \emph{Bayesian bandit} as a familiar description of our Bayesian arm-sampling model, while using \emph{Markov chain selection} when referring to the decision framework that establishes the Gittins policy.

%% file: problem_setup.tex
\section{Cost-Aware LLM Configuration Recommendation}
\label{sec:problem}

\subsection{Problem Setup}

We consider $K$ candidate LLM configurations evaluated across a benchmark of $N$ examples. Evaluating configuration $k \in [K]$ on example $j \in [N]$ yields a scalar score $Z_{k,j} \in \mathbb{R}$ ($[0,1]$ in this work), modeled as an observation with latent mean $\theta_k$. The latent mean $\theta_k$ represents configuration $k$'s expected performance on the underlying example distribution and is our primary evaluation target.

An adaptive evaluation policy sequentially queries configurations in batches
(with a single example as the special case $B=1$) and terminates by returning
a recommendation $\hat{k}_T$. Let $A_t$ and $b_t$ denote the configuration and
batch size at allocation step $t$. If $\widetilde c_k$ is the raw per-example
cost of configuration $k$, the effective cost of this batch is
$c_{k,b_t}:=\lambda b_t\widetilde c_k>0$. For batch size $B$,
we abbreviate $c_k:=c_{k,B}$. Our goal is to identify a high-performing
configuration while minimizing cumulative evaluation cost.

When each configuration has a finite evaluation horizon, we call a configuration \emph{completed} once all allowed benchmark observations have been collected. In practical LLM evaluation, however, completion is optional: an unfinished configuration remains eligible for recommendation based on the observations collected so far. We therefore allow an anytime recommendation after any allocation step, whether or not the recommended configuration is completed.

\subsection{Problem Formulations}
We consider two complementary cost-aware formulations. The first targets
\emph{anytime recommendation}: under any externally specified evaluation budget
$C$, the policy returns a recommendation $\hat{k}_{T_C}$, which may be an
unfinished configuration. Let $T_C$ denote the last allocation step satisfying
$\sum_{t=1}^{T_C} c_{A_t,b_t} \le C$. The policy's budget-indexed simple regret is
\[
    R_{\pi}(C)
    :=
    \mathbb{E}_{\pi}\!\left[
        \max_{k \in [K]} \theta_k
        - \theta_{\hat{k}_{T_C}}
    \right].
\]
We seek an anytime policy with low $R_{\pi}(C)$ across evaluation budgets.

The second targets \emph{adaptive stopping}: rather than fixing $C$ in advance,
the evaluator chooses when to stop by trading off recommendation quality against
the cost of further evaluation. We formulate this objective under required
completion. Let $\Pi^{\mathrm{req}}$ denote policies that evaluate only
unfinished configurations and recommend only completed ones. Its expected
net-utility objective is
\begin{equation}
\label{eq:required-completion-objective}
    \sup_{\pi\in\Pi^{\mathrm{req}}}\; \mathbb{E}_{\pi}\!\left[
        \theta_{\hat{k}_T}
        - \sum_{t=1}^T c_{A_t,b_t}
    \right].
\end{equation}
The Gittins allocation and stopping policy underlying \textsc{GittinsEval} is
derived from this adaptive-stopping formulation. For evaluation on a fixed
finite benchmark, Section~\ref{sec:anytime-recommendation} specializes the same
Gittins construction to an observable empirical target and augments it with an
LCB-style recommendation over all configurations to provide the anytime output
in the first formulation. Thus, our method supports both anytime recommendation
and adaptive stopping.

In our experiments, $\widetilde c_k$ is derived from published API prices and
is constant across examples for a given arm. The factor $\lambda>0$ scales
measurement cost to the utility of benchmark performance; a smaller final
batch is automatically charged according to its actual size $b_t$.

%% file: method.tex
\section{Gittins Index Policy for Cost-Aware Bayesian Evaluation}
\label{sec:method}

\edit{We connect the required-completion objective to Markov-chain selection
to obtain a Bayesian Gittins policy for allocation and stopping. Instantiating
this policy requires the posterior transitions induced by batch evaluations;
we derive these transitions and exploit their Gaussian random-walk structure
for efficient index computation. We then specialize the construction to
full-benchmark empirical scores.}
\edit{A separate
LCB-style rule supports optional completion by allowing recommendations of partially
evaluated configurations at any budget while retaining the Gittins allocation
and stopping rules.}
Figure~\ref{fig:gittins-eval-loop} summarizes the overall evaluation loop.

\subsection{Bayesian Latent-Mean Model and Posterior Updates}
\label{sec:gaussian-observation}

\edit{We first specify the Bayesian model whose posterior state will drive the
Gittins policy.}
Throughout, one pull means evaluating one batch. After $n$ local pulls of arm
$k$, let $\mathcal{D}_{k,n}$ denote the batch means collected from that arm.
Suppose the $i$th pull evaluates a fresh batch $\mathcal{B}_{k,i}$ of $B$
benchmark examples. The scalar observation used for posterior updating is
\[
    \overline{Y}_{k,i}
    :=
    \frac{1}{B}\sum_{j\in\mathcal{B}_{k,i}} Z_{k,j}.
\]
Under the working population model, example scores are conditionally i.i.d.
given $\theta_k$, so the batch mean has conditional mean $\theta_k$. For binary
correctness observations, its conditional variance is
$\theta_k(1-\theta_k)/B$. \edit{We model configurations independently, with Gaussian priors on their latent means and a Gaussian approximation to batch means with fixed observation-noise variance $\tau^2$.}
\begin{equation}
    \theta_k \sim \mathcal{N}(\mu_0,v_0),
    \qquad
    \overline{Y}_{k,i}\mid\theta_k
    \sim
    \mathcal{N}(\theta_k,\tau^2).
    \label{eq:gaussian-model}
\end{equation}
Section~\ref{sec:design-choices} describes our choice of $\tau^2$ for binary and continuous benchmark scores.

To state the posterior update compactly, fix a representative arm and suppress
the arm index. After $n$ local pulls,
\[
    \theta\mid\mathcal{D}_n
    \sim
    \mathcal{N}(\mu_n,v_n).
\]
After observing $\overline{Y}_{n+1}$, conjugacy gives
\begin{equation}
\label{eq:gaussian-posterior-update}
\begin{aligned}
    v_{n+1}
    &=
    \left(v_n^{-1}+\tau^{-2}\right)^{-1},
    \\
    \mu_{n+1}
    &=
    \mu_n
    +
    \frac{v_n}{v_n+\tau^2}
    \left(\overline{Y}_{n+1}-\mu_n\right).
\end{aligned}
\end{equation}
Under fixed $\tau^2$, the variance schedule is deterministic, with
$v_n-v_{n+1}=v_n^2/(v_n+\tau^2)$. Thus, arms with the same initial variance
$v_0$ and observation noise $\tau^2$ share the same posterior variance schedule.
The framework can also be extended to accommodate arm-specific prior variances
or noise levels.
Appendix~\ref{app:gaussian-random-walk-details} provides the full derivation of the posterior update and the induced Gaussian random-walk transition.

\subsection{Posterior Dynamics and the Gittins Policy}
\label{sec:gittins-policy}

\edit{We next characterize how this posterior state evolves under future
evaluations, giving the Markov transitions needed for the Gittins policy.}
Let $S_n$ denote the posterior mean of a representative arm after $n$ local pulls, viewed as a random variable over future observations. Conditional on the current realized state $S_n=s$, Eq.~(\ref{eq:gaussian-posterior-update}) implies
\begin{equation}
    S_{n+1}\mid S_n=s
    \sim
    \mathcal{N}
    \left(
        s,
        \frac{v_n^2}{v_n+\tau^2}
    \right).
    \label{eq:posterior-transition}
\end{equation}
Hence the posterior mean follows a Gaussian random walk with a transition variance determined by the local pull count.

The horizon is measured in batch pulls. Fix a finite per-arm horizon $H$, so
that an arm is completed after $H$ batches; for a benchmark with $N$ examples
and batch size $B$, $H:=\lceil N/B\rceil$. Its state is therefore $(S_n,n)$.
Selecting an arm changes only that arm's state and incurs its batch cost. A
smaller final batch, when needed, has a known stage-specific noise variance and
cost and is handled by the same construction.

The Gittins construction reduces the multi-arm allocation problem to a single-arm stopping problem. Intuitively, it asks: \emph{how attractive must the best alternative be before we prefer to stop evaluating the current arm?}

Fix one arm with state $S_n=s$, pull cost $c$, and remaining horizon $H-n$, alongside an outside terminal reward $\alpha$ available from competing arms. For $n < H$, the local decision balances the continuation cost $c$ and future information gain against terminating with outside value $\alpha$. Consequently, the terminal value is $V_H(s;\alpha) := \max\{s, \alpha\}$, and for $n < H$ the value function satisfies the Bellman recursion:
\begin{equation}
    V_n(s;\alpha) := \max\left\{ \alpha, \; -c + \mathbb{E}[V_{n+1}(S_{n+1};\alpha) \mid S_n=s] \right\}.
    \label{eq:bellman-compact}
\end{equation}
For $n < H$, the cost-aware Gittins index $G_n^c(s)$ is the smallest outside value for which stopping is optimal:
\begin{equation}
    G_n^c(s) := \inf\left\{ \alpha \in \mathbb{R} : \alpha \ge -c + \mathbb{E}[V_{n+1}(S_{n+1};\alpha) \mid S_n=s] \right\}.
    \label{eq:gittins-index}
\end{equation}
The index is larger when the arm currently looks promising, when uncertainty creates greater value of information, or when evaluation is cheaper.

At global allocation step $t$, let $n_k(t)$ denote the number of pulls
previously made from arm $k$. For an unfinished arm,
\begin{equation}
    G_{k,t}
    :=
    G_{n_k(t)}^{c_k}
    \left(\mu_{k,n_k(t)}\right).
    \label{eq:arm-gittins-index}
\end{equation}
For a completed arm, $G_{k,t}:=\mu_{k,H}$.
In the required-completion adaptive-stopping formulation, the policy evaluates an unfinished arm with the largest index. If an arm attaining the largest index is already completed, the policy stops and recommends a completed arm with the largest terminal index. GittinsEval uses this allocation and stopping policy while retaining optional-completion anytime recommendations.

Since the arms evolve as independent Markov chains, the classical Markov-chain selection theorem implies that this batch-level Gittins policy is optimal for the latent-mean required-completion objective
\citep{dumitriu2003playing,scully2025gittins}. Appendix~\ref{app:gittins-proof}
states the objective and policy class formally and gives the reduction.

\subsection{Efficient Gittins Index Computation}
\label{sec:gittins-computation}

\edit{The batch-level posterior dynamics derived above connect our LLM
evaluation setting to the abstract Gaussian random-walk model studied in
\citet[Sections~5.3.2--5.3.3]{xie2026gittins}. Building on its computational
machinery, we obtain an efficient Gittins index implementation with offline
precomputation and constant-time online lookup.}

The Gaussian transition in Eq.~(\ref{eq:posterior-transition}) is translation invariant, and its variance schedule is deterministic. Consequently, for fixed $(c,v_0,\tau^2,H)$, the exact single-arm index has the form
\begin{equation}
    G_n^c(s)
    =
    s-r_n(c,v_0,\tau^2,H),
    \label{eq:gittins-root}
\end{equation}
where $r_n$ is a stage-dependent stopping root that does not depend on the current posterior mean $s$.

We compute these roots offline using a discretized dynamic program. Let
$\{\widehat r_n(c,v_0,\tau^2,H)\}_{n=0}^{H-1}$ denote the resulting numerical
approximation to the root schedule. Arms with the same pull cost, prior variance,
observation noise, and horizon share the same table.

The numerical counterpart of the exact single-arm index is obtained by a single
lookup:
\begin{equation}
    \widehat G_n^c(s)
    :=
    s
    -
    \widehat r_n(c,v_0,\tau^2,H).
    \label{eq:gittins-index-lookup}
\end{equation}

The following proposition summarizes the computational guarantee of the FFT
construction adapted from
\citet[Sections~5.3.2--5.3.3 and Appendix~D.2]{xie2026gittins}.

\begin{proposition}[FFT precomputation complexity]
\label{prop:fft-precomputation-complexity}
On a uniform grid of $P$ points, FFT-accelerated Bellman updates compute the $H$-stage numerical root schedule in $\mathcal{O}(HP\log P)$ time and $\mathcal{O}(P)$ working memory, plus $\mathcal{O}(H)$ root storage. Once stored, each index lookup takes $\mathcal{O}(1)$ time, and selecting among $K$ arms takes $\mathcal{O}(K)$ time.
\end{proposition}

The FFT computation replaces the $\mathcal{O}(HP^2)$ cost of explicit summation over the discretized state space. Thus the dynamic program is solved offline; online allocation requires only posterior updates, root lookups, and index comparisons. Appendices~\ref{app:one-dimension-gittins-computation} and~\ref{app:fft-computation} give the root characterization and FFT construction.

\subsection{Empirical-Mean Allocation, Stopping, and Anytime Recommendation}
\label{sec:anytime-recommendation}

For a fixed finite benchmark, the latent population mean $\theta_k$ remains
unobserved even when the response matrix is complete. Our offline
response-matrix experiments therefore use each configuration's full-row
empirical mean as an observable finite-benchmark proxy and specialize the
allocation, stopping, and recommendation rules to this target. We retain the
batch as the decision unit: under the Gaussian surrogate, the posterior mean of
each empirical target again follows a Gaussian random walk at batch boundaries
(with the arm-wise processes forming independent Markov chains).

\paragraph{Full-benchmark empirical target.}
Define the full-benchmark empirical target of arm $k$ as
\begin{equation}
\bar Z_k
:=
\frac{1}{N}\sum_{j=1}^N Z_{k,j}.
\label{eq:finite-target}
\end{equation}
It remains unknown until row $k$ is exhaustively observed.
Let $M_{k,t}$ and $V_{k,t}$ denote the posterior mean and variance,
respectively, of $\bar Z_k$ given the batches observed by
global time $t$.
Appendix~\ref{app:anytime-recommendation} derives this posterior process.

Under the Gaussian surrogate, $(M_{k,t},n_k(t))$ is a Markov state with a
deterministic transition variance schedule; at completion,
$M_{k,t}=\bar Z_k$. Thus the same Markov-chain selection result
applies.

For this finite-benchmark specialization, the required-completion objective is
\begin{equation}
\sup_{\pi\in\Pi^{\mathrm{req}}}\; \mathbb{E}_{\pi}\!\left[
    \bar Z_{\hat{k}_T}
    - \sum_{t=1}^T c_{A_t,b_t}
\right].
\label{eq:finite-required-completion-objective}
\end{equation}

\begin{proposition}[Exact batch-level Gittins optimality]
\label{prop:empirical-target-optimality}
Under the Gaussian surrogate, suppose the arms are independent, example scores
are conditionally i.i.d. within each arm, and the batch sizes, transition
variances, and positive batch costs are predetermined functions of the arm's
local stage. Assign
$G_{k,t}:=M_{k,t}-r_{k,n_k(t)}$ to an unfinished arm and
$G_{k,t}:=\bar Z_k$ to a completed arm. If a completed arm attains
$\max_k G_{k,t}$, stop and recommend a completed arm with largest terminal
index; otherwise, evaluate an unfinished arm with largest index. This exact
empirical-target Gittins policy attains the supremum in
Eq.~(\ref{eq:finite-required-completion-objective}) among required-completion
policies that act at batch boundaries.
\end{proposition}

\edit{This Bayesian optimality guarantee covers both allocation and stopping for the required-completion objective under the Gaussian surrogate.} It is not a distribution-free finite-sample identification guarantee for
the realized response matrix. Appendices~\ref{app:gittins-proof} and~\ref{app:anytime-recommendation} give
the Markov-chain reduction and empirical-target transition calculation,
respectively.

Applying the same construction to the empirical-target chains, our
implementation precomputes gridded stopping roots $\widehat r_{k,n}$ and, for
an unfinished arm, uses
\begin{equation}
\widehat G_{k,t}
:=M_{k,t}-\widehat r_{k,n_k(t)}.
\label{eq:finite-gittins-index}
\end{equation}
The index of a completed arm is $\bar Z_k$. The hat marks the
numerical root approximation. Thus
Proposition~\ref{prop:empirical-target-optimality} applies to the corresponding
exact batch-level roots, not to their gridded approximation, continuation past
the stopping time, or the anytime recommendation below.

The numerical \textsc{GittinsEval} stopping time is the first post-batch time at which a completed arm attains the largest value of $\widehat G_{k,t}$. The required-completion policy stops at this signal and returns a completed arm with the largest terminal index. \edit{Fixed-budget experiments record this inherited stopping signal while continuing evaluation to the prescribed budget, so that anytime regret can be compared across methods over the same budget range.}

\paragraph{Anytime recommendation.}
An anytime recommendation must be available after every batch, even before any
configuration is fully evaluated. We therefore allow optional completion.
\leavevmode\edit{Early in evaluation, a few observations may leave substantial
uncertainty about a configuration's full-benchmark score.
A configuration can lead on an optimistic posterior mean and lose its lead
after another batch. Repeated changes of this kind make early recommendations
unstable. For more reliable recommendations, we maximize an LCB-style score
that subtracts one posterior standard deviation from the posterior mean:}
\begin{equation}
    \hat{k}_t
    \in
    \operatorname*{arg\,max}_{k=1,\ldots,K}
    \left\{
        M_{k,t}-\sqrt{V_{k,t}}
    \right\}.
    \label{eq:finite-lcb-recommendation}
\end{equation}
The penalty vanishes at completion and affects only the recommendation.
Simple regret uses the recommended arm's full-row mean $\bar Z_{\hat{k}_t}$.
\edit{Compared with posterior-mean selection, the LCB-style rule can yield more
stable early recommendations (Appendix~\ref{sec:lcb-motivation}) and smoother
simple-regret curves for \textsc{GittinsEval}-G (Figure~\ref{fig:recommendation-rule-ablation}).
Appendix~\ref{sec:lcb-theory} gives two anytime guarantees.
First, compared with posterior-mean selection, the LCB-style rule recommends
configurations with equal or lower posterior uncertainty at a bounded cost
in posterior expected simple regret. Second, under the Gaussian surrogate,
simple regret is bounded with high probability given the evaluations observed
so far, even before any configuration is fully evaluated.}
Appendix~\ref{app:anytime-recommendation} derives the posterior mean and variance
used in the recommendation rule, and Algorithm~\ref{alg:gittins} summarizes the
evaluation procedure.

\subsection{Design Choices for LLM Benchmarks}
\label{sec:design-choices}

\paragraph{Prior specification.}
For accuracy-style benchmarks, performance lies naturally on the $[0,1]$ scale. Our general prior is $\mathcal{N}(0.5,0.04)$, whose mean expresses no preference between low and high accuracy and whose standard deviation is $0.2$. \edit{To study the value of coarse prior information, GittinsEval-S selects a shared prior from four difficulty buckets using aggregate mean scores from the evaluation response matrices. The G--S comparison assesses sensitivity to empirically informed priors.} All arms in a problem instance share the corresponding prior. \edit{In prospective use, this information would need to come from related tasks or an independent calibration sample.} Table~\ref{tab:prior_settings} gives the prior values and experimental assignments.

\paragraph{Observation noise.}
Under the working Bernoulli population model, a batch of $B$ binary correctness observations has conditional mean variance $\theta_k(1-\theta_k)/B$. Batching can also reduce wall-clock latency when examples are evaluated in parallel or request overhead is amortized, at the cost of less frequent adaptation \citep{zhou2024speeding}. Since $\theta_k$ is unknown, we use the conservative fixed approximation $\tau^2:=1/(4B)$, corresponding to the worst-case Bernoulli variance. This approximation may overestimate the noise for very easy or very hard tasks, but fixing $\tau^2$ makes the posterior variance schedule deterministic and allows the Gittins root schedule to be precomputed. More generally, any random variable supported on $[0,1]$ has variance at most $1/4$. We therefore use the same conservative bound for the continuous pairwise preference scores in AlpacaEval: the per-comparison working variance is $\tau_{\mathrm{cell}}^2:=1/4$, and, treating comparisons within a batch as conditionally independent in the working model, the batch-mean variance is $\tau^2:=\tau_{\mathrm{cell}}^2/B=1/(4B)$. This is a fixed working approximation rather than an estimate of the empirical AlpacaEval score variance.

\paragraph{Cost scaling.}
As in Section~\ref{sec:problem}, $\widetilde c_k$ is a raw per-example
API-price proxy and a full batch has effective cost
$c_k:=c_{k,B}=\lambda B\widetilde c_k$. \edit{Our model-level cost proxies omit variation in token usage and latency across individual evaluations.} \edit{The scale $\lambda$ controls the
cost--quality tradeoff and hence the stopping time: lower values encourage
further evaluation, while higher values favor earlier stopping.} We study
several cost-scaling factors in addition to the raw
inference-cost ratios.

%% file: experiment.tex
\section{Experiments}
\label{sec:experiments}

\paragraph{Benchmarks.}
We study two groups of benchmarks using precomputed response matrices. The large-example group uses BanditEval matrices~\citep{zhou2024speeding} for GSM8K~\citep{cobbe2021training}, PIQA~\citep{bisk2020piqa}, and AlpacaEval~\citep{li2023alpacaeval}. GSM8K and PIQA have $122$ and $103$ configurations, respectively, with $1000$ examples each. The AlpacaEval~2.0 matrix has $152$ leaderboard models and $805$ instructions, with continuous pairwise preference scores in $[0,1]$. The other benchmarks use binary correctness.

The large-candidate group uses model responses from DOVE~\citep{habba2025dove} on all $57$ subjects of MMLU~\citep{hendrycks2020measuring}, a multiple-choice benchmark spanning diverse disciplines. Each subject has $1500$ candidate configurations from $15$ models and $100$ prompt templates. \edit{Subjects are grouped by example count (small, medium, or large) and empirical mean accuracy across configurations (easy, medium, or hard).} Appendix~\ref{app:dataset-detail} provides dataset details.

\paragraph{Baselines.}
We compare the two \textsc{GittinsEval} variants with \edit{prior LLM-evaluation baselines}---UCB-E~\citep{audibert2010best}, BanditEval's low-rank-factorization variant LRF~\citep{zhou2024speeding}, SySRs~\citep{lyu2026cutting}, \edit{FormatSpread-TS~\citep{sclar2024quantifying}}, and PromptEval~\citep{polo2024efficient}---and implementations of established Bayesian optimization algorithms adapted to our setting: configuration-level Bayesian optimization with PBGI acquisition and its associated stopping rule~\citep{xie2024cost,xie2025cost}, and \edit{Bayesian Optimization and Hyperband (BOHB)~\citep{falkner2018bohb}.} We include PBGI to compare our adaptive stopping rule with one from Bayesian optimization. We also ran log expected improvement (LogEI) under unit costs and its cost-aware counterpart (LogEIPC) under heterogeneous costs. Their regret trends were similar to PBGI's, so we omit their curves for space. \edit{Since FormatSpread targets performance ranges across equivalent prompt formats, we adapt its best-format search to configuration selection as FormatSpread-TS, retaining Beta--Bernoulli Thompson sampling.} \edit{For PromptEval, we adapt the authors' best-prompt identification approach by treating configurations as arms. On GSM8K and PIQA, where BanditEval varies models and decoding settings with two prompting strategies, we use one-hot configuration features (PE-OneHot). On MMLU, PE-EmbPT uses interactions between one-hot model indicators and prompt-template embeddings.} \textsc{GittinsEval-G} uses the general prior; \textsc{GittinsEval-S} uses \edit{an empirically informed prior shared across configurations} (Table~\ref{tab:prior_settings}). UCB-E, LRF, SySRs, and PromptEval are cost-unaware; initialization details are given in Appendix~\ref{app:experiment-setup}. \edit{In the appendix, we also evaluate Gaussian-TS as an allocation ablation, replacing Gittins-index allocation with Thompson sampling in the same Gaussian Bayesian framework.}

\paragraph{Evaluation metric.}
The primary metric is simple regret versus evaluation resources. After step $t$, each policy recommends an arm $\hat{k}_t$. Since the latent means $\theta_k$ are not observed even in the completed response matrices, we use the full-benchmark empirical targets in Eq.~(\ref{eq:finite-target}) as observable proxies and compute $\mathrm{SR}_t:=\max_k\bar Z_k-\bar Z_{\hat{k}_t}$. We plot regret against cumulative evaluation cost $C_t:=\sum_{s=1}^t b_s\tilde c_{A_s}$, where $b_s$ is the batch size and $\tilde c_k$ is the raw per-example cost; the unit-cost setting is the special case $\tilde c_k\equiv1$. The \textsc{GittinsEval} variants use the LCB-style rule in Eq.~(\ref{eq:finite-lcb-recommendation}). The uncertainty penalty affects the recommendation, while reported regret is computed from the empirical targets above. \edit{FormatSpread-TS recommends the evaluated configuration with the highest empirical mean.} Each baseline uses its own recommendation rule.

\begin{figure}[t]
\centering
\includegraphics[width=1\columnwidth]{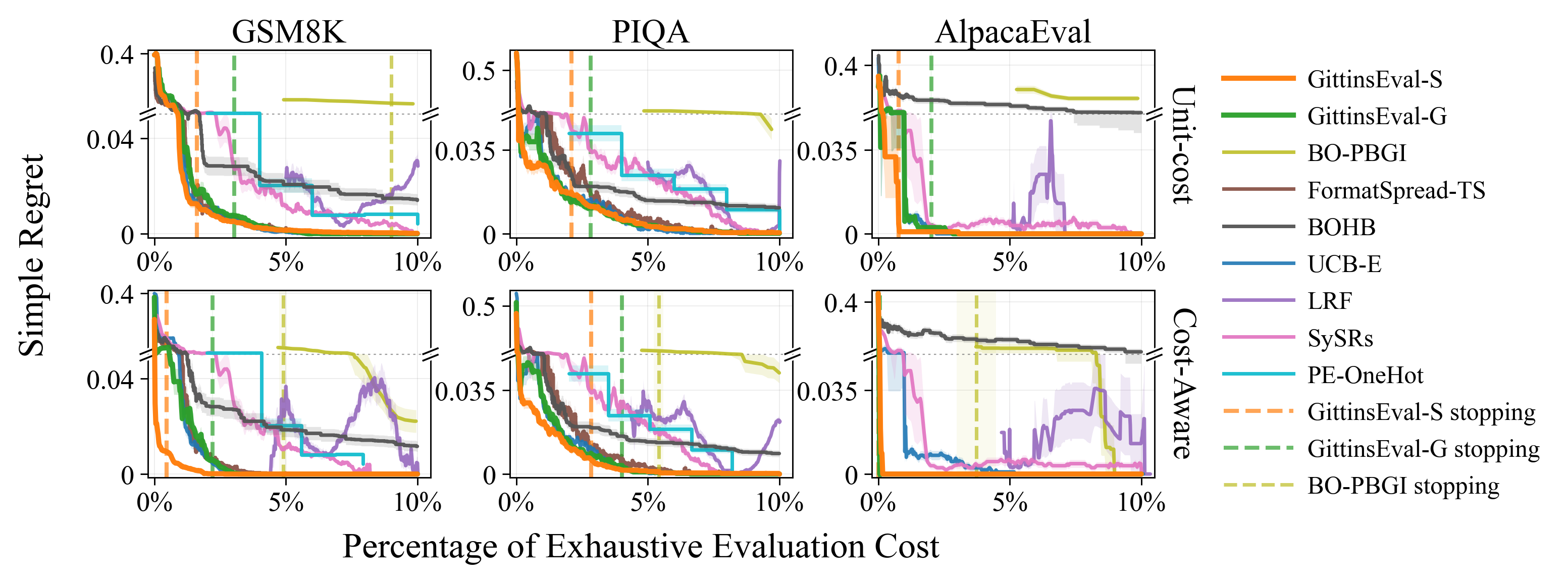}
\caption{Simple regret on BanditEval under unit costs (top) and heterogeneous costs (bottom). \edit{\textsc{GittinsEval-S} generally reaches low regret at lower cost than the baselines, with a clear advantage over Bayesian optimization baselines PBGI and BOHB. Under unit costs, FormatSpread-TS performs similarly to \textsc{GittinsEval-S} on GSM8K and PIQA. Curves are averaged over random seeds for all benchmarks and over independently generated response matrices for GSM8K and PIQA.} Shading denotes $\pm1$ standard error, and dashed lines mark stopping times.}
\label{fig:gsm8k-piqa-main}
\end{figure}

\paragraph{Experimental conditions.}
For each benchmark, we consider both unit-cost and cost-aware settings. Each evaluation has unit cost in the former and incurs a model-specific cost derived from API-price proxies in the latter, requiring the policy to account for heterogeneous costs when allocating evaluations. Configurations that use the same base model share the same model-level price, computed from published input/output rates using fixed benchmark-level ratios described in Appendix~\ref{app:experiment-setup}; we do not model \edit{token usage for individual evaluations}.

Here, exhaustive evaluation means evaluating every configuration on every benchmark item. All methods receive a nominal budget equal to 10\% of exhaustive evaluation, measured in response-matrix entries under unit costs and total entry cost under heterogeneous costs. UCB-E and LRF run to this limit following the BanditEval protocol of \citet{zhou2024speeding}. LRF's uniform random warm-up targets $5\%$ of the response-matrix entries. For the \textsc{GittinsEval} variants, we record the stopping time from Section~\ref{sec:anytime-recommendation} while continuing each trajectory to the same limit. PBGI evaluates each selected configuration once on the full benchmark. After random initialization capped at 5\% of configurations, it selects configurations by acquisition within the remaining budget. \edit{BOHB uses the fraction of benchmark examples as fidelity, charging only for newly observed examples at their unit or model-specific cost.}

\edit{Sensitivity studies on these benchmarks inform our batch sizes and GittinsEval's $\lambda=10^{-4}$ (Appendix~\ref{sec:ablation-studies}).} GittinsEval and UCB-E use $B=8$ on GSM8K, PIQA, and AlpacaEval, and $B=2,4,8$ on small ($N\leq150$), medium ($151\leq N\leq400$), and large ($N>400$) MMLU subjects. \edit{FormatSpread-TS shares these nominal batch sizes on GSM8K, PIQA, and MMLU. LRF uses $B=32$ to reduce low-rank update overhead.}

\paragraph{Reporting.}
Figures~\ref{fig:gsm8k-piqa-main} and~\ref{fig:mmlu-aggregate} report the main simple-regret curves. PromptEval and FormatSpread-TS are included for GSM8K, PIQA, and MMLU, but not AlpacaEval because their binary-response models do not directly apply to its continuous pairwise preference scores. Shaded bands denote one standard error across runs in Figure~\ref{fig:gsm8k-piqa-main} and across equally weighted subject means in Figure~\ref{fig:mmlu-aggregate}. Dashed vertical lines mark average stopping times for \textsc{GittinsEval} and PBGI, recorded without truncating their budgeted trajectories. \edit{The vertical axis uses two linear scales to magnify differences near zero, with the transition marked. PromptEval's five phase-boundary regret estimates are connected by step curves.} Appendix~\ref{app:anytime-recommendation} describes the LCB-style recommendation rule, while Appendix~\ref{app:additional-experiment-results} presents ablations and per-subject MMLU results grouped by difficulty.

\begin{figure*}[t]
\centering
\includegraphics[width=1\textwidth]{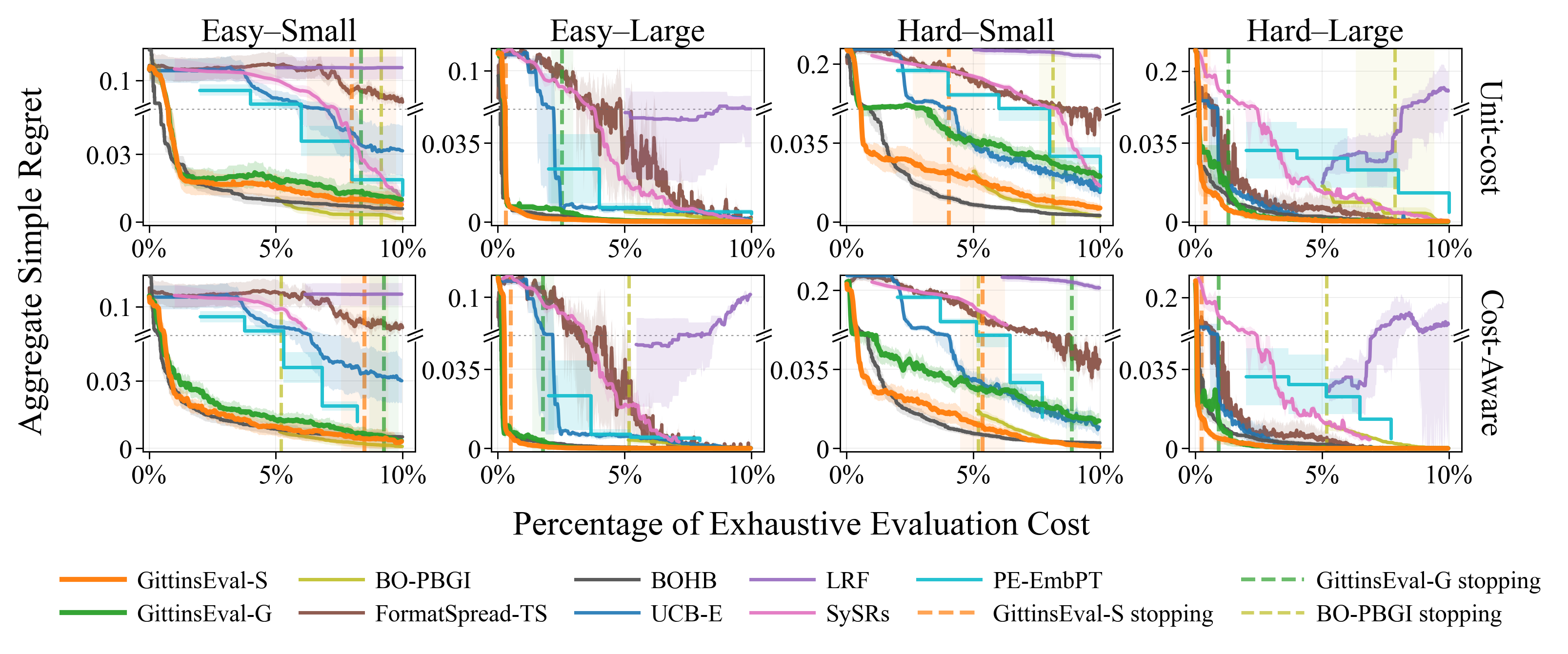}
\caption{Aggregate raw MMLU simple regret by number of examples (small/large) and difficulty (easy/hard) under unit-cost (top) and cost-aware (bottom) evaluation. \edit{\textsc{GittinsEval-S} substantially outperforms the frequentist baselines and FormatSpread-TS, particularly at low evaluation budgets, although PBGI and BOHB sometimes achieve lower regret on subjects with fewer examples.} Curves first average 20 runs per subject and then weight subjects equally within each size--difficulty bucket; shading denotes $\pm1$ standard error, and dashed lines mark stopping times.}
\label{fig:mmlu-aggregate}
\end{figure*}

\paragraph{Main findings.}
On the large-example BanditEval benchmarks, which have fewer candidates, \textsc{GittinsEval} generally matches or outperforms the frequentist baselines at low budgets and reaches low regret at substantially lower cost than \edit{Bayesian optimization baselines PBGI and BOHB}.

On the large-candidate MMLU benchmarks, \textsc{GittinsEval} has a more pronounced low-budget advantage over the frequentist baselines and \edit{the Bayesian bandit baseline FormatSpread-TS}, though it \edit{sometimes incurs slightly higher regret than Bayesian optimization methods, particularly on subjects with fewer examples}.

\edit{Ablations show that LCB-style recommendation, which combines posterior means with an uncertainty penalty, can substantially reduce early-budget regret for both GittinsEval and FormatSpread-TS (Figures~\ref{fig:recommendation-rule-ablation}--\ref{fig:recommendation-benchmarks-g}).} \textsc{GittinsEval-S} generally achieves lower regret than \textsc{GittinsEval-G}. \edit{The benefit of data-specific priors is clearest on hard MMLU subjects.}

\textsc{GittinsEval}'s \edit{adaptive stopping rule further reduces evaluation cost, typically stopping after the regret curves flatten at $1\text{--}10\%$ of exhaustive-evaluation cost with our empirically chosen $\lambda=10^{-4}$. Among the baselines, only PBGI provides an adaptive stopping rule, but it generally incurs higher evaluation cost before stopping. Like most baselines, \textsc{GittinsEval} is computationally efficient, with substantially lower overhead than LRF (Appendix~\ref{app:computational-overhead}).}

%% file: conclusion.tex
\section{Conclusion}
We introduced \textsc{GittinsEval}, \edit{a cost-aware Bayesian framework that efficiently instantiates the Gittins policy for sequential LLM evaluation.} Across benchmarks, \textsc{GittinsEval} is consistently competitive, with anytime recommendations often attaining near-zero simple regret at low evaluation cost. An additional adaptive stopping rule helps avoid unnecessary evaluation. \edit{Future work includes principled prior calibration, selection of $\lambda$, better use of cross-arm correlations, and modeling costs that vary across individual evaluations.} \edit{Beyond LLM evaluation, our framework can be applied to more general candidate-selection problems with noisy evaluations, allowing recommendations from partial observations to balance selection quality against evaluation cost.}

%% file: statements.tex
\section*{AI Use Statement}
The authors used generative-AI tools, including LLM-based assistants, to retrieve
and discover potentially relevant literature, support research ideation and
execution, draft and revise portions of the manuscript, and improve the clarity
of the writing. The tools also assisted in proving mathematical claims by
working out detailed derivations and proof steps from sketches and arguments
provided by the authors. The authors checked the cited sources, mathematical
derivations, and AI-assisted text, revised them as needed, and take full
responsibility for the paper's contents.

Separately, LLMs are the objects of evaluation in this work. The experiments use
precomputed response matrices released by BanditEval for GSM8K, PIQA, and
AlpacaEval, and by DOVE for MMLU. We did not query model APIs to generate new
evaluation responses for this study; the provenance and processing of the
experimental data are described in Appendix~\ref{app:dataset-detail}.

\section*{Reproducibility Statement}
The main text specifies the probabilistic model, allocation rule, numerical
index computation, stopping rule, evaluation metrics, and experimental
conditions. Appendix~\ref{app:theoretical-details} states the theoretical
assumptions and provides the derivations and proofs; Appendix~\ref{app:anytime-recommendation}
derives the finite-benchmark recommendation and index construction and gives
pseudocode. Appendix~\ref{app:dataset-detail} documents the data sources,
processing, and dataset statistics, while Appendix~\ref{app:experiment-setup}
reports the computing environment, repetitions, priors, batch sizes, cost
construction, baseline setup, and aggregation procedures. Additional results
and sensitivity analyses appear in Appendix~\ref{app:additional-experiment-results}.
All experiments operate on the precomputed response data rather than requiring
new LLM API calls.
\ifarxivversion
\else\fi

%% file: appendices/theory_details_iclr.tex
\section{Theoretical Details}
\label{app:theoretical-details}

\subsection{Gaussian Random-Walk Details}
\label{app:gaussian-random-walk-details}

We derive the posterior update and the induced random-walk transition used in Section~\ref{sec:method}. Fix an arm and suppress the arm index. Let $\overline{Y}_{n+1}$ denote the batch mean observed on the next pull. Suppose that after $n$ local pulls the current posterior distribution of the arm mean is
\[
\theta\mid \mathcal{D}_n
\sim
\mathcal{N}(\mu_n,v_n),
\]
and that the next observation follows the fixed-noise Gaussian likelihood
\[
\overline{Y}_{n+1}\mid \theta
\sim
\mathcal{N}(\theta,\tau^2).
\]

By Bayes' rule, the posterior density after observing $\overline{Y}_{n+1}$ is proportional to the likelihood times the current posterior:
\[
p(\theta\mid \mathcal{D}_{n+1})
\propto
p(\overline{Y}_{n+1}\mid \theta)\,
p(\theta\mid \mathcal{D}_n).
\]
Substituting the Gaussian likelihood and the current Gaussian posterior gives
\[
p(\theta\mid \mathcal{D}_{n+1})
\propto
\exp\left(
-\frac{1}{2}
\left[
\frac{(\theta-\mu_n)^2}{v_n}
+
\frac{(\overline{Y}_{n+1}-\theta)^2}{\tau^2}
\right]
\right).
\]
Expanding the terms that depend on $\theta$,
\[
\frac{(\theta-\mu_n)^2}{v_n}
+
\frac{(\overline{Y}_{n+1}-\theta)^2}{\tau^2}
=
\left(v_n^{-1}+\tau^{-2}\right)\theta^2
-
2\left(
\frac{\mu_n}{v_n}
+
\frac{\overline{Y}_{n+1}}{\tau^2}
\right)\theta
+
\text{const.}
\]
Thus the exponent is quadratic in $\theta$, so the posterior is Gaussian. Matching the quadratic and linear coefficients with the Gaussian form
\[
\exp\left(
-\frac{(\theta-\mu_{n+1})^2}{2v_{n+1}}
\right)
\]
gives
\[
v_{n+1}^{-1}
=
v_n^{-1}+\tau^{-2},
\qquad
\frac{\mu_{n+1}}{v_{n+1}}
=
\frac{\mu_n}{v_n}
+
\frac{\overline{Y}_{n+1}}{\tau^2}.
\]
Therefore,
\[
v_{n+1}
=
\left(v_n^{-1}+\tau^{-2}\right)^{-1},
\qquad
\mu_{n+1}
=
v_{n+1}
\left(
\frac{\mu_n}{v_n}
+
\frac{\overline{Y}_{n+1}}{\tau^2}
\right).
\]
Equivalently, the posterior mean update can be written in the incremental form
\[
\mu_{n+1}
=
\mu_n
+
\frac{v_n}{v_n+\tau^2}
\left(\overline{Y}_{n+1}-\mu_n\right).
\]
This form shows that the new posterior mean moves from the old posterior mean toward the new observation, with gain
\[
\frac{v_n}{v_n+\tau^2}.
\]

To describe the posterior mean as a state process, we first compute the predictive distribution of the next observation. Conditioned on the current data $\mathcal{D}_n$, we may write
\[
\overline{Y}_{n+1}
=
\theta+\varepsilon_{n+1},
\qquad
\theta\mid\mathcal{D}_n\sim\mathcal{N}(\mu_n,v_n),
\qquad
\varepsilon_{n+1}\sim\mathcal{N}(0,\tau^2),
\]
with $\theta$ and $\varepsilon_{n+1}$ conditionally independent given $\mathcal{D}_n$. Hence
\[
\overline{Y}_{n+1}\mid\mathcal{D}_n
\sim
\mathcal{N}(\mu_n,v_n+\tau^2).
\]

We now view the posterior mean as the Markov state used below. Let $S_n$ denote the posterior mean after $n$ local pulls, viewed as a random variable over future observations. If the current realized state is $S_n=s$, then the posterior update gives
\[
S_{n+1}
=
s
+
\frac{v_n}{v_n+\tau^2}
\left(\overline{Y}_{n+1}-s\right).
\]
Since
\[
\overline{Y}_{n+1}-s
\mid S_n=s
\sim
\mathcal{N}(0,v_n+\tau^2),
\]
we have
\[
S_{n+1}\mid S_n=s
\sim
\mathcal{N}\!\left(
s,
\left(\frac{v_n}{v_n+\tau^2}\right)^2
(v_n+\tau^2)
\right).
\]
Therefore,
\[
S_{n+1}\mid S_n=s
\sim
\mathcal{N}\!\left(
s,
\frac{v_n^2}{v_n+\tau^2}
\right).
\]

Finally, under the fixed-noise likelihood, the variance update
\[
v_{n+1}
=
\left(v_n^{-1}+\tau^{-2}\right)^{-1}
\]
does not depend on the realized observation $\overline{Y}_{n+1}$. Therefore the posterior variance follows a deterministic schedule indexed only by the local pull count $n$, while the posterior mean follows the Gaussian random-walk transition above whenever the arm is selected. Across arms, this schedule is identical only for arms with the same initial variance and fixed noise level; otherwise each arm has its own deterministic schedule.

\subsection{\texorpdfstring{\edit{Proof of Bayesian Optimality under Required Completion}}{Proof of Bayesian Optimality under Required Completion}}
\label{app:mdp-formulation}
\label{app:gittins-proof}

\begin{proposition}[Latent-mean optimality under required completion]
\label{thm:gittins-optimality-restated}
Consider the Gaussian Bayesian bandit model in Section~\ref{sec:gaussian-observation} with independent arms, common prior $\mathcal{N}(\mu_0,v_0)$, fixed observation noise $\tau^2$, finite per-arm horizon $H$, and nonnegative effective costs of arm pulls $c_k$ for arms $k=1,\ldots,K$. Let $\Pi_H^{\mathrm{req}}$ be the class of policies that evaluate only unfinished arms and may terminate only by recommending an arm that has reached its finite horizon. For the required-completion objective
\[
\sup_{\pi\in\Pi_H^{\mathrm{req}}}\;
\mathbb{E}_{\pi}\!\left[
\theta_{\hat{k}_T}
-
\sum_{t=1}^{T} c_{A_t}
\right],
\]
an optimal policy is to assign terminal index $G_{k,t}=\mu_{k,H}$ to completed arms, assign $G_{k,t}=G_{n_k(t)}^{c_k}(\mu_{k,n_k(t)})$ to unfinished arms, and stop when an arm attaining $\max_k G_{k,t}$ is completed. Otherwise, the policy evaluates an unfinished arm with largest Gittins index $G_{k,t}$. At stopping, an optimal recommendation is any completed arm with largest posterior mean among completed arms, equivalently any completed arm with largest terminal index.
\end{proposition}

\ifdefined\iclrsysrs
Proposition~\ref{prop:empirical-target-optimality} applies the same Markov-chain selection result to the full-benchmark empirical target at batch boundaries. Numerical root approximation, continuation beyond the stopping time, and the LCB-style anytime recommendation are not part of either optimality claim.
\fi

\begin{proof}
We prove the proposition by reducing the required-completion problem to a Markov chain selection problem. For each arm $k$, define one local chain:
\begin{description}
\item[State.]
The local state is $x_k=(s_k,n_k)$, where $s_k$ is the posterior mean and $n_k\in\{0,\ldots,H\}$ is the number of observed batches. The terminal states are those with $n_k=H$.

\item[Action.]
At each decision time, choose one arm $k$. Only the chosen arm moves; every other arm keeps its current state.

\item[Reward.]
If the chosen arm is unfinished, the immediate reward is $-c_k$. If the chosen arm is terminal, the process stops and the terminal reward is $s_k$.

\item[Transition.]
If the chosen arm is in state $(s_k,n_k)$ with $n_k<H$, then its next state is $(S',n_k+1)$, where
\[
S'\mid s_k,n_k
\sim
\mathcal{N}\!\left(s_k,\frac{v_{n_k}^2}{v_{n_k}+\tau^2}\right).
\]
\end{description}
The global state is the collection of local states, initialized at $((\mu_0,0),\ldots,(\mu_0,0))$.

By Eq.~(\ref{eq:gaussian-posterior-update}), the state of arm $k$ can be represented by $(\mu_{k,n_k},n_k)$. Under fixed $\tau^2$, Appendix~\ref{app:gaussian-random-walk-details} shows that the posterior mean follows the Gaussian random-walk transition, so this state is Markov. Conditional on this state, the next state distribution of arm $k$ is independent of the states and histories of all other arms, and pulling arm $k$ incurs only its own effective cost $c_k$. Since $s_k=\mathbb{E}[\theta_k\mid x_k]$, the terminal reward is the posterior expected reward from recommending completed arm $k$. Thus maximizing expected total reward in the constructed selection problem is exactly the required-completion objective.

The Gittins-index optimality theorem for Markov chain selection implies that an optimal policy selects a chain with largest index until a terminal chain has largest index \citep{dumitriu2003playing,scully2025gittins}. Applying this result to the single-arm value function used to define $G_n^c(s)$ gives the stated rule. Terminal chains have index equal to their posterior mean, so when a completed arm attains the largest index, recommending any completed arm with largest posterior mean among completed arms gives an optimal terminal reward.
\end{proof}

The same reduction allows predetermined stage-specific batch sizes. The local
stage $n$ then determines both the next batch-mean variance and its cost, so
including $n$ in the state preserves the Markov property and arm independence.

\subsection{One-Dimensional Gittins Index Computation}
\label{app:one-dimension-gittins-computation}

We now spell out the reduction used by the precomputation step in Algorithm~\ref{alg:gittins}; the corresponding translational-equivariance argument also appears in \citep[Section~5.3.2]{xie2026gittins}. Fix a single arm, suppress the arm index, and let $\alpha$ be the outside terminal reward available from the other arms. With the cost $c$ fixed and suppressed in the value-function notation, the terminal value is
\[
V_H(s;\alpha)=\max\{s,\alpha\},
\]
and, for $n<H$,
\[
\begin{aligned}
Q_n^{\mathrm{cont}}(s;\alpha)
&=
-c+\mathbb{E}\!\left[V_{n+1}(S_{n+1};\alpha)\mid S_n=s\right],
&\qquad
Q_n^{\mathrm{stop}}(s;\alpha)
&=\alpha,
\\
V_n(s;\alpha)
&=
\max\{Q_n^{\mathrm{cont}}(s;\alpha),Q_n^{\mathrm{stop}}(s;\alpha)\}.
\end{aligned}
\]
The transition kernel depends only on differences in posterior means:
\[
S_{n+1}\mid S_n=s\sim \mathcal{N}(s,\sigma_n^2),
\qquad
\sigma_n^2=\frac{v_n^2}{v_n+\tau^2}.
\]

\begin{lemma}[Translational invariance of the continuation value]
\label{lem:translational-invariance}
Fix the cost $c$ and Gaussian random-walk transition. For any shift $\beta\in\mathbb{R}$ and any stage $n<H$,
\[
Q_n^{\mathrm{cont}}(s+\beta;\alpha+\beta)
=
Q_n^{\mathrm{cont}}(s;\alpha)+\beta.
\]
\end{lemma}

\begin{proof}[Proof of Lemma~\ref{lem:translational-invariance}]
The terminal value satisfies
\[
V_H(s+\beta;\alpha+\beta)
=
\max\{s+\beta,\alpha+\beta\}
=
V_H(s;\alpha)+\beta.
\]
Assume the corresponding shift identity holds for the next-stage value. Since the transition kernel is Gaussian with additive mean, the next state from $s+\beta$ has the same distribution as $S_{n+1}+\beta$ when the current state is $s$. Therefore
\[
\begin{aligned}
Q_n^{\mathrm{cont}}(s+\beta;\alpha+\beta)
&=
-c+
\mathbb{E}\!\left[
V_{n+1}(S_{n+1}+\beta;\alpha+\beta)\mid S_n=s
\right]
\\
&=
-c+
\mathbb{E}\!\left[
V_{n+1}(S_{n+1};\alpha)+\beta\mid S_n=s
\right]
\\
&=
Q_n^{\mathrm{cont}}(s;\alpha)+\beta.
\end{aligned}
\]
Taking the maximum with the shifted outside reward gives the next-stage value identity needed for the induction. The result follows by backward induction.
\end{proof}

\begin{corollary}[Centered continuation value]
\label{cor:centered-reduction}
Let $x:=s-\alpha$ and define $q_n(x):=Q_n^{\mathrm{cont}}(x;0)$. For any $n<H$,
\[
Q_n^{\mathrm{cont}}(s;\alpha)
=
q_n(s-\alpha)+\alpha.
\]
Thus evaluating the unfinished arm is worthwhile exactly when $q_n(s-\alpha)>0$.
\end{corollary}

\begin{proof}
Apply Lemma~\ref{lem:translational-invariance} to the centered state $x=s-\alpha$, outside reward $0$, and shift $\beta=\alpha$. Then
\[
Q_n^{\mathrm{cont}}(s;\alpha)
=
Q_n^{\mathrm{cont}}(s-\alpha;0)+\alpha
=
q_n(s-\alpha)+\alpha.
\]
Continuing is preferable to the outside reward precisely when $Q_n^{\mathrm{cont}}(s;\alpha)>\alpha$, which is equivalent to $q_n(s-\alpha)>0$.
\end{proof}

Define the centered value
\[
W_n(x):=V_n(\alpha+x;\alpha)-\alpha.
\]
Then the recursion no longer depends on $\alpha$:
\begin{equation}
\label{eq:centered-dp}
\begin{aligned}
W_H(x)
&=
\max\{x,0\},
\\
W_n(x)
&=
\max\{0,q_n(x)\},
\\
q_n(x)
&=
-c+\mathbb{E}\!\left[W_{n+1}(X_{n+1})\mid X_n=x\right],
\end{aligned}
\end{equation}
where $X_{n+1}\mid X_n=x\sim\mathcal{N}(x,\sigma_n^2)$ and the recursion for $W_n$ applies for $n<H$. This is the one-dimensional dynamic program: for each stage $n$, the numerical grid is only over the centered posterior-mean advantage $x$.

For comparison with an outside reward $\alpha$, evaluating an unfinished arm is worthwhile exactly when $q_n(s-\alpha)>0$. Let $r_n$ denote the crossing point satisfying
\[
q_n(r_n)=0.
\]
The Gittins index for cost $c$ is therefore
\begin{equation}
\label{eq:gittins-root-index}
G_n^c(s)=s-r_n.
\end{equation}
The online policy does not need to store the full two-argument value function. During precomputation we represent the one-dimensional functions $W_n$ and $q_n$ on a grid in order to propagate the recursion backward. The exact policy is characterized by the roots $\{r_n\}_{n=0}^{H-1}$, while the numerical policy stores their gridded approximations $\{\widehat r_n\}_{n=0}^{H-1}$ defined below.

\subsection{FFT Implementation}
\label{app:fft-computation}

For completeness, we summarize the offline FFT computation used in our implementation. The piecewise-linear expected-improvement construction and its Gaussian expectation are derived in \citep[Section~5.3.3 and Appendix~D.2]{xie2026gittins}.

Fix $n<H$, let $Z_n\sim\mathcal{N}(0,\sigma_n^2)$ with $\sigma_n^2:=v_n^2/(v_n+\tau^2)$, and use a uniform grid $\xi_j:=\xi_0+j\delta$, $j=0,\ldots,P-1$. For $w_j:=W_{n+1}(\xi_j)$, set
\[
m_0:=0,\qquad
m_i:=\frac{w_i-w_{i-1}}{\delta}\ (1\leq i<P),\qquad
m_P:=1,\qquad d_i:=m_{i+1}-m_i.
\]
The boundary-extended piecewise-linear surrogate is
\[
\widehat W_{n+1}(x):=w_0+\sum_{i=0}^{P-1}d_i(x-\xi_i)_+.
\]
Writing $\operatorname{EI}_{\sigma}(z):=\mathbb{E}[(z+Z)_+]=z\Phi(z/\sigma)+\sigma\phi(z/\sigma)$ for $Z\sim\mathcal{N}(0,\sigma^2)$ gives
\[
\mathbb{E}[\widehat W_{n+1}(\xi_j+Z_n)]
=w_0+\sum_{i=0}^{P-1}d_i\operatorname{EI}_{\sigma_n}((j-i)\delta).
\]

The sum is a linear convolution. Store $\boldsymbol d=(d_0,\ldots,d_{P-1})$ and the length-$(2P-1)$ kernel
\[
e^{(n)}_\ell:=\operatorname{EI}_{\sigma_n}((\ell-(P-1))\delta),
\qquad \ell=0,\ldots,2P-2.
\]
If $\boldsymbol h^{(n)}:=\boldsymbol d*\boldsymbol e^{(n)}$ is their full convolution, the required $P$ values are the \emph{valid} slice $h^{(n)}_{P-1},\ldots,h^{(n)}_{2P-2}$. The full convolution has length $P+(2P-1)-1=3P-2$, so zero-padding both inputs to an FFT length $L\geq3P-2$ prevents circular wrap-around; for the experimental grid $P=1025$, we use $L=4096$.

The backward update and numerical crossing are
\[
q_n(\xi_j):=-c+w_0+h^{(n)}_{j+P-1},\qquad
W_n(\xi_j):=\max\{0,q_n(\xi_j)\},\qquad
\widehat r_n:=\min\{\xi_j:q_n(\xi_j)\geq0\}.
\]
Each stage takes $\mathcal{O}(P\log P)$ time and $\mathcal{O}(P)$ working memory, giving $\mathcal{O}(HP\log P)$ time over $H$ stages rather than $\mathcal{O}(HP^2)$ for explicit summation. In exact arithmetic the FFT reproduces the discrete convolution; approximation relative to the continuous dynamic program remains from the grid, boundary extension, root discretization, and floating-point arithmetic.

%% file: appendices/anytime_recommendation.tex
\section{Anytime Recommendation for Full-Benchmark Empirical Means}
\label{app:anytime-recommendation}

\subsection{Posterior, allocation, and stopping}
\label{sec:finite-target-implementation}

\paragraph{Posterior of the full empirical mean.}
For the finite-benchmark specialization in Section~\ref{sec:anytime-recommendation}, fix an arm and suppress its index. Of the $N$ available scores, let $N_t$ have been observed with sum $R_t$, leaving $m_t=N-N_t$ unobserved. Conditional on the latent mean $\theta$, the remaining scores are independent under our Gaussian surrogate, with mean $\theta$ and variance $\tau_{\mathrm{cell}}^2$. Let $\mathcal H_t$ denote all observations collected through global allocation step $t$. Given the latent posterior $\theta\mid\mathcal H_t\sim\mathcal N(\mu_t,v_t)$,
\[
\bar Z:=\frac{R_t+\sum_{j\,\mathrm{unobserved}}Z_j}{N}.
\]
Restoring the arm index, let $N_{k,t}$ be the observed-example count,
$m_{k,t}:=N-N_{k,t}$ the remaining-example count, and $n_k(t)$ the number of
posterior updates (batches). Taking conditional expectations and applying the
law of total variance give
\begin{equation}
\label{eq:finite-posterior-moments}
\begin{aligned}
M_{k,t}
&:=\mathbb E[\bar Z_k\mid\mathcal H_t]
=\frac{R_{k,t}+m_{k,t}\mu_{k,n_k(t)}}{N},\\
V_{k,t}
&:=\operatorname{Var}(\bar Z_k\mid\mathcal H_t)
=\frac{m_{k,t}^2v_{k,n_k(t)}+m_{k,t}\tau_{\mathrm{cell}}^2}{N^2}.
\end{aligned}
\end{equation}
For full batches, $\tau_{\mathrm{cell}}^2=B\tau^2$. Both terms are needed: uncertainty about unobserved example scores remains even when the latent mean is accurately estimated. At completion, the posterior collapses to the known full-row mean. Before any observations, $M_0=\mu_0$ and $V_0=v_0+\tau_{\mathrm{cell}}^2/N$. These moments follow from the Gaussian approximation, including when the original benchmark scores are binary or bounded.

\paragraph{Empirical-target indices and stopping.}
For fixed $N$, prior $(\mu_0,v_0)$, and per-example noise, the posterior mean of the full-benchmark empirical target is an affine transformation of the latent posterior mean:
\[
M_t=a\mu_t+(1-a)\mu_0,
\qquad
a:=1+\frac{\tau_{\mathrm{cell}}^2}{Nv_0}.
\]
Let $H:=\lceil N/B\rceil$ and let
$b_n:=\min\{B,N-nB\}$ be the predetermined size of local batch $n+1$. Its
noise variance is $\tau_{\mathrm{cell}}^2/b_n$, and
\[
M_{n+1}\mid M_n
\sim\mathcal N\!\left(M_n,
a^2\frac{v_n^2}{v_n+\tau_{\mathrm{cell}}^2/b_n}\right).
\]
This deterministic schedule of transition variances lets us reuse the centered
recursion in Eq.~(\ref{eq:centered-dp}) with state $M_n$. Charging batch cost
$c_{k,n}:=\lambda b_n\widetilde c_k$ gives the batch-level root schedule
$\{r_{k,n}\}_{n=0}^{H-1}$. Thus an unfinished arm has exact
index
\[
G_{k,t}
:=M_{k,t}-r_{k,n_k(t)},
\]
whereas a completed arm has terminal index $\bar Z_k$. The
transition variance here is the uncertainty resolved by a new batch,
not the total remaining variance $V_t$ used in the recommendation penalty.

\begin{proof}[Proof of Proposition~\ref{prop:empirical-target-optimality}]
At batch boundaries, $(M_{k,t},n_k(t))$ is an arm-wise Markov state, and the
chains are independent under the Gaussian surrogate. The local stage determines
the next batch size, transition variance, and positive batch cost. Completion
reveals the terminal value $M_{k,t}=\bar Z_k$. Substituting these
chains into the Markov-chain selection reduction in
Appendix~\ref{app:gittins-proof} therefore gives the stated Gittins policy and
required-completion optimality among policies that act at batch boundaries.
\end{proof}

The implementation replaces the exact roots by gridded approximations
$\widehat r_{k,n}$, with each transition using the corresponding
batch size, including a partial final batch. Budget checks occur after each
batch, so the final batch can cross the nominal budget. Completed arms have
numerical index equal to their known empirical mean. The stopping
rule records the first post-batch crossing at which a completed arm attains the
largest numerical index. The optional-completion evaluation may end before
this signal because of a budget or external interruption, use the signal as an
adaptive endpoint, or continue sampling past it; in every case, the LCB-style
anytime recommendation remains available. Numerical discretization,
continuation past the signal, and the recommendation rule are outside the
batch-level optimality claim.

\begin{algorithm}[H]
\caption{Optional-completion Gittins evaluation for full-benchmark empirical scores}
\label{alg:gittins}
\begin{algorithmic}[1]
\REQUIRE Prior $(\mu_0,v_0)$, raw per-example costs $\{\widetilde c_k\}_{k=1}^K$, noise $\tau_{\mathrm{cell}}^2$, $N$ examples, batch size $B$, cost scale $\lambda$, and an evaluation-end condition (for example, a budget, external interruption, or the inherited stopping signal).
\STATE Set $H=\lceil N/B\rceil$ and $b_n=\min\{B,N-nB\}$; precompute gridded empirical-target roots $\{\widehat r_{k,n}\}_{n=0}^{H-1}$ using batch noise $\tau_{\mathrm{cell}}^2/b_n$ and batch cost $\lambda b_n\widetilde c_k$.
\STATE Initialize latent moments $(\mu_{k,0},v_{k,0})=(\mu_0,v_0)$, batch counts $n_k(0)=0$, observed counts $N_{k,0}=0$, sums $R_{k,0}=0$, and $t=0$.
\STATE Compute $(M_{k,0},V_{k,0})$ by Eq.~(\ref{eq:finite-posterior-moments}); set $\hat{k}_0$ by Eq.~(\ref{eq:finite-lcb-recommendation}).
\STATE Compute $\widehat G_{k,0}=M_{k,0}-\widehat r_{k,0}$ for every arm.
\WHILE{an unfinished arm exists and evaluation has not ended}
    \STATE Select an unfinished arm $A_t$ with largest $\widehat G_{k,t}$.
    \STATE Evaluate $b=\min\{B,N-N_{A_t,t}\}$ fresh examples and observe their batch mean $\overline Y$; charge $\lambda b\widetilde c_{A_t}$ to the budget.
    \STATE Let $n=n_{A_t}(t)$; update $(\mu_{A_t,n+1},v_{A_t,n+1})$ from $(\mu_{A_t,n},v_{A_t,n})$ by Eq.~(\ref{eq:gaussian-posterior-update}) with batch noise $\tau_{\mathrm{cell}}^2/b$.
    \STATE For every arm $k$, set $R_{k,t+1}:=R_{k,t}+\mathbf{1}\{k=A_t\}b\overline Y$ and $N_{k,t+1}:=N_{k,t}+\mathbf{1}\{k=A_t\}b$.
    \STATE Set $n_k(t+1):=n_k(t)+\mathbf{1}\{k=A_t\}$ for every arm $k$, then set $t\leftarrow t+1$.
    \STATE Compute $(M_{k,t},V_{k,t})$ by Eq.~(\ref{eq:finite-posterior-moments}); set and report $\hat{k}_t$ by Eq.~(\ref{eq:finite-lcb-recommendation}).
    \STATE Set $\widehat G_{k,t}=M_{k,t}$ for completed arms and $\widehat G_{k,t}=M_{k,t}-\widehat r_{k,n_k(t)}$ otherwise.
    \IF{a completed arm attains $\max_k \widehat G_{k,t}$}
        \STATE Record the inherited Gittins stopping signal if this is the first crossing.
    \ENDIF
\ENDWHILE
\STATE \textbf{return} $\hat{k}_t$.
\end{algorithmic}
\end{algorithm}

\subsection{Why use an LCB-style recommendation rule?}
\label{sec:lcb-motivation}
The required-completion adaptive-stopping formulation excludes unfinished arms from its terminal output, but the anytime formulation must be able to recommend such an arm at any time. A mean-only recommendation treats a high estimate based on very few examples as favorably as the same estimate supported by many examples. Under a general prior that is optimistic for a benchmark, an unobserved or sparsely observed arm can repeatedly become the recommended arm, lose its lead when evaluated, and be replaced by another uncertain arm. The score $M_{k,t}-\sqrt{V_{k,t}}$ discounts this uncertainty without preventing later recommendation once enough evidence accumulates. Changing the target from the latent mean to the full-benchmark empirical mean alone does not resolve this effect: when all arms share $N$, prior, and noise, the affine transformation above preserves their mean-only ranking on the same observed data.

The coefficient is fixed to one for the reported LCB-style recommendation rule. It affects which arm is recommended, not which arm is sampled or the stopping time. The penalty is a stabilization choice rather than a confidence guarantee. In particular,
\[
\mathbb E\left[\max_j\bar Z_j-\bar Z_k\mid\mathcal H_t\right]
=\mathbb E\left[\max_j\bar Z_j\mid\mathcal H_t\right]-M_{k,t},
\]
\edit{Thus posterior-mean recommendation minimizes posterior expected simple regret under the assumed model, while the LCB-style rule trades posterior expected performance for lower uncertainty. Its effect on realized regret varies across settings.}

\paragraph{Recommendation stability and simple regret.}
We count how often the recommended configuration changes from one batch to the
next, using the same Gittins evaluation trajectory for both rules under each
prior. For space, we illustrate this difference with one GSM8K run, where the
effect on recommendation stability was particularly clear. We use a response matrix
with 122 configurations and 1,000 examples (matrix seed 1),
sampling seed 0, batches of 16, unit costs, and Gittins cost scale $10^{-4}$.
Under the general prior $\mathcal N(0.5,0.04)$, the LCB-style recommendation rule reduces the
number of recommendation changes from 139 to 38 and budget-weighted mean simple regret from
0.04824 to 0.02903. Under the data-specific prior $\mathcal N(0.2,0.01)$,
recommendation changes fall from 49 to 37 while budget-weighted mean simple regret changes
from 0.03128 to 0.03177. All four recommendations attain zero regret at the
budget endpoint. We count changes through the common budget of 12,200 evaluations.

Figure~\ref{fig:recommendation-rule-ablation} in Appendix~\ref{sec:ablation-studies}
shows the simple-regret comparison across runs under heterogeneous costs.
\edit{Within limited space, we selected GSM8K and MMLU Hard-Large because the
LCB-style recommendation visibly reduces simple regret for GittinsEval-G in
these settings. These panels illustrate this effect; they do not establish a
uniform benefit across benchmarks.} Both rules use the same allocation trajectory in each run.

\edit{
\subsection{Theoretical justification for LCB-style recommendation}
\label{sec:lcb-theory}

Our LCB-style recommendation rule is designed to reduce
posterior uncertainty about the recommended configuration under
partial evaluation.
Posterior-mean recommendation optimizes posterior expected
performance, but does not distinguish between configurations
with similar estimated performance and different levels of
uncertainty.
The rule accounts for this uncertainty without changing
Gittins-based allocation or stopping.

We establish two complementary guarantees.
First, relative to posterior-mean recommendation, our rule selects
a configuration with no greater posterior variance, and its
additional posterior expected simple regret is bounded by
the reduction in posterior standard deviation.
Second, under the Gaussian surrogate, we bound the rule's
simple regret with high posterior probability in terms of
the uncertainties of the optimal and recommended configurations.
Both guarantees apply before any configuration is fully evaluated.

\paragraph{Setup.}
Let $\mathcal H_t$ denote the sigma-field generated by the
evaluation history through batch $t$, and let $\overline Z_k$
be the full-benchmark empirical score of configuration $k$.
Write
\[
M_{k,t}
:=
\mathbb E[\overline Z_k\mid\mathcal H_t],
\qquad
V_{k,t}
:=
\operatorname{Var}(\overline Z_k\mid\mathcal H_t),
\qquad
s_{k,t}:=\sqrt{V_{k,t}}.
\]
On the same evaluation history, we compare our LCB-style recommendation rule
(with a unit penalty coefficient) and posterior-mean recommendation:
\begin{align}
\widehat k_t^{\mathrm L}
&\in
\arg\max_{k\in[K]}
\{M_{k,t}-s_{k,t}\},
\label{eq:lcb-theory-rule-l}\\
\widehat k_t^{\mathrm M}
&\in
\arg\max_{k\in[K]}M_{k,t}.
\label{eq:lcb-theory-rule-m}
\end{align}
Ties may be resolved by any $\mathcal H_t$-measurable rule.
For $a\in\{\mathrm L,\mathrm M\}$, abbreviate
\[
m_t^a:=M_{\widehat k_t^a,t},
\qquad
s_t^a:=s_{\widehat k_t^a,t}.
\]

Let
\[
\overline Z_*:=\max_{k\in[K]}\overline Z_k,
\qquad
\operatorname{SR}_t(j):=\overline Z_*-\overline Z_j
\]
denote the optimal full-benchmark score and the simple regret
of recommending configuration $j$, respectively.
Its posterior expected simple regret is
\[
\mathcal R_t(j)
:=
\mathbb E[
\operatorname{SR}_t(j)
\mid\mathcal H_t
].
\]
These definitions also apply when $j$ is an
$\mathcal H_t$-measurable recommendation.

\begin{theorem}[Recommendation uncertainty and regret tradeoff]
\label{thm:lcb-uncertainty-tradeoff}
Suppose the posterior second moments are finite.
At each reporting time,
\begin{equation}
s_t^{\mathrm L}\le s_t^{\mathrm M},
\label{eq:lcb-uncertainty-order}
\end{equation}
and consequently
\begin{equation}
\operatorname{Var}
(\overline Z_{\widehat k_t^{\mathrm L}}\mid\mathcal H_t)
\le
\operatorname{Var}
(\overline Z_{\widehat k_t^{\mathrm M}}\mid\mathcal H_t).
\label{eq:lcb-variance-order}
\end{equation}
Moreover,
\begin{equation}
\begin{aligned}
0
\le
\mathcal R_t(\widehat k_t^{\mathrm L})
-
\mathcal R_t(\widehat k_t^{\mathrm M})=
m_t^{\mathrm M}-m_t^{\mathrm L}\le
s_t^{\mathrm M}-s_t^{\mathrm L}.
\end{aligned}
\label{eq:lcb-regret-price}
\end{equation}
Thus, the LCB-style recommendation rule selects a configuration with no greater posterior
uncertainty, while its additional posterior expected simple
regret is at most the reduction in posterior standard deviation.
\end{theorem}

\begin{proof}
Posterior-mean optimality gives
\[
m_t^{\mathrm M}\ge m_t^{\mathrm L}.
\]
By the definition of $\widehat k_t^{\mathrm L}$,
\[
m_t^{\mathrm L}-s_t^{\mathrm L}
\ge
m_t^{\mathrm M}-s_t^{\mathrm M}.
\]
Combining these inequalities yields
\begin{equation}
0
\le
m_t^{\mathrm M}-m_t^{\mathrm L}
\le
s_t^{\mathrm M}-s_t^{\mathrm L}.
\label{eq:lcb-proof-tradeoff}
\end{equation}
In particular, $s_t^{\mathrm L}\le s_t^{\mathrm M}$.
Since the recommendations are $\mathcal H_t$-measurable,
their conditional variances are $(s_t^{\mathrm L})^2$
and $(s_t^{\mathrm M})^2$, respectively.
Squaring the nonnegative standard deviations therefore
proves \eqref{eq:lcb-variance-order}.

For any $\mathcal H_t$-measurable recommendation $j$,
\[
\mathcal R_t(j)
=
\mathbb E[\overline Z_*\mid\mathcal H_t]-M_{j,t}.
\]
Consequently,
\[
\mathcal R_t(\widehat k_t^{\mathrm L})
-
\mathcal R_t(\widehat k_t^{\mathrm M})
=
m_t^{\mathrm M}-m_t^{\mathrm L}.
\]
Substituting this identity into
\eqref{eq:lcb-proof-tradeoff}
proves \eqref{eq:lcb-regret-price}.
\end{proof}

Theorem~\ref{thm:lcb-uncertainty-tradeoff}
quantifies an uncertainty--regret tradeoff rather than
asserting an improvement in posterior expected simple regret.
Indeed, posterior-mean recommendation minimizes
$\mathcal R_t(j)$, whereas our rule minimizes
uncertainty-penalized posterior expected simple regret:
\[
\widehat k_t^{\mathrm L}
\in
\arg\min_{j\in[K]}
\{\mathcal R_t(j)+s_{j,t}\}.
\]
The first theorem does not require Gaussian posteriors.

\paragraph{Simple-regret guarantee.}
Under the Gaussian surrogate, we can additionally bound
the simple regret of the LCB-style recommendation with
high posterior probability.
The confidence parameter below is used only in the analysis:
the recommendation rule retains coefficient one.

\begin{theorem}[Posterior simple-regret bound for the LCB-style recommendation rule]
\label{thm:lcb-simple-regret}
Assume the Gaussian surrogate is correctly specified, so that
\[
\overline Z_k\mid\mathcal H_t
\sim
\mathcal N(M_{k,t},V_{k,t})
\qquad\text{for every }k\in[K].
\]
Let
$k^*\in\arg\max_{k\in[K]}\overline Z_k$
be an optimal configuration.
For $\delta\in(0,1)$, define
\[
b_\delta
:=
\sqrt{2\log\!\left(\frac{2K}{\delta}\right)}.
\]
Then, with conditional probability at least $1-\delta$
given $\mathcal H_t$,
\begin{equation}
\begin{aligned}
\operatorname{SR}_t(\widehat k_t^{\mathrm L})
\le{}
(b_\delta+1)s_{k^*,t}+
(b_\delta-1)s_{\widehat k_t^{\mathrm L},t}.
\end{aligned}
\label{eq:lcb-simple-regret-refined}
\end{equation}
In particular, writing
$V_{\max,t}:=\max_{k\in[K]}V_{k,t}$, we have
\begin{equation}
\Pr\!\left(
\operatorname{SR}_t(\widehat k_t^{\mathrm L})
\le
2\sqrt{
2V_{\max,t}\log\!\left(\frac{2K}{\delta}\right)
}
\;\middle|\;
\mathcal H_t
\right)
\ge
1-\delta.
\label{eq:lcb-simple-regret-uniform}
\end{equation}

The same statements hold with $t$ replaced by any bounded
stopping time $T$, provided the stated Gaussian posterior
laws hold conditional on $\mathcal H_T$.
No configuration is required to be fully evaluated.
\end{theorem}

\begin{proof}
Consider the simultaneous posterior event
\[
\mathcal E_t
:=
\bigcap_{k=1}^{K}
\left\{
|\overline Z_k-M_{k,t}|
\le
b_\delta s_{k,t}
\right\}.
\]
For an arm with $s_{k,t}>0$, the Gaussian tail bound gives
\[
\Pr\!\left(
|\overline Z_k-M_{k,t}|
>
b_\delta s_{k,t}
\;\middle|\;
\mathcal H_t
\right)
\le
2e^{-b_\delta^2/2}.
\]
If $s_{k,t}=0$, the corresponding interval contains
$\overline Z_k$ with conditional probability one.
A union bound therefore yields
\[
\Pr(\mathcal E_t\mid\mathcal H_t)
\ge
1-2K e^{-b_\delta^2/2}
=
1-\delta.
\]
Independence between arms is not needed for this argument.

On $\mathcal E_t$, abbreviating
$\widehat k=\widehat k_t^{\mathrm L}$, we have
\begin{align*}
\operatorname{SR}_t(\widehat k)
&=
\overline Z_{k^*}-\overline Z_{\widehat k}\\
&\le
M_{k^*,t}-M_{\widehat k,t}
+
b_\delta(s_{k^*,t}+s_{\widehat k,t}).
\end{align*}
Although $k^*$ depends on the unknown full-benchmark scores,
the preceding inequality is valid because $\mathcal E_t$
holds simultaneously for every configuration.

By the definition of $\widehat k_t^{\mathrm L}$,
\[
M_{\widehat k,t}-s_{\widehat k,t}
\ge
M_{k^*,t}-s_{k^*,t},
\]
and hence
\[
M_{k^*,t}-M_{\widehat k,t}
\le
s_{k^*,t}-s_{\widehat k,t}.
\]
Combining these inequalities gives
\[
\operatorname{SR}_t(\widehat k)
\le
(b_\delta+1)s_{k^*,t}
+
(b_\delta-1)s_{\widehat k,t},
\]
which proves \eqref{eq:lcb-simple-regret-refined}.

Since $K\ge1$ and $\delta\in(0,1)$, we have $b_\delta>1$.
Both coefficients in the preceding bound are therefore
nonnegative.
Using
\[
s_{k^*,t},\ s_{\widehat k,t}
\le
\sqrt{V_{\max,t}},
\]
we obtain
\[
\operatorname{SR}_t(\widehat k)
\le
2b_\delta\sqrt{V_{\max,t}},
\]
proving \eqref{eq:lcb-simple-regret-uniform}.

For a bounded stopping time $T$, repeat the same argument
conditional on $\mathcal H_T$, using the assumed Gaussian
posterior laws at that reporting time.
\end{proof}

\paragraph{Interpretation and scope.}
Together, the two theorems justify our LCB-style recommendation rule for
conservative recommendation under partial evaluation. The rule
reduces posterior uncertainty relative to posterior-mean
recommendation at a bounded cost in posterior expected simple
regret and admits an uncertainty-dependent simple-regret bound.
These are recommendation-layer guarantees and apply along
the Gittins evaluation trajectory without requiring exact
index computation or invoking Gittins optimality.

The variance comparison concerns posterior uncertainty about
the recommended configuration's full-benchmark score;
it does not imply uniformly smaller realized regret variance
or fewer recommendation switches.
The probability guarantee is Bayesian under the Gaussian
surrogate, rather than distribution-free for an arbitrary
fixed response matrix.
It applies at a chosen, possibly adaptive reporting time,
not simultaneously over all reporting times.
}

%% file: appendices/dataset_details.tex
\section{Dataset Details}
\label{app:dataset-detail}

\subsection{Dataset Description}
\label{app:dataset-description}

\paragraph{GSM8K and PIQA.}
For GSM8K and PIQA, we use the response matrices released with BanditEval~\citep{zhou2024speeding}. The underlying benchmarks are Grade School Math 8K (GSM8K)~\citep{cobbe2021training} and Physical Interaction: Question Answering (PIQA)~\citep{bisk2020piqa}. GSM8K consists of grade-school math word problems, while PIQA evaluates physical commonsense reasoning through question answering.

The BanditEval response matrices are constructed from a collection of publicly available language models and sampling configurations. In particular, BanditEval considers $11$ models: GPT2, GPT2-Large, CodeLLaMA, Tulu-7B, Tulu-2-7B, Gemma-7B, Phi2, Llema-7B, LLaMA-2-7B, Mistral-7B, and StarCoder-7B. For each model, responses are generated using three temperature choices $\{0, 0.5, 1\}$, two maximum decoding lengths $\{128, 512\}$, and two zero-shot prompting strategies: directly asking for the answer, and using the chain-of-thought prompt ``Let's think step by step.'' The Cartesian product of these choices yields $11 \times 3 \times 2 \times 2 = 132$ possible model--configuration arms.

However, the released response matrices contain some missing configurations. As a result, the final matrices used in our experiments contain $122$ arms for GSM8K and $103$ arms for PIQA. Both GSM8K and PIQA contain $1{,}000$ examples in the response matrices. Each benchmark is represented by five independently generated response matrices, corresponding to five random seeds used when querying the LLMs. Since the benchmark examples and arm set are fixed, the differences across these matrices mainly reflect randomness in LLM response generation.

\paragraph{AlpacaEval.}
We use the AlpacaEval response matrix released with BanditEval~\citep{zhou2024speeding} to evaluate instruction-following quality~\citep{li2023alpacaeval}. Unlike the BanditEval GSM8K and PIQA model$\times$prompt configuration matrices, this matrix is derived from the AlpacaEval~2.0 leaderboard comparisons annotated by \texttt{weighted\_alpaca\_eval\_gpt4\_turbo}. Each row corresponds to one AlpacaEval leaderboard model, and each column corresponds to one of the $805$ fixed AlpacaEval instructions. Entry $(k,n)$ stores the model's pairwise score against the GPT-4 Turbo baseline on instruction~$n$. These scores are continuous values in $[0,1]$ derived from the auto-annotator's preference probabilities; they are \emph{not} binary correctness labels. Accordingly, the full-benchmark empirical target $\bar Z_k$ in Eq.~(\ref{eq:finite-target}) is the row mean of these pairwise scores.

We exclude the annotator row \texttt{gpt4\_1106\_preview\_verbose} and the degenerate \texttt{gpt4\_1106\_preview} row, whose $805$ comparison values are all $0.5$. The resulting matrix contains $152$ leaderboard-model arms and $805$ instruction examples. It keeps the AlpacaEval pairwise scores as continuous values rather than converting them to binary outcomes. Each arm is therefore a single leaderboard model rather than a model$\times$sampling-configuration pair. In contrast to GSM8K and PIQA, we use one fixed AlpacaEval matrix rather than five independently generated response matrices; algorithm randomness enters only through the $20$ randomized trial seeds used in each sweep.

For PBGI, we convert this matrix into $152$ configuration-level arms. Each query evaluates one model on all $805$ instructions, revealing the row-average score and consuming that model's all-example evaluation cost. Cost-aware experiments use model-specific token prices from our Alpaca pricing table, with an input-to-output token ratio of $1{:}8$ estimated from typical AlpacaEval prompt/response lengths (Appendix~\ref{app:experiment-setup}). For GittinsEval-S on AlpacaEval, we use the dataset-specific prior mean/variance $(0.2,\,0.01)$; GittinsEval-G uses the general prior $(0.5,\,0.04)$ listed in Table~\ref{tab:prior_settings}.

\paragraph{MMLU.}
For MMLU, we use the DOVE response matrices~\citep{habba2025dove}. MMLU is a multiple-choice question-answering benchmark consisting of $57$ subjects and approximately $14{,}000$ examples in total. We treat each MMLU subject as a separate problem instance. Each subject is evaluated across $15$ LLM configurations and $100$ prompting techniques, resulting in $1500$ arms per subject.

The $15$ LLM configurations used in the MMLU response matrices are Llama-3-8B, Llama-3-8B-Instruct, Llama-3-70B-Instruct, CodeLlama-34B-Instruct, FLAN-T5-XL, FLAN-T5-XXL, FLAN-UL2, Merlinite-7B, Mixtral-8x7B-Instruct-v0.1, Mistral-7B-Instruct-v0.2, Gemma-7B, Gemma-7B-IT, Falcon-40B, Mistral-7B-v0.1, and Falcon-180B. The number of examples varies across subjects; the full list of subjects used in our evaluation is provided in Table~\ref{tab:mmlu-datasets}.

\subsection{Dataset Statistics}
\label{app:dataset-statistics}

We provide descriptive statistics for the response matrices used in our evaluation. For each dataset or subject, we compute the empirical quality of each arm as its mean score over benchmark examples. For binary correctness matrices, this score is mean accuracy; for AlpacaEval, it is the mean pairwise preference score. The following figures visualize the resulting distribution of arm qualities. These plots help illustrate the spread of arm performance, the location of the empirical mean, the best arm, and the prior mean used in informative-prior experiments.

\paragraph{GSM8K and PIQA.}
Figure~\ref{fig:gsm8k-piqa-quality-distribution} shows the per-arm quality distributions for GSM8K and PIQA across the five independently generated response matrices. The distributions are highly stable across seeds for both datasets. This indicates that, although the response matrices are generated independently, the randomness from LLM response generation introduces only limited variation at the aggregate level. GSM8K has a noticeably lower overall accuracy distribution than PIQA, while PIQA exhibits a tighter concentration of high-quality arms.

\begin{figure}[!htbp]
    \centering
    \includegraphics[width=\textwidth]{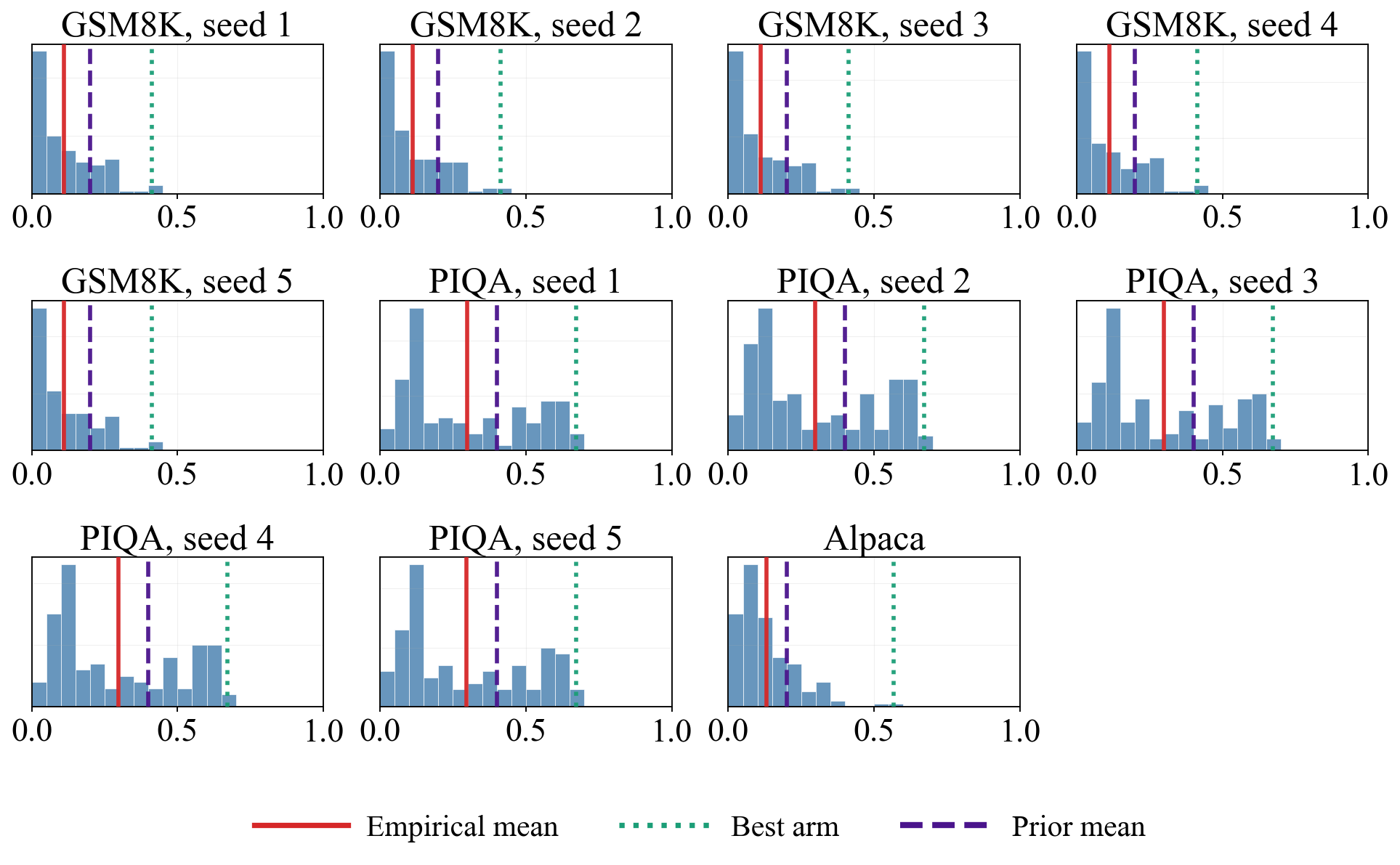}
    \caption{Per-arm quality distributions for GSM8K, PIQA, and AlpacaEval. Each panel shows one response matrix; bars report mean accuracy for GSM8K and PIQA and mean pairwise score for AlpacaEval. Vertical lines mark the empirical mean (red), best arm (green), and prior mean (purple).}
    \label{fig:gsm8k-piqa-quality-distribution}
\end{figure}

\paragraph{MMLU.}
For MMLU, we treat each subject as a separate problem instance. \edit{We assign subjects to easy, medium, and hard buckets using their empirical mean accuracy across configurations in the evaluation response matrices. These groups summarize difficulty for the evaluated configurations. Each subject receives one common prior across all $1500$ configurations; Table~\ref{tab:mmlu-datasets} lists the assignments.} We refer to these buckets as high-, medium-, and low-prior buckets, with prior means $\mu_0 = 0.75$, $\mu_0 = 0.6$, and $\mu_0 = 0.4$, respectively. \edit{Figures~\ref{fig:mmlu-prior-bucket-high}--\ref{fig:mmlu-prior-bucket-low} show the empirical per-arm accuracy distributions underlying these groups.}

\begin{figure}[!htbp]
    \centering
    \includegraphics[width=\textwidth]{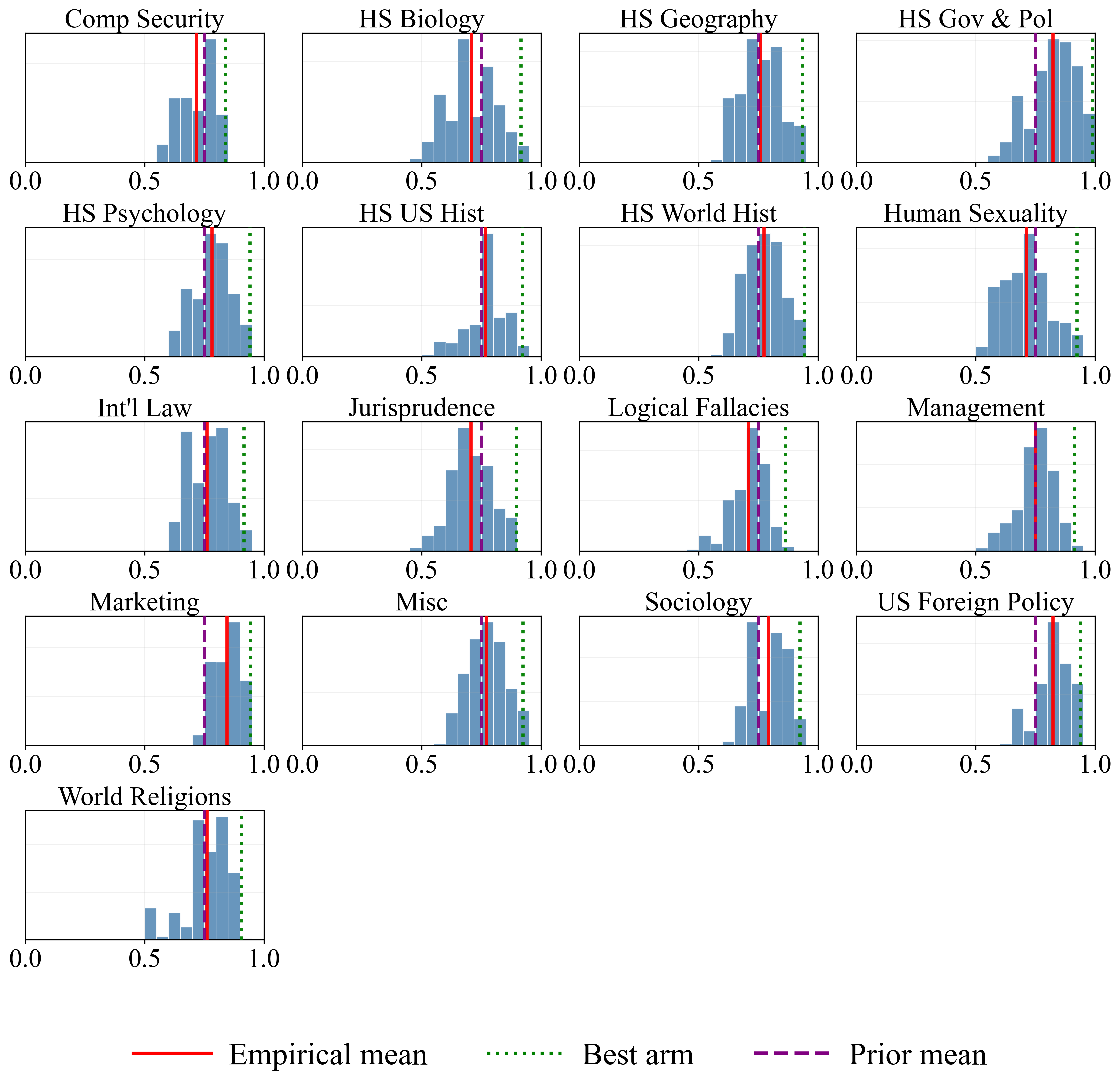}
    \caption{Per-arm accuracy distributions for easy MMLU subjects in the high-prior bucket ($\mu_0=0.75$). Each panel shows one subject; vertical lines mark the empirical mean (red), best arm (green), and prior mean (purple).}
    \label{fig:mmlu-prior-bucket-high}
\end{figure}

\begin{figure}[!htbp]
    \centering
    \includegraphics[width=\textwidth]{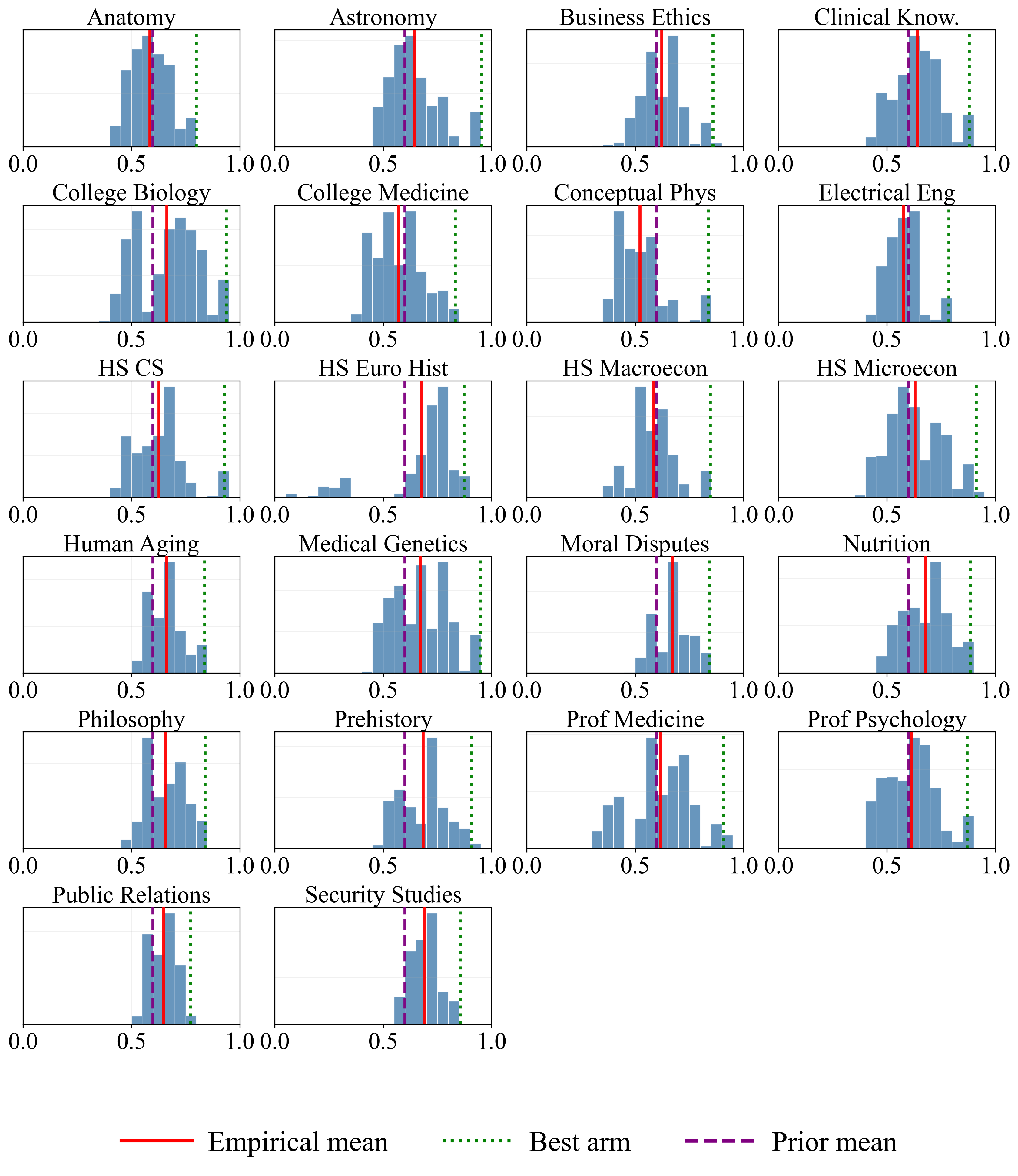}
    \caption{Per-arm accuracy distributions for medium-difficulty MMLU subjects in the medium-prior bucket ($\mu_0=0.6$). Each panel shows one subject; vertical lines mark the empirical mean (red), best arm (green), and prior mean (purple).}
    \label{fig:mmlu-prior-bucket-medium}
\end{figure}

\begin{figure}[!htbp]
    \centering
    \includegraphics[width=\textwidth]{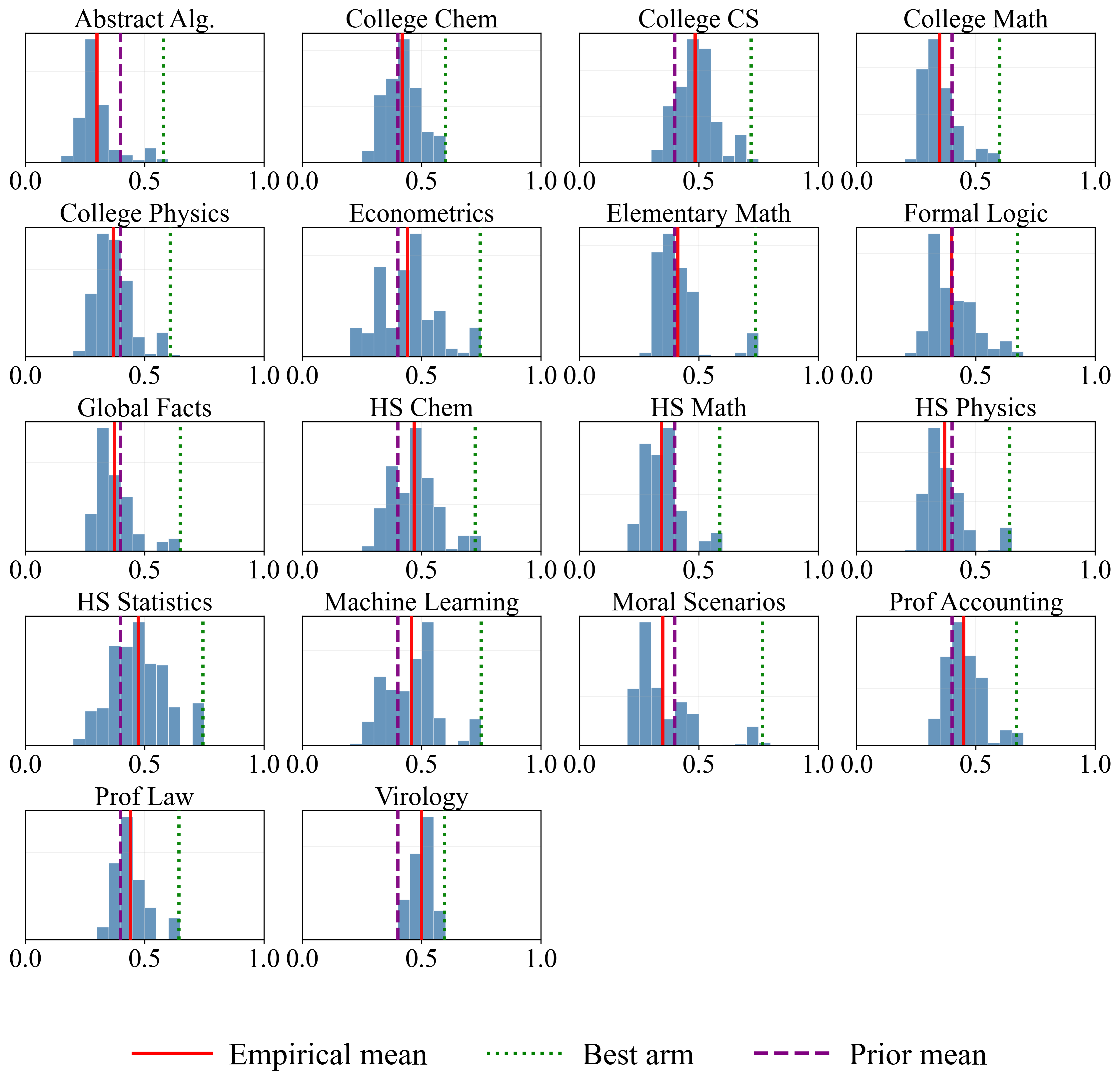}
    \caption{Per-arm accuracy distributions for hard MMLU subjects in the low-prior bucket ($\mu_0=0.4$). Each panel shows one subject; vertical lines mark the empirical mean (red), best arm (green), and prior mean (purple).}
    \label{fig:mmlu-prior-bucket-low}
\end{figure}

\begin{table}[!htbp]
\centering
\small
\renewcommand{\arraystretch}{0.98}
\caption{MMLU subjects used in our evaluation. Each response matrix has $1500$ arms ($15$ LLM configurations $\times$ $100$ prompts) and $N$ benchmark examples. \edit{Difficulty is based on empirical mean accuracy across configurations and determines their shared S-prior; size reflects the number of examples.}}
\begin{tabular}{lcccl}
\toprule
Dataset & Matrix Shape & Difficulty & Size & Full Name \\
\midrule
Abstract Alg. & $1500 \times 100$ & hard & small & Abstract Algebra \\
Anatomy & $1500 \times 135$ & medium & small & Anatomy \\
Astronomy & $1500 \times 152$ & medium & medium & Astronomy \\
Business Ethics & $1500 \times 100$ & medium & small & Business Ethics \\
Clinical Know. & $1500 \times 265$ & medium & medium & Clinical Knowledge \\
College Biology & $1500 \times 144$ & medium & small & College Biology \\
College Chem & $1500 \times 100$ & hard & small & College Chemistry \\
College CS & $1500 \times 100$ & hard & small & College Computer Science \\
College Math & $1500 \times 100$ & hard & small & College Mathematics \\
College Medicine & $1500 \times 173$ & medium & medium & College Medicine \\
College Physics & $1500 \times 102$ & hard & small & College Physics \\
Comp Security & $1500 \times 100$ & easy & small & Computer Security \\
Conceptual Phys & $1500 \times 235$ & medium & medium & Conceptual Physics \\
Econometrics & $1500 \times 114$ & hard & small & Econometrics \\
Electrical Eng & $1500 \times 145$ & medium & small & Electrical Engineering \\
Elementary Math & $1500 \times 378$ & hard & medium & Elementary Mathematics \\
Formal Logic & $1500 \times 126$ & hard & small & Formal Logic \\
Global Facts & $1500 \times 100$ & hard & small & Global Facts \\
HS Biology & $1500 \times 310$ & easy & medium & High School Biology \\
HS Chem & $1500 \times 203$ & hard & medium & High School Chemistry \\
HS CS & $1500 \times 100$ & medium & small & High School Computer Science \\
HS Euro Hist & $1500 \times 165$ & medium & medium & High School European History \\
HS Geography & $1500 \times 198$ & easy & medium & High School Geography \\
HS Gov \& Pol & $1500 \times 193$ & easy & medium & High School Government and Politics \\
HS Macroecon & $1500 \times 390$ & medium & medium & High School Macroeconomics \\
HS Math & $1500 \times 270$ & hard & medium & High School Mathematics \\
HS Microecon & $1500 \times 238$ & medium & medium & High School Microeconomics \\
HS Physics & $1500 \times 151$ & hard & medium & High School Physics \\
HS Psychology & $1500 \times 545$ & easy & large & High School Psychology \\
HS Statistics & $1500 \times 216$ & hard & medium & High School Statistics \\
HS US Hist & $1500 \times 204$ & easy & medium & High School US History \\
HS World Hist & $1500 \times 237$ & easy & medium & High School World History \\
Human Aging & $1500 \times 223$ & medium & medium & Human Aging \\
Human Sexuality & $1500 \times 131$ & easy & small & Human Sexuality \\
Int'l Law & $1500 \times 121$ & easy & small & International Law \\
Jurisprudence & $1500 \times 108$ & easy & small & Jurisprudence \\
Logical Fallacies & $1500 \times 163$ & easy & medium & Logical Fallacies \\
Machine Learning & $1500 \times 112$ & hard & small & Machine Learning \\
Management & $1500 \times 103$ & easy & small & Management \\
Marketing & $1500 \times 234$ & easy & medium & Marketing \\
Medical Genetics & $1500 \times 100$ & medium & small & Medical Genetics \\
Misc & $1500 \times 783$ & easy & large & Miscellaneous \\
Moral Disputes & $1500 \times 346$ & medium & medium & Moral Disputes \\
Moral Scenarios & $1500 \times 895$ & hard & large & Moral Scenarios \\
Nutrition & $1500 \times 306$ & medium & medium & Nutrition \\
Philosophy & $1500 \times 311$ & medium & medium & Philosophy \\
Prehistory & $1500 \times 324$ & medium & medium & Prehistory \\
Prof Accounting & $1500 \times 282$ & hard & medium & Professional Accounting \\
Prof Law & $1500 \times 1534$ & hard & large & Professional Law \\
Prof Medicine & $1500 \times 272$ & medium & medium & Professional Medicine \\
Prof Psychology & $1500 \times 612$ & medium & large & Professional Psychology \\
Public Relations & $1500 \times 110$ & medium & small & Public Relations \\
Security Studies & $1500 \times 245$ & medium & medium & Security Studies \\
Sociology & $1500 \times 201$ & easy & medium & Sociology \\
US Foreign Policy & $1500 \times 100$ & easy & small & US Foreign Policy \\
Virology & $1500 \times 166$ & hard & medium & Virology \\
World Religions & $1500 \times 171$ & easy & medium & World Religions \\
\bottomrule
\end{tabular}

\label{tab:mmlu-datasets}
\end{table}

\FloatBarrier

%% file: appendices/experiment_setup.tex
\section{Experiment Setup and Implementation Details}
\label{app:experiment-setup}

\paragraph{Computing environment.}
Experiments were run on CPU-only computing resources using precomputed response matrices rather than online LLM inference. The response matrix serves as the offline ground-truth oracle: \edit{during each run}, each policy observes only the selected $(\text{arm}, \text{question})$ entries, while per-arm matrix means are used to compute simple regret. PBGI observes the full-row aggregate score of each selected configuration. For GSM8K and PIQA, each converted BanditEval configuration corresponds to a model$\times$prompt arm. The BanditEval AlpacaEval matrix instead has one leaderboard-model arm per row, so each PBGI evaluation queries that model on all $805$ instructions. For MMLU, each converted DOVE configuration corresponds to a model$\times$template arm. BOHB evaluates configurations on nested subsets of benchmark examples. \edit{We measure total wall-clock runtime per completed offline evaluation run on these matrices.}

\paragraph{Policy implementation.}
For \textsc{GittinsEval}, we precompute empirical-target root schedules for the prior variance, batch-level noise and cost schedules, and horizon used in each condition. Online allocation then requires only posterior updates, empirical-target moment computation, and root-table lookup, keeping runtime close to UCB-E and well below LRF; Appendix~\ref{app:additional-experiment-results} reports timing comparisons. After random initialization, configuration-level Bayesian optimization fits mixed Gaussian-process models and optimizes their acquisition functions in BoTorch~\citep{balandat2020botorch}. Categorical kernels encode configuration features.

\paragraph{Experimental repetitions.}
For GSM8K and PIQA, we run each adaptive policy on each of the five independently generated response matrices. For every response matrix, we perform $20$ independent algorithm trials, each with independently randomized example orders and algorithmic randomness. We report averages over all matrix--trial pairs. For AlpacaEval, we use a single fixed pairwise score matrix with $152$ models and $805$ instructions. As in MMLU, we run $20$ independent algorithm trials on this matrix and report averages over those trials. Unlike GSM8K and PIQA, we do not average over multiple independently generated response matrices for AlpacaEval. For MMLU, we run $20$ randomized algorithm trials for each subject and aggregate results across the corresponding set of subjects. The Bayesian optimization baselines use the same randomized-trial structure on their converted configuration-level inputs.

\ifdefined\iclrsysrs
\paragraph{Anytime recommendations and regret.}
At each reporting step, GittinsEval-G and GittinsEval-S recommend the arm maximizing $M_{k,t}-\sqrt{V_{k,t}}$ as in Eq.~(\ref{eq:finite-lcb-recommendation}), where $M_{k,t}$ and $V_{k,t}$ are the posterior mean and variance of the full-row empirical score, given in Eq.~(\ref{eq:finite-posterior-moments}). This optional-completion recommendation considers all arms, including those with unevaluated examples; it does not require the recommended arm to be completed. For an arm that has been fully evaluated, $M_{k,t}$ equals its observed row mean and $V_{k,t}=0$. The uncertainty penalty is intended to reduce early recommendation switching when arms with few observations have unreliable mean estimates, especially under a general prior; see Appendix~\ref{app:anytime-recommendation}. Simple regret is always computed from the recommended arm's actual full-row mean in the response matrix, not from its posterior mean or its penalized recommendation score. The penalty affects recommendation only; it does not enter the numerical Gittins indices used for allocation or affect the stopping time.

\paragraph{GittinsEval stopping time.}
The GittinsEval stopping time is the first post-batch time at which an arm attaining the largest numerical empirical-target Gittins index is fully observed. This is the stopping rule from the required-completion adaptive-stopping formulation; it does not restrict the optional-completion LCB-style recommendation. Root discretization, continuation beyond the signal, and the recommendation rule do not inherit the exact batch-level optimality guarantee. In fixed-budget experiments, we record this time and continue the regret trajectory to the evaluation budget.

\else\fi

\paragraph{Baseline initialization and warm-up.}
For a response matrix with $K$ arms and $N$ examples, PBGI uses a random-initialization phase. A $5\%$ random initialization means drawing approximately $0.05K$ arms uniformly without replacement and observing their full rows before acquisition-driven evaluation. This corresponds to $6$ initial configurations for GSM8K ($K=122$), $5$ for PIQA ($K=103$), $8$ for AlpacaEval ($K=152$), and $75$ for each MMLU subject ($K=1500$). Under unit costs, the nominal $10\%$ budget permits $12$, $10$, $15$, and $150$ complete configurations, respectively; heterogeneous costs instead determine which evaluations fit within the remaining cost budget. \edit{For LRF, we follow BanditEval~\citep{zhou2024speeding} and draw individual arm--example pairs uniformly during warm-up, with a target of $T_0=0.05KN$ entries.}

\paragraph{\edit{Thompson sampling and BOHB.}}
\edit{Unless otherwise specified in an ablation, Gaussian-TS (TS-S) uses independent Gaussian posteriors with the same prior and batch size as GittinsEval-S, and recommends by posterior mean without an uncertainty penalty. FormatSpread-TS uses independent Beta--Bernoulli posteriors with $\beta_0=5$ and $\alpha_0=\max\{5x/(1-x),1.1\}$, where we set $x$ to the data-specific prior mean of GittinsEval-S (Table~\ref{tab:prior_settings}). We denote this main-experiment variant by FormatSpread-TS-S in the recommendation ablation. There, we also evaluate FormatSpread-TS-G with $x=0.5$, giving a $\operatorname{Beta}(5,5)$ prior, and compare empirical-mean, posterior-mean, and LCB-style recommendations separately within each prior category (Figures~\ref{fig:recommendation-mmlu-s}--\ref{fig:recommendation-benchmarks-g}). In the main experiments, it uses the same batch sizes and recommends the arm with the highest observed empirical mean, without an uncertainty penalty. It runs until no further evaluation is affordable or all entries have been observed, with no separate value-based stopping criterion. We restrict FormatSpread-TS to binary-response benchmarks, excluding the continuous scores in AlpacaEval. BOHB uses HpBandSter's categorical kernel density estimator and native Hyperband scheduler with $\eta=3$ and nominal example fractions $1/81,1/27,1/9,1/3,1$. Each evaluation reveals a prefix of a seeded item order for its configuration, with the prefix length rounded up to an integer. All newly revealed examples count toward the budget, including initialization, and repeated observations are cached. BOHB recommends the configuration with the highest observed mean at the highest completed fidelity. Each run terminates if the next scheduled evaluation would exceed its budget.}

\paragraph{Curve aggregation and visualization.}
Unless otherwise noted, repeated trajectories for GittinsEval-G, GittinsEval-S, UCB-E, LRF, and SySRs are aggregated on a common budget grid. For each method and panel, we linearly interpolate each randomized-run trajectory onto a shared grid of $350$ evenly spaced points spanning the observed budget range and compute the mean and standard error across the runs available at each grid point. Values outside a run's observed range are treated as missing rather than extrapolated, so early-ending runs are not right-extended. The same procedure is used in the unit-cost and cost-aware settings, with cumulative evaluation cost defining the horizontal coordinate in both; under unit costs, this equals cumulative example evaluations. Gittins stopping points are reported only as overlays and do not truncate the regret trajectories. PBGI trajectories are handled separately. We omit the random-initialization segment from the displayed curves, align the post-initialization trajectories across randomized runs to their mean initialization endpoint, and aggregate them on the resulting common budget axis. Runs that terminate earlier are right-held at their final observed simple regret so that they continue to contribute over the displayed post-initialization budget range. This convention is used in both the unit-cost and cost-aware settings and is particularly relevant in the latter, where heterogeneous configuration costs lead to different initialization endpoints across runs. For the main cross-subject MMLU aggregates, the horizontal coordinate is actual cumulative cost divided by the subject's exhaustive-evaluation cost, rather than by the run's final cost. We align the saved subject-mean curves on a common grid over the nominal $0$--$10\%$ budget, weight subjects equally within each bucket, and report one standard error across subject means. Points require coverage from all selected subjects. \edit{Except in the recommendation-rule ablation (Figure~\ref{fig:recommendation-rule-ablation}), TS-S, FormatSpread-TS, and BOHB are aggregated on 1,001 evenly spaced budget points using the latest completed observation at or before each point; no recommendation is backfilled before the first observation. Completed FormatSpread-TS runs retain their final recommendation when no further action is affordable, and BOHB runs do so when the next scheduled evaluation is unaffordable.} PromptEval uses the phase-wise aggregation described below.

\ifdefined\iclrsysrs
\paragraph{PromptEval evaluation and aggregation.}
We evaluate PromptEval's best-prompt-identification instantiation~\citep{polo2024efficient}, \edit{labeled PromptEval-BAI in the timing figures}. Its variants are labeled PE-OneHot and \edit{PE-EmbPT} in the main figures. Its native allocation follows a unit-cost successive-halving best-arm-identification rule that exploits configuration covariates; costs are recorded so that the same trajectories can be plotted against cumulative observations in the unit-cost setting and against recorded cumulative costs in the cost-aware setting. On GSM8K and PIQA we use five response matrices with $20$ randomized trials each; on each MMLU subject we use $20$ randomized trials. The main GSM8K/PIQA plots use PE-OneHot \edit{(PE-Rasch)}, which assigns each candidate arm an identity feature vector. \edit{The main MMLU plot uses our interaction PE-EmbPT adaptation with balanced question sampling: model one-hot indicators interact with prompt embeddings, allowing model-specific prompt effects, with additional model offsets.} \edit{For model indicator $e_m$ and pretrained prompt embedding $z_p$, the interaction representation is $[e_m; e_m\otimes z_p]$; concatenation instead uses $[e_m; z_p]$ and shares prompt effects across models. These are adaptations to the joint model--prompt space, rather than separately named variants in the original PromptEval paper.} Figure~\ref{fig:prompteval-representation} compares these representations with PE-OneHot and GittinsEval. The individual-subject MMLU appendix plots use the same interaction PE-EmbPT representation. We do not report PromptEval-BAI on AlpacaEval because the released best-arm-identification implementation assumes binary correctness matrices and therefore does not directly apply to our continuous model-by-instruction preference matrix. PromptEval's own AlpacaEval experiment instead studies sensitivity to judge prompts and binarizes continuous judge scores for model fitting; it is separate from the paper's best-prompt-identification experiments on MMLU, BBH, and LMentry. Applying PromptEval-BAI here would therefore require an additional, non-native modeling choice such as thresholding the scores. For the main GSM8K/PIQA curves, we aggregate repeated trajectories by successive-halving phase, reporting mean simple regret at the mean phase budget with standard-error bands. Since PromptEval-BAI produces recommendations at discrete successive-halving phase boundaries, we preserve this native phase structure rather than interpolating additional intermediate points. For individual MMLU subject panels, both cost settings use phase-wise aggregation over 20 trials, with post-step lines connecting the mean budget and regret at successive phases. For the main cross-subject MMLU plot, we average the 20 seeds within each subject at each of the five successive-halving phases, then average subjects equally within each bucket. Each point uses that phase's mean actual budget and mean regret; bands show one standard error across subject means. We connect these five phase points with post-step lines, so cost-aware phase locations need not be equally spaced.
\fi

\paragraph{Evaluation costs.}
The cost-aware bandit experiments assign each arm a batch cost proportional to the estimated cost of running that model configuration on $B$ benchmark examples. For GSM8K and PIQA, an arm is a model plus sampling or prompting configuration; when several arms share a base model, they share the model-level price, while the batch multiplier accounts for how many examples are queried at that allocation step. For MMLU, each arm is a model--prompt pair, and model-level input prices are repeated across the corresponding prompt arms. For AlpacaEval, each arm is a single leaderboard model rather than a model$\times$sampling-configuration pair. Throughout our experiments, costs are model-level proxies derived from published API prices or model-level price metadata. Fixed benchmark-level input/output ratios convert these prices into a per-example arm cost, so all examples evaluated by the same arm have the same cost; we do not use \edit{token counts or latency measured for individual evaluations}. For the configuration-level Bayesian optimization baselines, each converted input row corresponds to a complete configuration, and its cost is the estimated all-example evaluation cost of that configuration on the corresponding benchmark or MMLU subject. The Gittins dynamic program uses costs rescaled to the same utility units as the index computation, while the configuration-level Bayesian optimization baselines use full-configuration costs for cost-aware acquisition and budget accounting. \edit{BOHB's allocation policy is unchanged across cost settings; unit or model-specific costs are used only for budget accounting.}

\paragraph{AlpacaEval evaluation costs.}
AlpacaEval uses model-specific API or proxy prices for the $152$ leaderboard models in the response matrix. For each model, we record separate input and output costs in USD per 1M tokens and form the evaluation cost as
\[
\text{cost} = c_{\mathrm{in}} + 8\,c_{\mathrm{out}},
\]
using the estimated AlpacaEval input-to-output token ratio $1{:}8$. These model-level costs define the per-arm cost vector for cost-aware bandit experiments and the all-example evaluation costs used by the converted Bayesian optimization inputs. Since AlpacaEval arms are individual leaderboard models rather than BanditEval-style model$\times$prompt configurations, we do not repeat a single model price across multiple prompting arms as in MMLU. Table~\ref{tab:gsm8k_piqa_model_costs} and Table~\ref{tab:mmlu_model_costs} therefore do not apply to AlpacaEval; the complete Alpaca pricing table is reported in Table~\ref{tab:alpaca_model_costs}.

\begin{table}[!htbp]
\centering
\caption{
Model-level evaluation costs used for GSM8K and PIQA. Costs are reported in USD per 1M tokens. GSM8K uses an assumed input-to-output token ratio of $1:2$, so the evaluation cost is computed as input cost plus twice the output cost. PIQA uses input-only pricing. These model-level costs are used to construct the per-arm cost vectors for the cost-aware experiments. }
\label{tab:gsm8k_piqa_model_costs}
\small
\setlength{\tabcolsep}{5pt}
\begin{tabular}{lcccc}
\toprule
Model & Input cost & Output cost & GSM8K cost & PIQA cost \\
\midrule
CodeLlama     & 0.30 & 0.30 & 0.90 & 0.30 \\
Gemma-7B      & 0.20 & 0.20 & 0.60 & 0.20 \\
GPT-2         & 0.10 & 0.10 & 0.30 & 0.10 \\
GPT-2 Large   & 0.10 & 0.10 & 0.30 & 0.10 \\
LLaMA2-7B     & 0.20 & 0.20 & 0.60 & 0.20 \\
Llemma-7B     & 0.80 & 1.20 & 3.20 & 0.80 \\
Mistral-7B    & 0.05 & 0.20 & 0.45 & 0.05 \\
Phi-2         & 0.05 & 0.10 & 0.25 & 0.05 \\
StarCoder2-7B & 0.20 & 0.20 & 0.60 & 0.20 \\
Tulu          & 0.20 & 0.20 & 0.60 & 0.20 \\
Tulu2         & 0.20 & 0.20 & 0.60 & 0.20 \\
\bottomrule
\end{tabular}
\end{table}

\begin{table}[!htbp]
\centering
\caption{
Model-level input costs used for MMLU. Costs are reported in USD per 1M input tokens. MMLU uses input-only pricing. Each MMLU subject contains 1500 arms, corresponding to 15 models paired with 100 prompt configurations; therefore, each model-level input cost is repeated across the corresponding prompt arms. }
\label{tab:mmlu_model_costs}
\small
\setlength{\tabcolsep}{5pt}
\begin{tabular*}{\textwidth}{@{\extracolsep{\fill}}lr@{\hspace{2em}}lr@{}}
\toprule
MMLU model & Input cost & MMLU model & Input cost \\
\midrule
CodeLlama-34B-Instruct & 0.776 & Llama-3-70B-Instruct      & 0.51 \\
Falcon-180B           & 1.25  & Llama-3-8B               & 0.05 \\
Falcon-40B            & 0.84  & Llama-3-8B-Instruct      & 0.03 \\
FLAN-T5-XL            & 0.60  & Merlinite-7B             & 0.60 \\
FLAN-T5-XXL           & 1.80  & Mistral-7B-Instruct-v0.2  & 0.14 \\
FLAN-UL2              & 5.00  & Mistral-7B-v0.1          & 0.11 \\
Gemma-7B              & 0.20  & Mixtral-8x7B-Instruct-v0.1 & 0.54 \\
Gemma-7B-IT           & 0.07  &                          &      \\
\bottomrule
\end{tabular*}
\end{table}

\clearpage
\IfFileExists{appendices/alpaca_pricing_table.tex}{%
  \input{appendices/alpaca_pricing_table}}{%
\input{appendices/alpaca_pricing_table}}

\paragraph{Gaussian approximation.}
For benchmark scores grouped in batches, a batch of $B$ responses produces an empirical mean score. For binary-accuracy benchmarks such as GSM8K, PIQA, and MMLU, the default observation variance is $\tau^2=1/(4B)$, the worst-case binary variance divided by the batch size. AlpacaEval matrix entries are continuous pairwise preference scores in $[0,1]$ rather than binary correctness labels. Since any random variable supported on $[0,1]$ has variance at most $1/4$, we also use the conservative per-comparison working variance $\tau_{\mathrm{cell}}^2=1/4$ for AlpacaEval. Under the working conditional-independence approximation, this gives the batch-mean variance $\tau^2=1/(4B)$. This fixed value is a conservative working bound, not an empirical variance estimate.

\paragraph{Prior settings.}
\edit{All arms within a problem instance share the same initial prior mean and variance.} Table~\ref{tab:prior_settings} lists the general default prior and the data-specific common priors used by GittinsEval-G and GittinsEval-S. \edit{GittinsEval-S selects difficulty buckets using aggregate mean scores from the evaluation response matrices. This retrospective setting examines the value of coarse information about the benchmark's overall performance level. It characterizes performance conditional on this supplied information; prospective use would require an independent source of prior information. Figures~\ref{fig:mmlu-prior-bucket-high}, \ref{fig:mmlu-prior-bucket-medium}, and \ref{fig:mmlu-prior-bucket-low} show the empirical accuracy distributions underlying the MMLU groups.}

\begin{table}[!htbp]
\centering
\caption{
Prior settings used in the Gittins-index experiments. We report Gaussian priors as $\mathcal{N}(\mu_0, v_0)$. The default prior is a general accuracy-scale choice. \edit{Data-specific priors encode coarse aggregate performance information from the evaluation response matrices, with a common prior shared by all arms in each problem instance.} }
\label{tab:prior_settings}
\small
\setlength{\tabcolsep}{9pt}
\begin{tabular}{lll}
\toprule
Benchmark / group & Prior & Description \\
\midrule
All benchmarks & $\mathcal{N}(0.5, 0.04)$ & General default prior \\
GSM8K & $\mathcal{N}(0.2, 0.01)$ & Very hard benchmark \\
PIQA & $\mathcal{N}(0.4, 0.02)$ & Hard benchmark \\
AlpacaEval & $\mathcal{N}(0.2, 0.01)$ & Very hard benchmark \\
\edit{MMLU low-accuracy bucket} & $\mathcal{N}(0.4, 0.02)$ & Hard subjects \\
\edit{MMLU medium-accuracy bucket} & $\mathcal{N}(0.6, 0.02)$ & Medium subjects \\
\edit{MMLU high-accuracy bucket} & $\mathcal{N}(0.75, 0.01)$ & Easy subjects \\
\bottomrule
\end{tabular}
\end{table}

%% file: appendices/alpaca_pricing_table.tex
\begingroup
\setlength{\tabcolsep}{2pt}
\renewcommand{\arraystretch}{1.03}
\newcommand{\alpacamodel}[1]{\raggedright\hyphenpenalty=10000\exhyphenpenalty=10000 #1}
\begin{table}[p]
\centering
\caption{AlpacaEval model-level pricing for the $152$-model response matrix. Input (In.) and output (Out.) prices are in USD per 1M tokens. The combined evaluation cost is $\mathrm{Cost}=\mathrm{Input}+8\times\mathrm{Output}$, matching the input-to-output token ratio used for cost-aware runs. Models are alphabetized down the left block, then the right block, continuing on the next page.}
\label{tab:alpaca_model_costs}
\fontsize{10}{11.6}\selectfont
\begin{tabular*}{\textwidth}{@{\extracolsep{\fill}}p{0.305\textwidth}rrr@{\hspace{12pt}}p{0.305\textwidth}rrr@{}}
\toprule
\textbf{Model} & \textbf{In.} & \textbf{Out.} & \textbf{Cost} & \textbf{Model} & \textbf{In.} & \textbf{Out.} & \textbf{Cost} \\
\midrule
\alpacamodel{airoboros-\allowbreak 33b} & 0.9 & 0.9 & 8.1 & \alpacamodel{gemma-\allowbreak 7b-\allowbreak it} & 0.2 & 0.2 & 1.8 \\
\alpacamodel{airoboros-\allowbreak 65b} & 0.9 & 0.9 & 8.1 & \alpacamodel{gpt-\allowbreak 3.5-\allowbreak turbo-\allowbreak 0301} & 1.5 & 2 & 17.5 \\
\alpacamodel{aligner-\allowbreak 2b\_\allowbreak claude-\allowbreak 3-\allowbreak opus-\allowbreak 20240229} & 0.1 & 0.1 & 0.9 & \alpacamodel{gpt-\allowbreak 3.5-\allowbreak turbo-\allowbreak 0613} & 1.5 & 2 & 17.5 \\
\alpacamodel{aligner-\allowbreak 2b\_\allowbreak qwen1.5-\allowbreak 72b-\allowbreak chat} & 0.1 & 0.1 & 0.9 & \alpacamodel{gpt-\allowbreak 3.5-\allowbreak turbo-\allowbreak 1106} & 1 & 2 & 17 \\
\alpacamodel{alpaca-\allowbreak 7b} & 0.2 & 0.2 & 1.8 & \alpacamodel{gpt-\allowbreak 3.5-\allowbreak turbo-\allowbreak 1106\_\allowbreak concise} & 1 & 2 & 17 \\
\alpacamodel{alpaca-\allowbreak 7b\_\allowbreak concise} & 0.2 & 0.2 & 1.8 & \alpacamodel{gpt-\allowbreak 3.5-\allowbreak turbo-\allowbreak 1106\_\allowbreak verbose} & 1 & 2 & 17 \\
\alpacamodel{alpaca-\allowbreak 7b\_\allowbreak verbose} & 0.2 & 0.2 & 1.8 & \alpacamodel{gpt-\allowbreak 4-\allowbreak 0125-\allowbreak preview} & 10 & 30 & 250 \\
\alpacamodel{alpaca-\allowbreak farm-\allowbreak ppo-\allowbreak human} & 0.2 & 0.2 & 1.8 & \alpacamodel{gpt35\_\allowbreak turbo\_\allowbreak instruct} & 1.5 & 2 & 17.5 \\
\alpacamodel{alpaca-\allowbreak farm-\allowbreak ppo-\allowbreak sim-\allowbreak gpt4-\allowbreak 20k} & 0.2 & 0.2 & 1.8 & \alpacamodel{gpt4} & 30 & 60 & 510 \\
\alpacamodel{baichuan-\allowbreak 13b-\allowbreak chat} & 0.2 & 0.2 & 1.8 & \alpacamodel{gpt4\_\allowbreak 0314} & 30 & 60 & 510 \\
\alpacamodel{baize-\allowbreak v2-\allowbreak 13b} & 0.2 & 0.2 & 1.8 & \alpacamodel{gpt4\_\allowbreak 0613} & 30 & 60 & 510 \\
\alpacamodel{baize-\allowbreak v2-\allowbreak 7b} & 0.2 & 0.2 & 1.8 & \alpacamodel{gpt4\_\allowbreak 0613\_\allowbreak concise} & 30 & 60 & 510 \\
\alpacamodel{causallm-\allowbreak 14b} & 0.2 & 0.2 & 1.8 & \alpacamodel{gpt4\_\allowbreak 0613\_\allowbreak verbose} & 30 & 60 & 510 \\
\alpacamodel{chatglm2-\allowbreak 6b} & 0.2 & 0.2 & 1.8 & \alpacamodel{gpt4\_\allowbreak 1106\_\allowbreak preview\_\allowbreak concise} & 10 & 30 & 250 \\
\alpacamodel{claude} & 8 & 24 & 200 & \alpacamodel{gpt4\_\allowbreak gamed} & 30 & 60 & 510 \\
\alpacamodel{claude-\allowbreak 2} & 8 & 24 & 200 & \alpacamodel{guanaco-\allowbreak 13b} & 0.2 & 0.2 & 1.8 \\
\alpacamodel{claude-\allowbreak 2.1} & 8 & 24 & 200 & \alpacamodel{guanaco-\allowbreak 33b} & 0.9 & 0.9 & 8.1 \\
\alpacamodel{claude-\allowbreak 2.1\_\allowbreak concise} & 8 & 24 & 200 & \alpacamodel{guanaco-\allowbreak 65b} & 0.9 & 0.9 & 8.1 \\
\alpacamodel{claude-\allowbreak 2.1\_\allowbreak verbose} & 8 & 24 & 200 & \alpacamodel{guanaco-\allowbreak 7b} & 0.2 & 0.2 & 1.8 \\
\alpacamodel{claude-\allowbreak 3-\allowbreak opus-\allowbreak 20240229} & 15 & 75 & 615 & \alpacamodel{humpback-\allowbreak llama-\allowbreak 65b} & 0.9 & 0.9 & 8.1 \\
\alpacamodel{claude-\allowbreak 3-\allowbreak sonnet-\allowbreak 20240229} & 3 & 15 & 123 & \alpacamodel{humpback-\allowbreak llama2-\allowbreak 70b} & 0.9 & 0.9 & 8.1 \\
\alpacamodel{claude-\allowbreak instant-\allowbreak 1.2} & 0.8 & 2.4 & 20 & \alpacamodel{internlm2-\allowbreak chat-\allowbreak 20b-\allowbreak ppo} & 0.9 & 0.9 & 8.1 \\
\alpacamodel{claude2-\allowbreak alpaca-\allowbreak 13b} & 0.2 & 0.2 & 1.8 & \alpacamodel{jina-\allowbreak chat} & 0.2 & 0.2 & 1.8 \\
\alpacamodel{cohere} & 1 & 2 & 17 & \alpacamodel{llama-\allowbreak 2-\allowbreak 13b-\allowbreak chat-\allowbreak hf} & 0.2 & 0.2 & 1.8 \\
\alpacamodel{Conifer-\allowbreak 7B-\allowbreak DPO} & 0.2 & 0.2 & 1.8 & \alpacamodel{llama-\allowbreak 2-\allowbreak 70b-\allowbreak chat-\allowbreak hf} & 0.9 & 0.9 & 8.1 \\
\alpacamodel{Contextual-\allowbreak KTO-\allowbreak Mistral-\allowbreak PairRM} & 0.2 & 0.2 & 1.8 & \alpacamodel{llama-\allowbreak 2-\allowbreak 7b-\allowbreak chat-\allowbreak hf} & 0.2 & 0.2 & 1.8 \\
\alpacamodel{cut-\allowbreak 13b} & 0.2 & 0.2 & 1.8 & \alpacamodel{llama-\allowbreak 2-\allowbreak chat-\allowbreak 7b-\allowbreak evol70k-\allowbreak neft} & 0.2 & 0.2 & 1.8 \\
\alpacamodel{dbrx-\allowbreak instruct} & 1.2 & 1.2 & 10.8 & \alpacamodel{LLaMA-\allowbreak 7B} & 0.2 & 0.2 & 1.8 \\
\alpacamodel{deepseek-\allowbreak llm-\allowbreak 67b-\allowbreak chat} & 0.9 & 0.9 & 8.1 & \alpacamodel{LMCocktail-\allowbreak 10.7B-\allowbreak v1} & 0.2 & 0.2 & 1.8 \\
\alpacamodel{deita-\allowbreak 7b-\allowbreak v1.0} & 0.2 & 0.2 & 1.8 & \alpacamodel{Meta-\allowbreak Llama-\allowbreak 3-\allowbreak 70B-\allowbreak Instruct} & 0.9 & 0.9 & 8.1 \\
\alpacamodel{dolphin-\allowbreak 2.2.1-\allowbreak mistral-\allowbreak 7b} & 0.2 & 0.2 & 1.8 & \alpacamodel{Meta-\allowbreak Llama-\allowbreak 3-\allowbreak 8B-\allowbreak Instruct} & 0.2 & 0.2 & 1.8 \\
\alpacamodel{evo-\allowbreak 7b} & 0.2 & 0.2 & 1.8 & \alpacamodel{minichat-\allowbreak 1.5-\allowbreak 3b} & 0.1 & 0.1 & 0.9 \\
\alpacamodel{evo-\allowbreak v2-\allowbreak 7b} & 0.2 & 0.2 & 1.8 & \alpacamodel{minichat-\allowbreak 3b} & 0.1 & 0.1 & 0.9 \\
\alpacamodel{falcon-\allowbreak 40b-\allowbreak instruct} & 0.9 & 0.9 & 8.1 & \alpacamodel{minotaur-\allowbreak 13b} & 0.2 & 0.2 & 1.8 \\
\alpacamodel{falcon-\allowbreak 7b-\allowbreak instruct} & 0.2 & 0.2 & 1.8 & \alpacamodel{Mistral-\allowbreak 7B-\allowbreak Instruct-\allowbreak v0.2} & 0.2 & 0.2 & 1.8 \\
\alpacamodel{FsfairX-\allowbreak Zephyr-\allowbreak Chat-\allowbreak v0.1} & 0.2 & 0.2 & 1.8 & \alpacamodel{Mistral-\allowbreak 7B-\allowbreak ReMax-\allowbreak v0.1} & 0.2 & 0.2 & 1.8 \\
\alpacamodel{gemini-\allowbreak pro} & 0.5 & 1.5 & 12.5 & \alpacamodel{mistral-\allowbreak large-\allowbreak 2402} & 8 & 24 & 200 \\
\alpacamodel{gemma-\allowbreak 2b-\allowbreak it} & 0.1 & 0.1 & 0.9 & \alpacamodel{mistral-\allowbreak medium} & 2.7 & 8.1 & 67.5 \\
\bottomrule
\end{tabular*}
\end{table}
\clearpage
\begin{table}[p]
\centering
\ContinuedFloat
\caption[]{AlpacaEval model-level pricing (continued). Prices are in USD per 1M tokens; $\mathrm{Cost}=\mathrm{Input}+8\times\mathrm{Output}$.}
\fontsize{10}{11.6}\selectfont
\begin{tabular*}{\textwidth}{@{\extracolsep{\fill}}p{0.305\textwidth}rrr@{\hspace{12pt}}p{0.305\textwidth}rrr@{}}
\toprule
\textbf{Model} & \textbf{In.} & \textbf{Out.} & \textbf{Cost} & \textbf{Model} & \textbf{In.} & \textbf{Out.} & \textbf{Cost} \\
\midrule
\alpacamodel{mistral-\allowbreak orpo-\allowbreak beta} & 0.2 & 0.2 & 1.8 & \alpacamodel{Qwen1.5-\allowbreak 7B-\allowbreak Chat} & 0.2 & 0.2 & 1.8 \\
\alpacamodel{Mixtral-\allowbreak 8x22B-\allowbreak Instruct-\allowbreak v0.1} & 1.2 & 1.2 & 10.8 & \alpacamodel{recycled-\allowbreak wizardlm-\allowbreak 7b-\allowbreak v1.0} & 0.2 & 0.2 & 1.8 \\
\alpacamodel{Mixtral-\allowbreak 8x7B-\allowbreak Instruct-\allowbreak v0.1} & 0.5 & 0.5 & 4.5 & \alpacamodel{recycled-\allowbreak wizardlm-\allowbreak 7b-\allowbreak v2.0} & 0.2 & 0.2 & 1.8 \\
\alpacamodel{Mixtral-\allowbreak 8x7B-\allowbreak Instruct-\allowbreak v0.1\_\allowbreak concise} & 0.5 & 0.5 & 4.5 & \alpacamodel{Samba-\allowbreak CoE-\allowbreak v0.1} & 0.2 & 0.2 & 1.8 \\
\alpacamodel{Mixtral-\allowbreak 8x7B-\allowbreak Instruct-\allowbreak v0.1\_\allowbreak verbose} & 0.5 & 0.5 & 4.5 & \alpacamodel{Samba-\allowbreak CoE-\allowbreak v0.2} & 0.2 & 0.2 & 1.8 \\
\alpacamodel{Nanbeige-\allowbreak Plus-\allowbreak Chat-\allowbreak v0.1} & 0.2 & 0.2 & 1.8 & \alpacamodel{Samba-\allowbreak CoE-\allowbreak v0.2-\allowbreak best-\allowbreak of-\allowbreak 16} & 3.2 & 3.2 & 28.8 \\
\alpacamodel{Nanbeige2-\allowbreak 8B-\allowbreak Chat} & 0.2 & 0.2 & 1.8 & \alpacamodel{Snorkel-\allowbreak Mistral-\allowbreak PairRM-\allowbreak DPO} & 0.2 & 0.2 & 1.8 \\
\alpacamodel{nous-\allowbreak hermes-\allowbreak 13b} & 0.2 & 0.2 & 1.8 & \alpacamodel{Snorkel-\allowbreak Mistral-\allowbreak PairRM-\allowbreak DPO-\allowbreak best-\allowbreak of-\allowbreak 16} & 3.2 & 3.2 & 28.8 \\
\alpacamodel{oasst-\allowbreak rlhf-\allowbreak llama-\allowbreak 33b} & 0.9 & 0.9 & 8.1 & \alpacamodel{Starling-\allowbreak LM-\allowbreak 7B-\allowbreak alpha} & 0.2 & 0.2 & 1.8 \\
\alpacamodel{oasst-\allowbreak sft-\allowbreak llama-\allowbreak 33b} & 0.9 & 0.9 & 8.1 & \alpacamodel{TempNet-\allowbreak LLaMA2-\allowbreak Chat-\allowbreak 13B-\allowbreak v0.1} & 0.2 & 0.2 & 1.8 \\
\alpacamodel{oasst-\allowbreak sft-\allowbreak pythia-\allowbreak 12b} & 0.1 & 0.1 & 0.9 & \alpacamodel{TempNet-\allowbreak LLaMA2-\allowbreak Chat-\allowbreak 70B-\allowbreak v0.1} & 0.9 & 0.9 & 8.1 \\
\alpacamodel{openbuddy-\allowbreak falcon-\allowbreak 40b-\allowbreak v9} & 0.9 & 0.9 & 8.1 & \alpacamodel{TempNet-\allowbreak LLaMA2-\allowbreak Chat-\allowbreak 7B-\allowbreak v0.1} & 0.2 & 0.2 & 1.8 \\
\alpacamodel{openbuddy-\allowbreak falcon-\allowbreak 7b-\allowbreak v6} & 0.2 & 0.2 & 1.8 & \alpacamodel{text\_\allowbreak davinci\_\allowbreak 001} & 20 & 20 & 180 \\
\alpacamodel{openbuddy-\allowbreak llama-\allowbreak 30b-\allowbreak v7.1} & 0.9 & 0.9 & 8.1 & \alpacamodel{text\_\allowbreak davinci\_\allowbreak 003} & 20 & 20 & 180 \\
\alpacamodel{openbuddy-\allowbreak llama-\allowbreak 65b-\allowbreak v8} & 0.9 & 0.9 & 8.1 & \alpacamodel{tulu-\allowbreak 2-\allowbreak dpo-\allowbreak 13b} & 0.2 & 0.2 & 1.8 \\
\alpacamodel{openbuddy-\allowbreak llama2-\allowbreak 13b-\allowbreak v11.1} & 0.2 & 0.2 & 1.8 & \alpacamodel{tulu-\allowbreak 2-\allowbreak dpo-\allowbreak 70b} & 0.9 & 0.9 & 8.1 \\
\alpacamodel{openbuddy-\allowbreak llama2-\allowbreak 70b-\allowbreak v10.1} & 0.9 & 0.9 & 8.1 & \alpacamodel{tulu-\allowbreak 2-\allowbreak dpo-\allowbreak 7b} & 0.2 & 0.2 & 1.8 \\
\alpacamodel{openchat-\allowbreak 13b} & 0.2 & 0.2 & 1.8 & \alpacamodel{ultralm-\allowbreak 13b} & 0.2 & 0.2 & 1.8 \\
\alpacamodel{openchat-\allowbreak v2-\allowbreak 13b} & 0.2 & 0.2 & 1.8 & \alpacamodel{ultralm-\allowbreak 13b-\allowbreak best-\allowbreak of-\allowbreak 16} & 3.2 & 3.2 & 28.8 \\
\alpacamodel{openchat-\allowbreak v2-\allowbreak w-\allowbreak 13b} & 0.2 & 0.2 & 1.8 & \alpacamodel{ultralm-\allowbreak 13b-\allowbreak v2.0} & 0.2 & 0.2 & 1.8 \\
\alpacamodel{openchat-\allowbreak v3.1-\allowbreak 13b} & 0.2 & 0.2 & 1.8 & \alpacamodel{ultralm-\allowbreak 13b-\allowbreak v2.0-\allowbreak best-\allowbreak of-\allowbreak 16} & 3.2 & 3.2 & 28.8 \\
\alpacamodel{openchat8192-\allowbreak 13b} & 0.2 & 0.2 & 1.8 & \alpacamodel{vicuna-\allowbreak 13b} & 0.2 & 0.2 & 1.8 \\
\alpacamodel{opencoderplus-\allowbreak 15b} & 0.2 & 0.2 & 1.8 & \alpacamodel{vicuna-\allowbreak 13b-\allowbreak v1.3} & 0.2 & 0.2 & 1.8 \\
\alpacamodel{OpenHermes-\allowbreak 2.5-\allowbreak Mistral-\allowbreak 7B} & 0.2 & 0.2 & 1.8 & \alpacamodel{vicuna-\allowbreak 13b-\allowbreak v1.5} & 0.2 & 0.2 & 1.8 \\
\alpacamodel{pairrm-\allowbreak tulu-\allowbreak 2-\allowbreak 13b} & 0.2 & 0.2 & 1.8 & \alpacamodel{vicuna-\allowbreak 33b-\allowbreak v1.3} & 0.9 & 0.9 & 8.1 \\
\alpacamodel{pairrm-\allowbreak tulu-\allowbreak 2-\allowbreak 70b} & 0.9 & 0.9 & 8.1 & \alpacamodel{vicuna-\allowbreak 7b} & 0.2 & 0.2 & 1.8 \\
\alpacamodel{pairrm-\allowbreak Yi-\allowbreak 34B-\allowbreak Chat} & 0.9 & 0.9 & 8.1 & \alpacamodel{vicuna-\allowbreak 7b-\allowbreak v1.3} & 0.2 & 0.2 & 1.8 \\
\alpacamodel{pairrm-\allowbreak zephyr-\allowbreak 7b-\allowbreak beta} & 0.2 & 0.2 & 1.8 & \alpacamodel{vicuna-\allowbreak 7b-\allowbreak v1.5} & 0.2 & 0.2 & 1.8 \\
\alpacamodel{phi-\allowbreak 2} & 0.1 & 0.1 & 0.9 & \alpacamodel{wizardlm-\allowbreak 13b} & 0.2 & 0.2 & 1.8 \\
\alpacamodel{phi-\allowbreak 2-\allowbreak dpo} & 0.1 & 0.1 & 0.9 & \alpacamodel{wizardlm-\allowbreak 13b-\allowbreak v1.1} & 0.2 & 0.2 & 1.8 \\
\alpacamodel{phi-\allowbreak 2-\allowbreak sft} & 0.1 & 0.1 & 0.9 & \alpacamodel{wizardlm-\allowbreak 13b-\allowbreak v1.2} & 0.2 & 0.2 & 1.8 \\
\alpacamodel{platolm-\allowbreak 7b} & 0.2 & 0.2 & 1.8 & \alpacamodel{wizardlm-\allowbreak 70b} & 0.9 & 0.9 & 8.1 \\
\alpacamodel{pythia-\allowbreak 12b-\allowbreak mix-\allowbreak sft} & 0.1 & 0.1 & 0.9 & \alpacamodel{xwinlm-\allowbreak 13b-\allowbreak v0.1} & 0.2 & 0.2 & 1.8 \\
\alpacamodel{Qwen-\allowbreak 14B-\allowbreak Chat} & 0.2 & 0.2 & 1.8 & \alpacamodel{xwinlm-\allowbreak 70b-\allowbreak v0.1} & 0.9 & 0.9 & 8.1 \\
\alpacamodel{Qwen1.5-\allowbreak 1.8B-\allowbreak Chat} & 0.1 & 0.1 & 0.9 & \alpacamodel{xwinlm-\allowbreak 7b-\allowbreak v0.1} & 0.2 & 0.2 & 1.8 \\
\alpacamodel{Qwen1.5-\allowbreak 110B-\allowbreak Chat} & 0.9 & 0.9 & 8.1 & \alpacamodel{Yi-\allowbreak 34B-\allowbreak Chat} & 0.9 & 0.9 & 8.1 \\
\alpacamodel{Qwen1.5-\allowbreak 14B-\allowbreak Chat} & 0.2 & 0.2 & 1.8 & \alpacamodel{zephyr-\allowbreak 7b-\allowbreak alpha} & 0.2 & 0.2 & 1.8 \\
\alpacamodel{Qwen1.5-\allowbreak 72B-\allowbreak Chat} & 0.1 & 0.1 & 0.9 & \alpacamodel{zephyr-\allowbreak 7b-\allowbreak beta} & 0.2 & 0.2 & 1.8 \\
\bottomrule
\end{tabular*}
\end{table}
\clearpage
\endgroup

%% file: appendices/experiment_results.tex
\section{Additional Experiment Results}
\label{app:additional-experiment-results}

\subsection{Computational Overhead Comparison}
\label{app:computational-overhead}

\paragraph{Timing protocol.}
\edit{Figures~\ref{fig:timing-total-wall-time} and~\ref{fig:timing-total-wall-time-aware} report recorded wall-clock runtime for completed evaluation runs on precomputed response matrices. LLM inference and response-matrix construction are excluded. For the bandit and Bayesian optimization implementations in these figures, the run timer covers data loading, method-specific setup, and policy execution, including allocation, model updates, and recommendation. For GittinsEval, setup includes constructing the stopping-root lookup tables before sequential allocation; this precomputation is included in the total runtime.}
\ifdefined\iclrsysrs
\edit{PromptEval-BAI uses its recorded replay time, which includes fitting and allocation but excludes shared input preparation.}

\paragraph{Baseline configurations.}
Gaussian-TS uses the data-specific prior and posterior-mean recommendation; FormatSpread-TS uses a Beta--Bernoulli model and empirical-mean recommendation. On MMLU, PromptEval-BAI uses PE-EmbPT interaction. On GSM8K and PIQA, PromptEval-BAI uses PE-OneHot.
\fi

\begin{figure}[!htbp]
\centering
\ifdefined\iclrsysrs
\includegraphics[width=\textwidth]{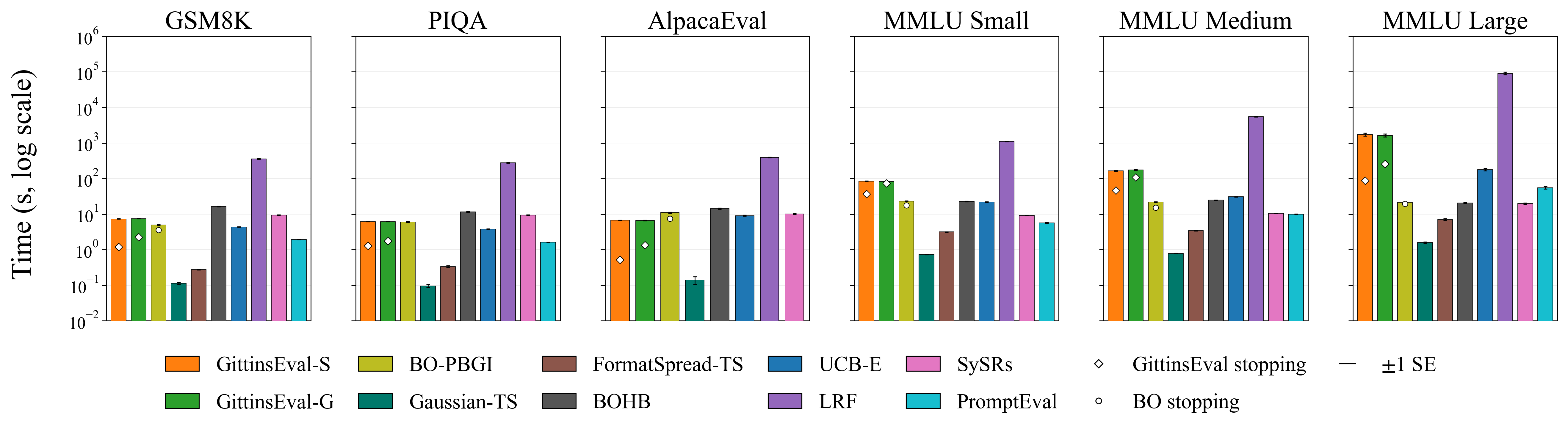}
\else\fi
\caption{
\ifdefined\iclrsysrs
Wall-clock runtime under unit costs for GSM8K, PIQA, AlpacaEval, and the three MMLU size groups. Bars show the arithmetic mean across completed runs, with black error bars indicating $\pm1$ standard error. The vertical axis uses a base-10 logarithmic scale; each major tick interval represents a tenfold increase in runtime. Diamond and circle markers indicate mean estimated stopping times for GittinsEval and PBGI, respectively. GittinsEval, UCB-E, and both TS methods use $B=2,4,8$ on MMLU-small/medium/large and $B=8$ elsewhere where applicable; LRF uses $B=32$. BOHB uses successive halving over example-subsampling fidelities. PromptEval uses PE-OneHot on GSM8K and PIQA and the interaction PE-EmbPT variant on MMLU. PromptEval and FormatSpread-TS are not reported on AlpacaEval.
\else\fi }
\label{fig:timing-total-wall-time}
\end{figure}

\begin{figure}[!htbp]
\centering
\ifdefined\iclrsysrs
\includegraphics[width=\textwidth]{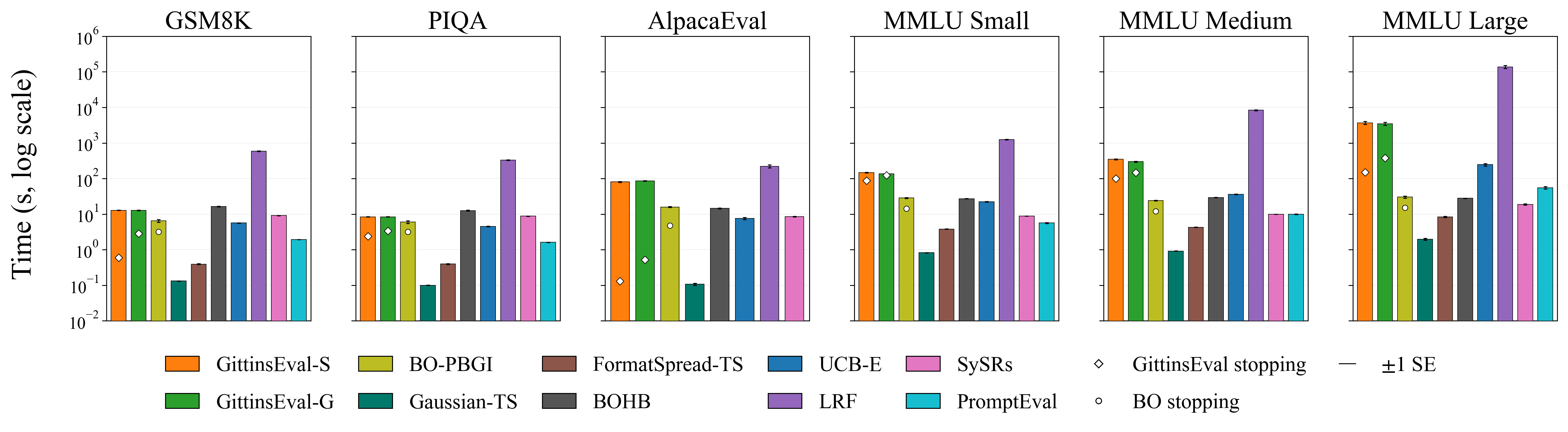}
\else\fi
\caption{
\ifdefined\iclrsysrs
Wall-clock runtime under cost-aware evaluation for GSM8K, PIQA, AlpacaEval, and the three MMLU size groups. Bars show the arithmetic mean across completed runs, with black error bars indicating $\pm1$ standard error. The vertical axis uses a base-10 logarithmic scale; each major tick interval represents a tenfold increase in runtime. Diamond and circle markers indicate mean estimated stopping times for GittinsEval and PBGI, respectively. GittinsEval, UCB-E, and both TS methods use $B=2,4,8$ on MMLU-small/medium/large and $B=8$ elsewhere where applicable; LRF uses $B=32$. BOHB uses successive halving over example-subsampling fidelities. PromptEval uses PE-OneHot on GSM8K and PIQA and the interaction PE-EmbPT variant on MMLU. PromptEval and FormatSpread-TS are not reported on AlpacaEval.
\else\fi }
\label{fig:timing-total-wall-time-aware}
\end{figure}

\FloatBarrier

\subsection{Ablation Studies}
\label{sec:ablation-studies}
We include six ablations to separate the statistical and cost-sensitive components of the Gittins index policy.

\paragraph{Choice of common prior.}
We compare the common priors listed in Table~\ref{tab:prior_settings}. \edit{The G--S comparison assesses sensitivity to empirically informed prior means and variances under the same allocation and recommendation rules. Each prior is shared across all arms; the S-prior uses aggregate response-matrix scores to supply coarse information about the benchmark's performance level. This comparison concerns prior informativeness within a common working model, rather than establishing correct model specification.} The resulting comparison between GittinsEval-G, which uses the general prior, and GittinsEval-S, which uses the data-specific common prior, is reported in Figures~\ref{fig:gsm8k-piqa-main} and~\ref{fig:mmlu-aggregate}; full per-subject MMLU results are provided in Figures~\ref{fig:mmlu-easy-unit-bo5pct}--\ref{fig:mmlu-hard-aware-bo5pct}.

\edit{GittinsEval-G remains broadly competitive across benchmarks using the
same general prior throughout. GittinsEval-S achieves more pronounced gains
on some benchmarks, especially hard MMLU subjects. These results illustrate
how coarse aggregate prior information can further improve performance.}

\begin{table}[!htbp]
\centering
\caption{
Batch-size settings used in the simple-regret experiments. For UCB-E and both Gittins variants, batch size $B$ denotes the number of examples queried for the selected arm at each allocation step. For AlpacaEval, these examples correspond to instruction-level pairwise preference scores rather than binary correctness labels. LRF uses $B=32$ in all reported settings to reduce the runtime overhead of repeated low-rank updates. PBGI evaluates complete configurations, whereas BOHB uses example-subsampling fidelities; neither uses a fixed bandit batch size. For MMLU, subject matrices are grouped by the number of benchmark examples, and the batch-size grid for UCB-E and Gittins is chosen according to this subject-size bucket. }
\label{tab:batch_size_settings}
\small
\setlength{\tabcolsep}{6pt}
\begin{tabular}{llll}
\toprule
Benchmark / group & Matrices & Example-count rule & Batch sizes \\
\midrule
GSM8K & 5 response matrices & 1000 examples each & $B \in \{8, 16, 32\}$ \\
PIQA & 5 response matrices & 1000 examples each & $B \in \{8, 16, 32\}$ \\
AlpacaEval & 1 pairwise score matrix & 805 instructions & $B \in \{8, 16, 32\}$ \\
MMLU small & 22 subject matrices & $n_{\mathrm{examples}} \leq 150$ & $B \in \{2, 4, 8\}$ \\
MMLU medium & 30 subject matrices & $151 \leq n_{\mathrm{examples}} \leq 400$ & $B \in \{4, 8, 16\}$ \\
MMLU large & 5 subject matrices & $n_{\mathrm{examples}} > 400$ & $B \in \{8, 16, 32\}$ \\
\bottomrule
\end{tabular}
\end{table}

\paragraph{Batch size.}
For UCB-E and the Gittins variants, we vary the batch size $B$ used to construct empirical batch-mean observations. Smaller batches provide more frequent adaptation but noisier observations; larger batches reduce Gaussian approximation error and posterior noise but make each allocation decision coarser. We use $B=32$ for LRF to reduce the runtime overhead of repeated low-rank updates. PBGI evaluates complete configurations, whereas BOHB uses example-subsampling fidelities; neither uses a fixed bandit batch size. Table~\ref{tab:batch_size_settings} summarizes the batch-size grids used for GSM8K, PIQA, MMLU, and AlpacaEval. On AlpacaEval, each batch query corresponds to $B$ additional pairwise comparisons for the selected model rather than $B$ independent benchmark examples with binary labels.

Figures~\ref{fig:benchmark-batch-size-sensitivity} and
\ref{fig:mmlu-batch-size-sensitivity} show that batch size primarily affects
the low-budget regime. On GSM8K, PIQA, and AlpacaEval, $B=8$ consistently provides the best early sample efficiency: its simple regret falls more quickly, whereas larger batches can spend more evaluations before the posterior and acquisition policy are updated. The effect is clearest for $B=32$ on unit-cost GSM8K, where early regret is substantially higher. The curves approach one another as the budget grows, suggesting that batch size has a larger effect on early efficiency and stopping cost than on the final recommendation. Differences are smaller under cost-aware evaluation, but $B=8$ still provides the best overall cost efficiency across these benchmarks.

\begin{figure*}[p]
\centering
\includegraphics[width=0.8\textwidth]{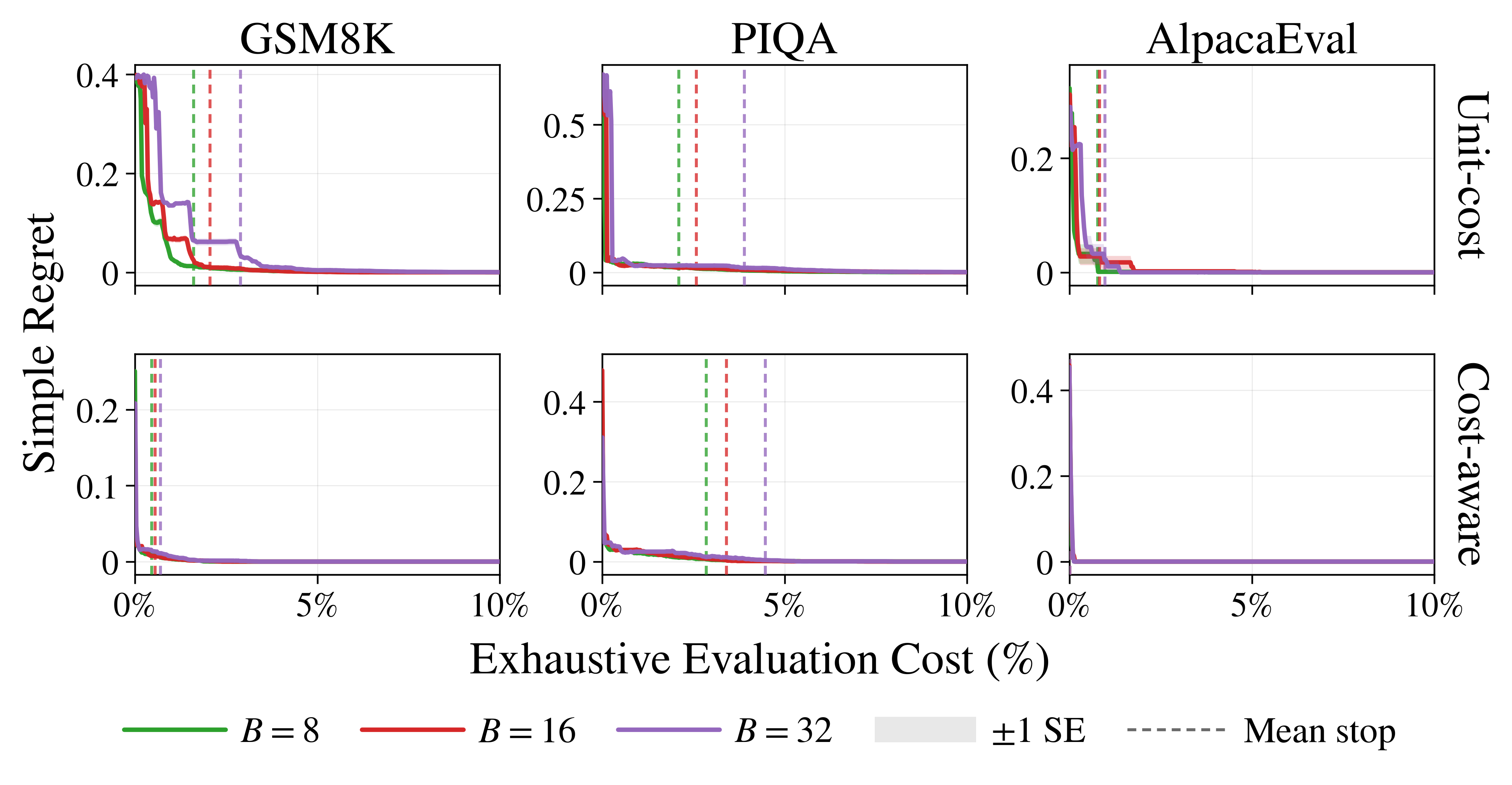}
\caption{
Batch-size sensitivity on GSM8K, PIQA, and AlpacaEval for $B\in\{8,16,32\}$. The top row uses unit costs and the bottom row uses original costs. Curves show mean simple regret against cumulative evaluation cost as a percentage of exhaustive evaluation; bands denote $\pm1$ standard error and dashed lines mark mean natural-stopping cost. GSM8K and PIQA average 100 runs each, and AlpacaEval averages 20 runs. }
\label{fig:benchmark-batch-size-sensitivity}
\end{figure*}

For MMLU, the smallest batch in each size-specific grid likewise tends to have the lowest overall simple regret: $B=2$ for small subjects, $B=4$ for medium subjects, and $B=8$ for large subjects. Larger batches are generally less efficient at low budgets because the policy updates only after an entire batch is completed; on the medium and large groups, they also tend to delay natural stopping. The late-budget convergence indicates qualitative robustness to batch size and supports the size-adaptive default $B=2/4/8$. The stopping lines for the small group require additional care: some configurations stop naturally on only a subset of subjects, so the plotted mean is conditional on the subjects with an observed stop and should not be compared without also considering stopping coverage.

\begin{figure*}[p]
\centering
\includegraphics[width=0.8\textwidth]{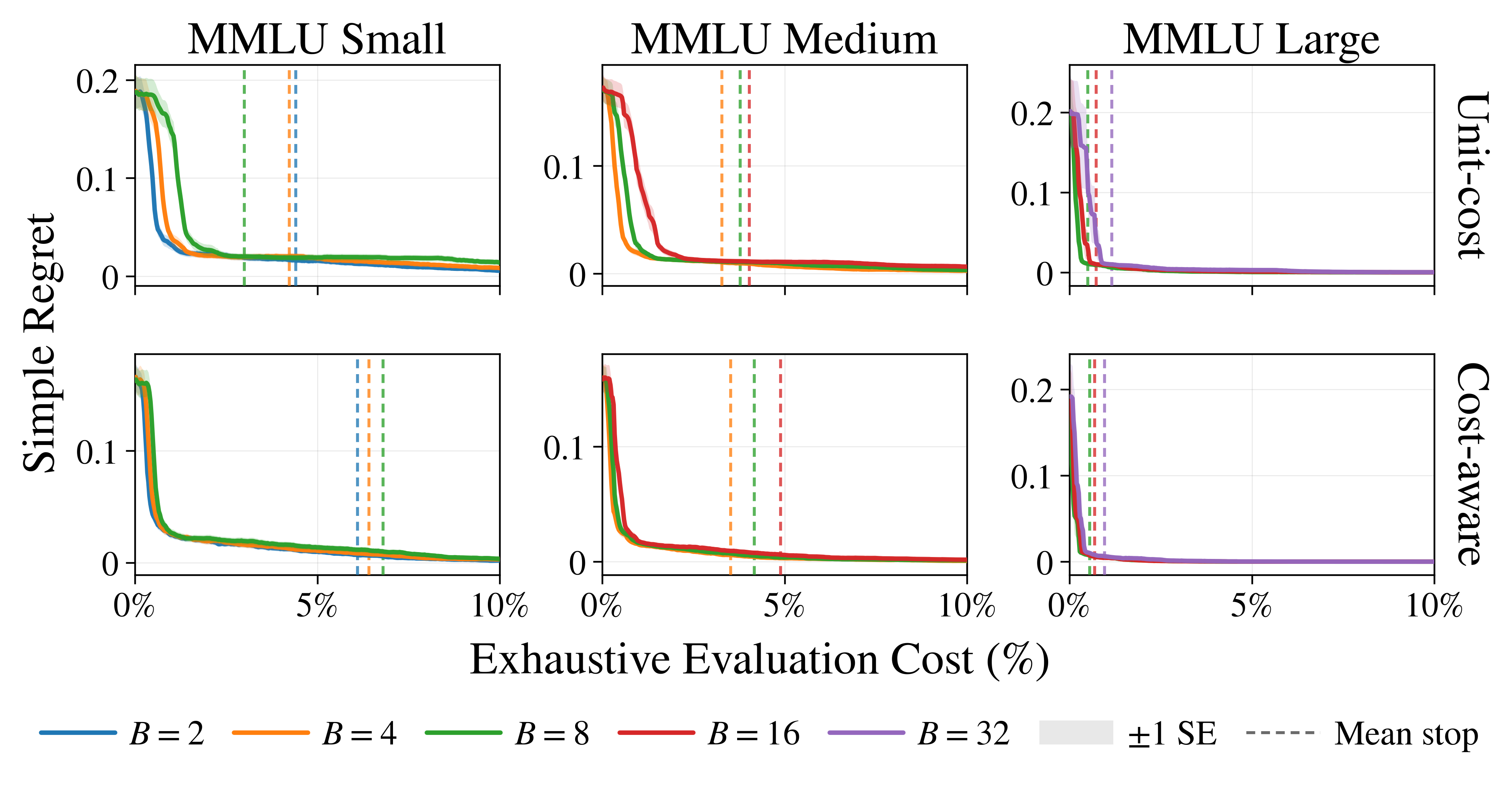}
\caption{
Batch-size sensitivity across the 57 MMLU subjects, grouped as small (22), medium (30), and large (5), with grids $B\in\{2,4,8\}$, $B\in\{4,8,16\}$, and $B\in\{8,16,32\}$. The top row uses unit costs and the bottom row uses original costs. Curves show equally weighted subject means against cumulative evaluation cost as a percentage of exhaustive evaluation; bands denote $\pm1$ subject-level standard error and dashed lines average subjects with observed natural stops. }
\label{fig:mmlu-batch-size-sensitivity}
\end{figure*}

\paragraph{Cost-scaling factor.}
The Gittins index policy uses a cost-scaling factor to map evaluation costs into the continuation penalty in the index computation. In the unit-cost setting, this factor is applied to a common unit cost for each evaluation; in the cost-aware setting, it is applied to the original per-evaluation costs. This parameter controls the effective price of continuing to sample: larger values make additional evaluations more costly and therefore favor earlier stopping, whereas smaller values encourage longer exploration. In our experiments, we sweep three cost-scaling factors, \(10^{-3}\), \(10^{-4}\), and \(10^{-5}\), for both unit-cost and cost-aware Gittins variants. \edit{Unless otherwise specified, the main figures use $\lambda=10^{-4}$, selected using these sweeps on the same benchmarks as a common experimental setting for the cost--quality tradeoff.}

Figure~\ref{fig:benchmark-lambda-sensitivity} confirms the expected stopping-time trade-off on GSM8K, PIQA, and AlpacaEval: increasing $\lambda$ triggers natural stopping earlier, but $\lambda=10^{-3}$ is often too aggressive. Under unit costs it consistently shows higher early simple regret on GSM8K and AlpacaEval. \edit{These fixed-budget curves continue beyond the recorded stopping signal and therefore also reflect the effect of $\lambda$ on allocation.} The regret curves for $10^{-5}$ and $10^{-4}$ are generally close, but $10^{-5}$ usually requires more evaluation cost before stopping. The three settings are less separated in the cost-aware panels. \edit{These results motivated $\lambda=10^{-4}$ as an experimental compromise between regret and stopping cost within the tested range.}

\begin{figure*}[p]
\centering
\includegraphics[width=0.8\textwidth]{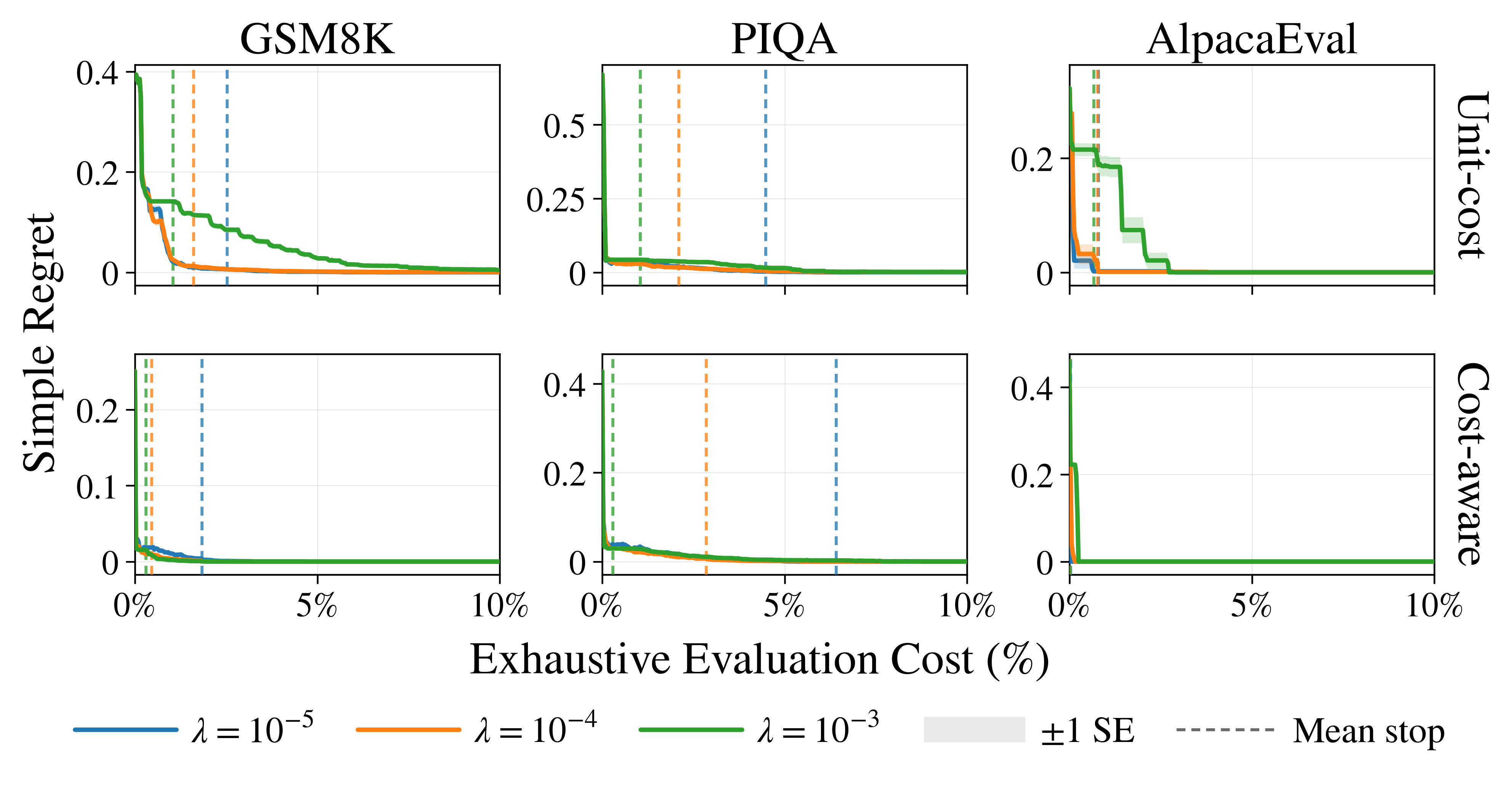}
\caption{
Cost-scaling sensitivity on GSM8K, PIQA, and AlpacaEval for $\lambda\in\{10^{-5},10^{-4},10^{-3}\}$. The top row uses unit costs and the bottom row uses original costs. Curves show mean simple regret against cumulative evaluation cost as a percentage of exhaustive evaluation; bands denote $\pm1$ standard error and dashed lines mark mean natural-stopping cost. GSM8K and PIQA average 100 runs each, and AlpacaEval averages 20 runs. }
\label{fig:benchmark-lambda-sensitivity}
\end{figure*}

The same pattern is visible across the MMLU size groups in Figure~\ref{fig:mmlu-lambda-sensitivity}. The \edit{smaller} $\lambda=10^{-5}$ setting reduces regret more slowly and does not naturally stop within the 10\% budget on some subjects. In contrast, $\lambda=10^{-3}$ stops very early but has consistently higher early regret on the medium and large groups. This behavior is especially pronounced for cost-aware large subjects, \edit{illustrating the tradeoff between an earlier stopping signal and recommendation quality}. Across size groups and cost settings, $\lambda=10^{-4}$ reduces regret quickly while retaining reasonable stopping times. \edit{The slightly lower late-budget regret of $\lambda=10^{-3}$ on the small group also illustrates that the preferred setting depends on the budget and evaluation regime.} Since the large group contains only five subjects and some error bands are wide, we interpret these results as a consistent trend rather than a claim of statistical significance.

\begin{figure*}[p]
\centering
\includegraphics[width=0.8\textwidth]{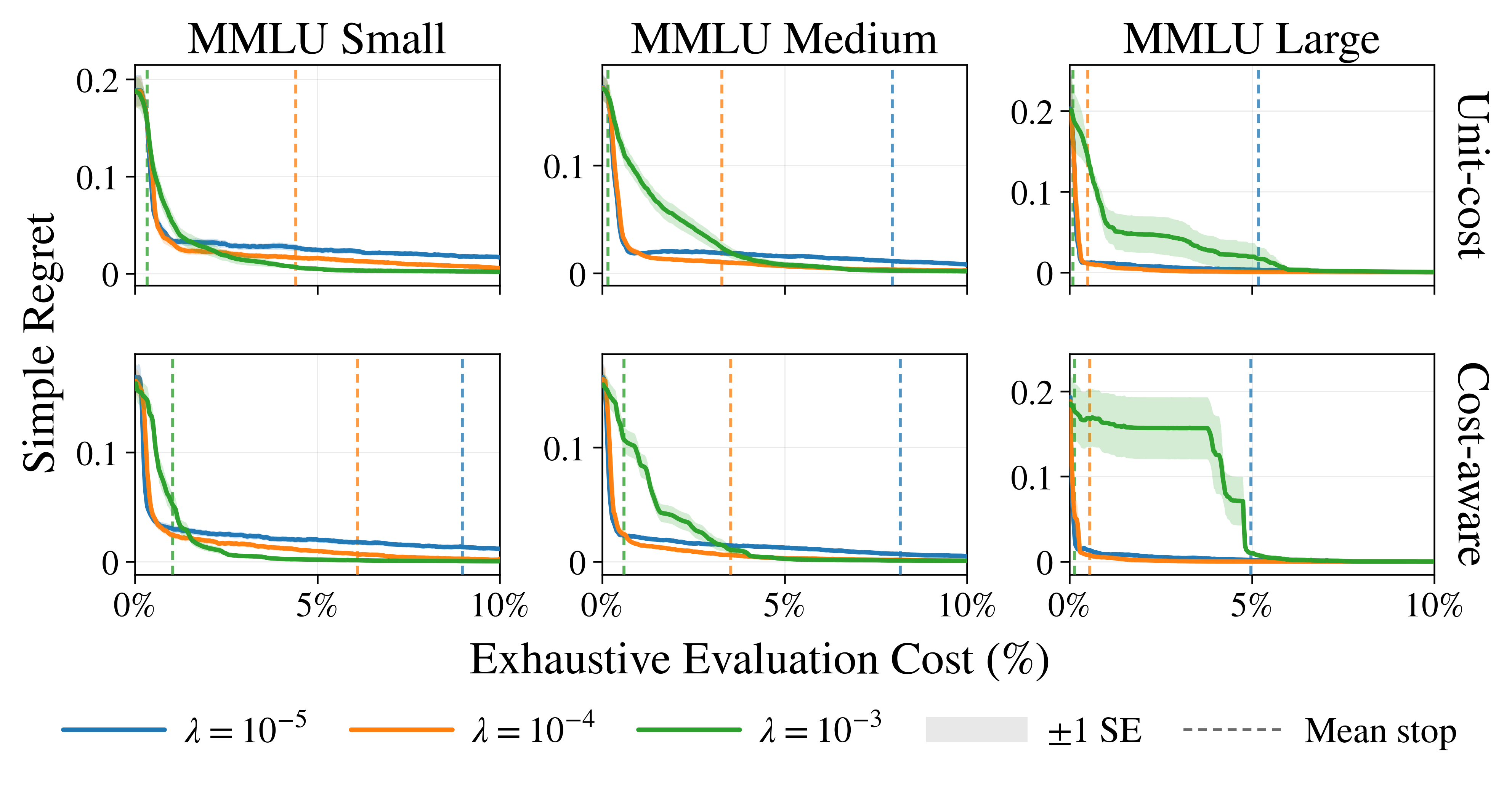}
\caption{
Cost-scaling sensitivity across the 57 MMLU subjects, grouped as small (22), medium (30), and large (5), for $\lambda\in\{10^{-5},10^{-4},10^{-3}\}$. The top row uses unit costs and the bottom row uses original costs. Curves show equally weighted subject means against cumulative evaluation cost as a percentage of exhaustive evaluation; bands denote $\pm1$ subject-level standard error and dashed lines mark mean natural-stopping cost. }
\label{fig:mmlu-lambda-sensitivity}
\end{figure*}

\edit{Figure~\ref{fig:mmlu-lambda-sensitivity-difficulty} groups the same MMLU cost-scaling runs by empirical difficulty (Table~\ref{tab:mmlu-datasets}), retaining the subject-specific priors and size-dependent batch sizes. Across all three difficulty groups, $\lambda=10^{-3}$ produces earlier stopping signals and higher early regret, but can reach lower regret later in the fixed-budget trajectories. The $\lambda=10^{-4}$ setting reduces regret quickly across the groups, whereas $\lambda=10^{-5}$ generally stops later and has higher late-budget regret. These descriptive comparisons reflect the combined subject, prior, and batch-size settings within each group; they do not isolate a causal effect of difficulty or prior choice.}

\begin{figure*}[p]
\centering
\includegraphics[width=0.8\textwidth]{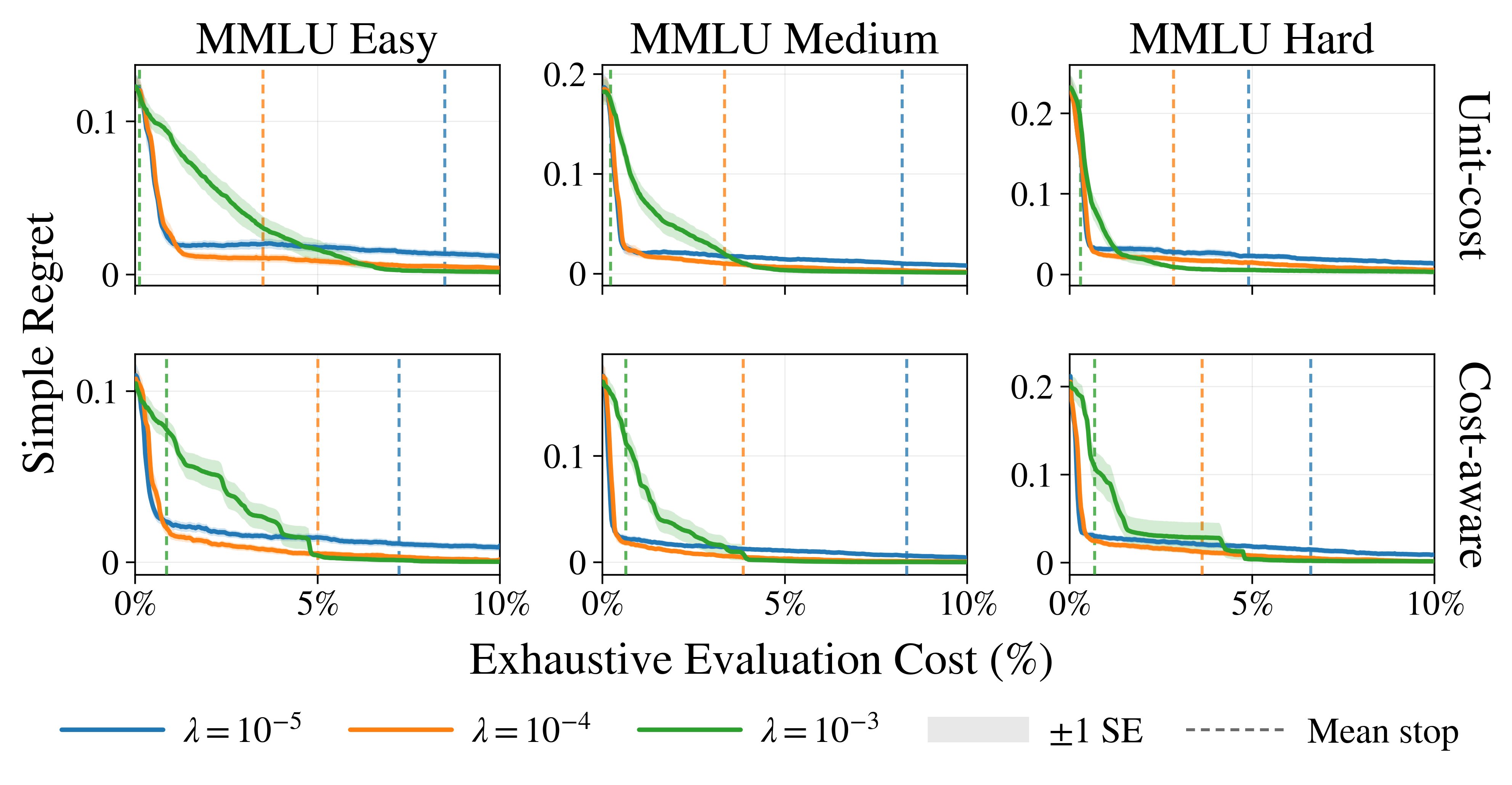}
\caption{\edit{Cost-scaling sensitivity of GittinsEval-S across all 57 MMLU subjects, grouped as easy (17), medium (22), and hard (18) by empirical mean accuracy across configurations. The same runs as Figure~\ref{fig:mmlu-lambda-sensitivity} are regrouped here, with $\lambda\in\{10^{-5},10^{-4},10^{-3}\}$ and size-dependent batches $B=2/4/8$ for small/medium/large subjects. The top row uses unit costs and the bottom row uses original costs. Curves average seeds within subjects and then weight subjects equally; shaded bands denote $\pm1$ standard error across subject means. Dashed vertical lines average subject-level natural-stopping costs over subjects with observed stops within the budget. The trajectories continue beyond the stopping signal to show fixed-budget regret.}}
\label{fig:mmlu-lambda-sensitivity-difficulty}
\end{figure*}

\edit{The reported batch sizes and $\lambda$ are experimental choices informed by these sensitivity studies on the evaluation benchmarks. The cost-scaling parameter $\lambda$ controls the tradeoff between evaluation cost and recommendation quality; the reported $1\%\text{--}10\%$ stopping range is an empirical observation under the stated settings.} Batch-size sensitivity is likewise modest beyond the low-budget regime: smaller batches generally provide better early sample efficiency through more frequent adaptation, whereas larger batches require fewer sequential posterior updates and allocation decisions and can therefore reduce policy-side overhead. Accordingly, the main-text plots use the smaller, empirically stronger setting for each reported group: $B=8$ for GSM8K, PIQA, and AlpacaEval, $B=2$ for MMLU-small, and $B=8$ for MMLU-large. The full MMLU results use the size-adaptive configuration $B=2/4/8$ for small, medium, and large subjects, respectively.
\FloatBarrier

\paragraph{Choice of recommendation rule.}
We compare the unpenalized posterior-mean recommendation
\[
\widehat{k}^{\mathrm{mean}}_t
=
\arg\max_k M_{k,t}
\]
with the LCB-style recommendation for GittinsEval-G and \edit{Gaussian Thompson sampling (Gaussian-TS)},
\[
\widehat{k}^{\mathrm{LCB}}_t
=
\arg\max_k
\left\{
M_{k,t}-\sqrt{V_{k,t}}
\right\}.
\]
Here, $M_{k,t}$ and $V_{k,t}$ are respectively the posterior mean and variance of arm $k$'s complete, fixed evaluation-row mean. Within each allocation method, both recommendation rules are evaluated on the same trajectories. \edit{Gaussian-TS uses the same general prior and batch size as GittinsEval-G.} The Gittins stopping time is shared by its two recommendation rules \edit{and does not apply to Gaussian-TS.}

Figure~\ref{fig:recommendation-rule-ablation} compares the recommendation rules under cost-aware evaluation on GSM8K and MMLU Hard-Large. The LCB-style recommendation substantially reduces simple regret for GittinsEval-G in both settings. \edit{The Gaussian-TS curves are closer, with the most visible differences occurring early on GSM8K.} Because the allocation trajectories are fixed within each method, these differences reflect the recommendation rule rather than changes in evaluation allocation.

The effect on simple regret varies across individual MMLU subjects.
Appendix~\ref{sec:lcb-motivation} reports fewer recommendation changes with the
LCB-style rule on the same GSM8K trajectory for each prior.

\begin{figure*}[htbp]
\centering
\includegraphics[width=0.8\linewidth]{
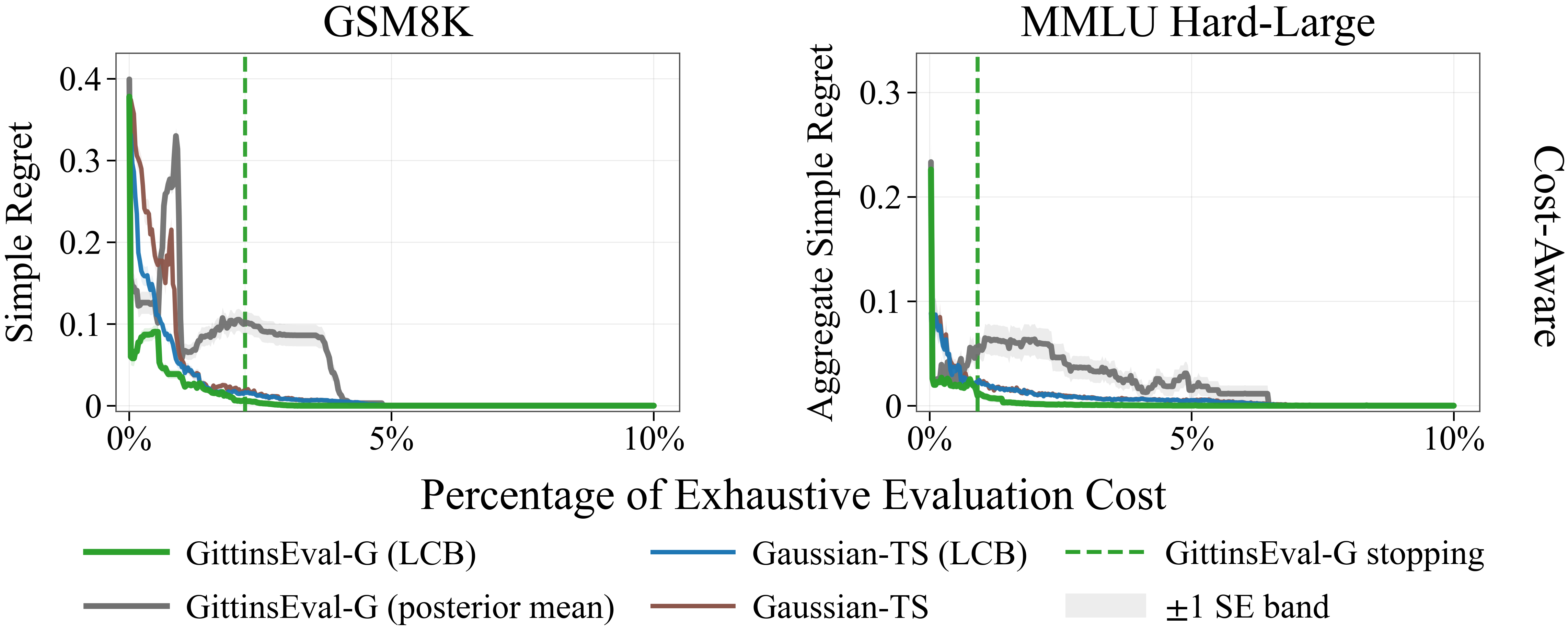 }
\caption{
Recommendation-rule ablation for GittinsEval-G and \edit{Gaussian-TS} under cost-aware evaluation. The left panel shows GSM8K (100 randomized runs per method); the right panel shows MMLU Hard-Large (40 subject--run pairs per method, with 20 runs for each of two subjects). Within each method, the posterior-mean and LCB-style recommendations use the same allocation trajectories. \edit{Gaussian-TS denotes posterior-mean recommendation; Gaussian-TS (LCB) uses our LCB-style recommendation rule.} Both panels use actual cumulative evaluation cost as a percentage of exhaustive-evaluation cost, with trajectories linearly interpolated onto a shared 350-point grid. Curves show mean simple regret (left) and aggregate raw simple regret (right); bands denote $\pm1$ standard error. Dashed vertical lines mark the mean GittinsEval-G stopping time, shared by its two recommendation rules. }
\label{fig:recommendation-rule-ablation}
\end{figure*}

\edit{To assess how recommendation choices affect the comparison with FormatSpread-TS, we compare the methods separately under data-specific (S) and general (G) priors (Figures~\ref{fig:recommendation-mmlu-s}--\ref{fig:recommendation-benchmarks-g}). FormatSpread-TS-S is the baseline used in the main experiments. FormatSpread-TS-G uses the same Beta--Bernoulli implementation with prior mean $x=0.5$, giving $\operatorname{Beta}(5,5)$; GittinsEval-G uses its Gaussian prior with mean $0.5$ and variance $0.04$. The S variants use the data-specific prior means in Table~\ref{tab:prior_settings}. These comparisons match the prior category, not the prior family or variance: for example, the general Beta prior has variance $1/44$, rather than $0.04$. Across methods, response matrices, seed sets, nominal batch sizes, cost settings, and budget limits are matched.}

\edit{Within each method and prior setting, we replay the same observations under empirical-mean, posterior-mean, and LCB-style recommendations, holding allocation decisions, observations, priors, batch sizes, and budgets fixed. Empirical-mean recommendation selects the highest observed sample mean among evaluated arms. For FormatSpread-TS, the LCB score is the mean minus one standard deviation of its native Beta posterior: with $a_{k,t}=\alpha_0+s_{k,t}$ and $b_{k,t}=\beta_0+n_{k,t}-s_{k,t}$, where $s_{k,t}$ is the number of observed successes, the score is}
\[
\edit{\frac{a_{k,t}}{a_{k,t}+b_{k,t}}-
\sqrt{\frac{a_{k,t}b_{k,t}}{(a_{k,t}+b_{k,t})^2(a_{k,t}+b_{k,t}+1)}}.}
\]
\edit{Posterior-mean recommendation uses $a_{k,t}/(a_{k,t}+b_{k,t})$ for FormatSpread-TS and $M_{k,t}$ for GittinsEval. The GittinsEval LCB-style score remains $M_{k,t}-\sqrt{V_{k,t}}$. Comparing posterior-mean and LCB-style recommendations isolates the effect of the posterior-standard-deviation penalty within each method and prior setting. Posterior-based recommendations consider all arms. Cross-method comparisons still include differences in allocation and Bayesian models, even when the prior category and recommendation rule agree.}

\edit{The replay covers all 57 MMLU subjects and all five response matrices each for GSM8K and PIQA, with 20 seeds per matrix, both cost settings, and both prior categories for each method. AlpacaEval is excluded because FormatSpread-TS assumes binary responses. Figures~\ref{fig:recommendation-mmlu-s} and~\ref{fig:recommendation-mmlu-g} display the same 19 subjects and four size--difficulty groups as the main MMLU figure. All curves use 1,001 budget points from 0 to 10\%, taking the latest completed observation at or before each point, without backfilling before the first observation. Corresponding S and G panels share axis ranges.}

\edit{The comparison separates the benefit of posterior-mean recommendation over empirical means from the additional effect of the LCB penalty. For both GittinsEval and FormatSpread-TS, empirical-mean recommendations can exhibit pronounced regret spikes or remain at high regret, especially on MMLU. Posterior-mean recommendation substantially reduces regret and fluctuations in many of these panels under both prior categories. This is consistent with posterior shrinkage tempering high empirical means supported by few observations. For FormatSpread-TS, this improvement is much smaller on GSM8K and PIQA, which have far fewer candidates than MMLU.}

\edit{LCB's additional benefit over posterior-mean recommendation is more localized. It is clearest for GittinsEval-G under cost-aware evaluation on GSM8K and MMLU Hard--Large (Figures~\ref{fig:recommendation-rule-ablation}, \ref{fig:recommendation-mmlu-g}, and~\ref{fig:recommendation-benchmarks-g}). Smaller gains are visible elsewhere, including cost-aware MMLU Hard--Small with GittinsEval-G. For GittinsEval-S and FormatSpread-TS, posterior-mean and LCB-style curves are generally close. Thus, much of the improvement over empirical-mean recommendation comes from using posterior means, while the uncertainty penalty provides a further, setting-dependent benefit.}

\edit{Under common recommendation rules, GittinsEval-S retains advantages over FormatSpread-TS-S in several settings, particularly cost-aware GSM8K and MMLU Hard--Small. The G-prior comparison is more mixed: FormatSpread-TS-G with posterior-based recommendations is competitive and can outperform GittinsEval-G on MMLU Easy--Small and Easy--Large, whereas GittinsEval-G with LCB reaches low regret faster on cost-aware Hard--Large. These results show that the original performance gap reflects recommendation, prior, and allocation choices together; the following Gaussian-TS comparison further controls the Bayesian model when examining allocation.}

\begin{figure*}[p]
\centering
\includegraphics[width=\textwidth]{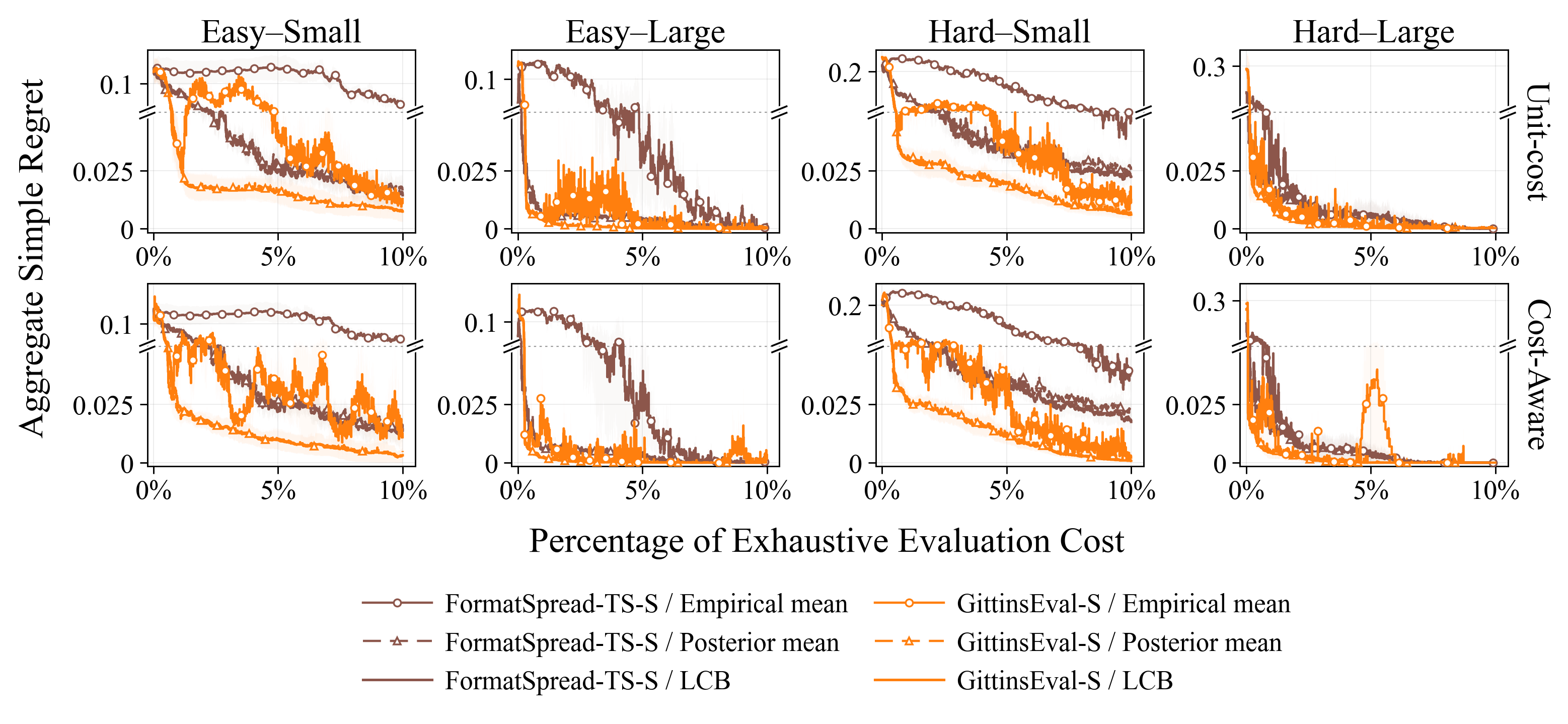}
\caption{\edit{Recommendation-rule ablation with data-specific (S) priors on the four MMLU groups in Figure~\ref{fig:mmlu-aggregate}: Easy--Small (6 subjects), Easy--Large (2), Hard--Small (9), and Hard--Large (2). Each panel compares FormatSpread-TS-S (brown) and GittinsEval-S (orange) under three recommendation rules, with trajectories fixed within each method. The top row uses unit costs and the bottom row uses original costs. Thin solid curves with sparse hollow circles use empirical means, long-dashed curves with sparse hollow triangles use posterior means, and unmarked solid curves use LCB-style scores. Curves average 20 seeds within each subject and then weight subjects equally; shaded bands denote $\pm1$ standard error across subject means. The vertical axis expands $[0,0.05]$ to 65\% of positive panel height using two linear scales, with no interval omitted.}}
\label{fig:recommendation-mmlu-s}
\end{figure*}

\begin{figure*}[p]
\centering
\includegraphics[width=\textwidth]{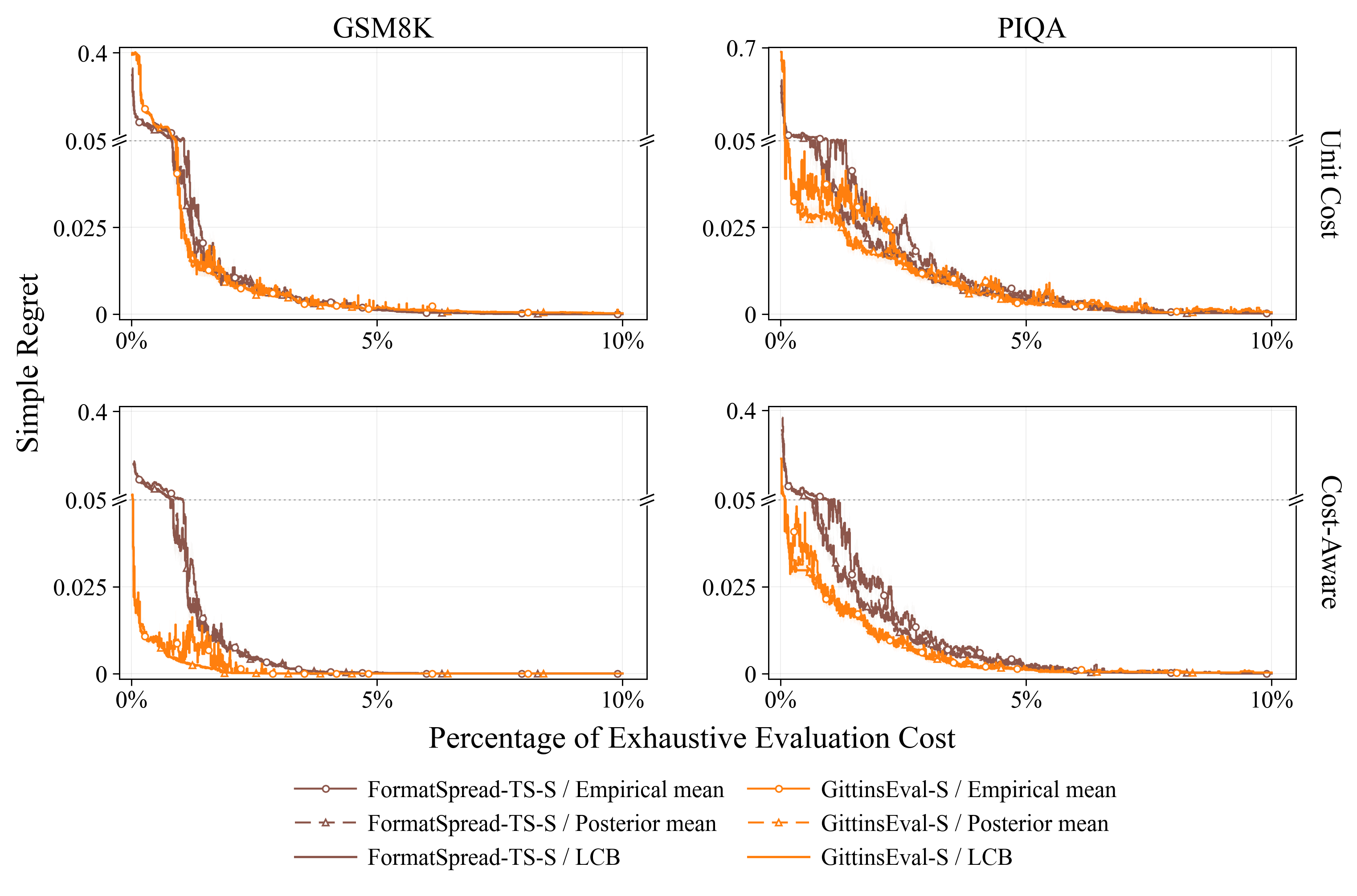}
\caption{\edit{Recommendation-rule ablation with data-specific (S) priors on GSM8K and PIQA: FormatSpread-TS-S (brown) and GittinsEval-S (orange), under unit costs (top) and original costs (bottom). Each method averages 100 trajectories per benchmark and cost setting (five matrices, 20 seeds each); shaded bands denote $\pm1$ standard error across trajectories. Thin solid curves with sparse hollow circles use empirical means, long-dashed curves with sparse hollow triangles use posterior means, and unmarked solid curves use LCB-style scores. Allocation and observations are fixed within each method. Axis quantities and the expanded lower regret range follow Figure~\ref{fig:recommendation-mmlu-s}.}}
\label{fig:recommendation-benchmarks-s}
\end{figure*}

\begin{figure*}[p]
\centering
\includegraphics[width=\textwidth]{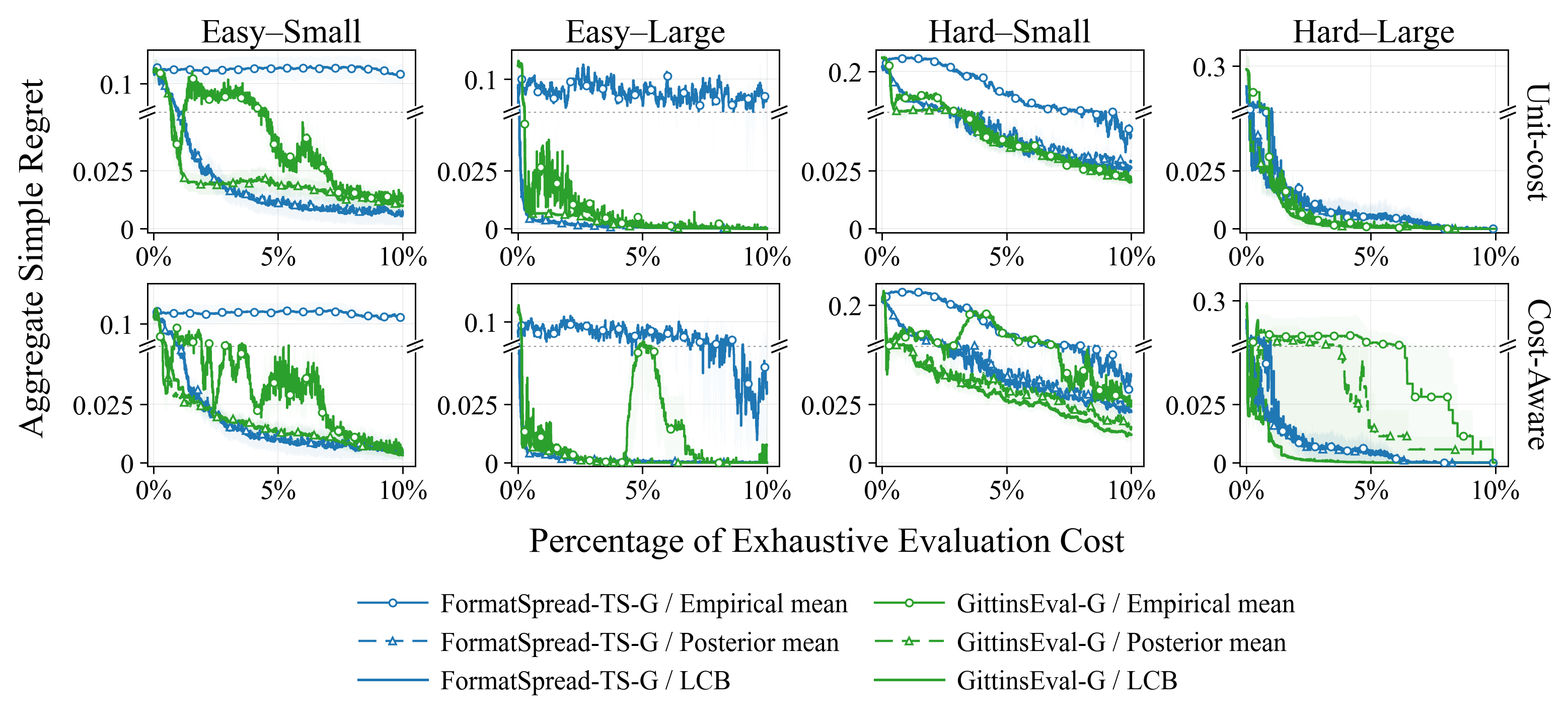}
\caption{\edit{Recommendation-rule ablation with general (G) priors on the four MMLU groups in Figure~\ref{fig:mmlu-aggregate}: Easy--Small (6 subjects), Easy--Large (2), Hard--Small (9), and Hard--Large (2). Each panel compares FormatSpread-TS-G (blue) and GittinsEval-G (green) under three recommendation rules, with trajectories fixed within each method. The top row uses unit costs and the bottom row uses original costs. Thin solid curves with sparse hollow circles use empirical means, long-dashed curves with sparse hollow triangles use posterior means, and unmarked solid curves use LCB-style scores. Curves average 20 seeds within each subject and then weight subjects equally; shaded bands denote $\pm1$ standard error across subject means. The vertical axis expands $[0,0.05]$ to 65\% of positive panel height using two linear scales, with no interval omitted.}}
\label{fig:recommendation-mmlu-g}
\end{figure*}

\begin{figure*}[p]
\centering
\includegraphics[width=\textwidth]{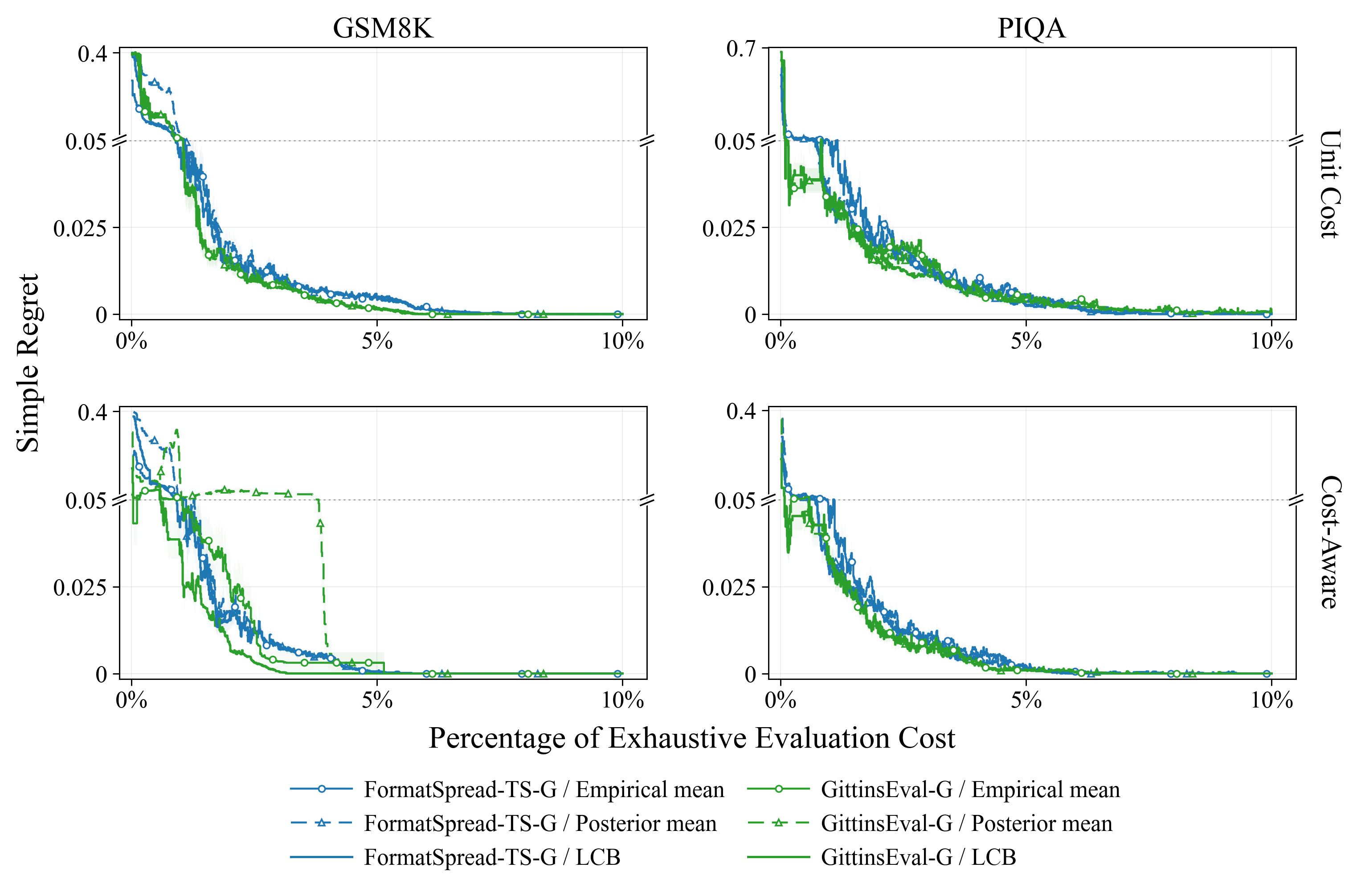}
\caption{\edit{Recommendation-rule ablation with general (G) priors on GSM8K and PIQA: FormatSpread-TS-G (blue) and GittinsEval-G (green), under unit costs (top) and original costs (bottom). Each method averages 100 trajectories per benchmark and cost setting (five matrices, 20 seeds each); shaded bands denote $\pm1$ standard error across trajectories. Thin solid curves with sparse hollow circles use empirical means, long-dashed curves with sparse hollow triangles use posterior means, and unmarked solid curves use LCB-style scores. Allocation and observations are fixed within each method. Axis quantities and the expanded lower regret range follow Figure~\ref{fig:recommendation-mmlu-g}.}}
\label{fig:recommendation-benchmarks-g}
\end{figure*}

\FloatBarrier

\paragraph{Effect of Allocation Strategies and Cross-Configuration Correlations.}
We compare four implementations in a $2\times2$ design crossing allocation strategy (Gittins or Gaussian Thompson sampling) with the incorporation of cross-configuration correlations (independent or correlated): GittinsEval-S, Corr-GittinsEval-S, Gaussian-TS-S, and Corr-Gaussian-TS-S. All four use the same data-specific prior settings and LCB-style recommendation rule $\widehat{k}_t=\arg\max_k\{M_{k,t}-\sqrt{V_{k,t}}\}$. Within each benchmark and cost setting, we match response matrices, seeds, batch sizes, and evaluation budgets. This comparison examines allocation strategies within each correlation setting and correlation within each allocation strategy, while fixing the prior choice and recommendation rule.

\edit{Both correlated variants use a joint Gaussian model for the latent configuration means, with a categorical kernel on configuration features. Evaluating one configuration can therefore update beliefs about other configurations. Corr-Gaussian-TS-S draws full-benchmark scores jointly from the posterior. Corr-GittinsEval-S computes each index from the current marginal posterior using a stopping problem that considers future evaluations of that configuration alone.}

\edit{Figures~\ref{fig:correlation-mmlu} and~\ref{fig:correlation-benchmarks} show that the effect of correlation varies across datasets and allocation methods. Correlation generally reduces mean regret for Gaussian-TS in all three MMLU size groups and helps GittinsEval most on MMLU subjects with fewer examples. On PIQA and unit-cost AlpacaEval, correlation slows regret reduction for both methods. On unit-cost GSM8K, it slows GittinsEval while the Gaussian-TS curves remain close. Other kernels or allocation rules could yield different results.}

Figures~\ref{fig:allocation-correlation-timing-unit} and~\ref{fig:allocation-correlation-timing-cost} compare recorded runtime for the four variants, with UCB-E and LRF as references. Correlated Gaussian-TS takes substantially longer than independent Gaussian-TS in every reported group. Corr-GittinsEval is slower than GittinsEval on small and medium MMLU subjects, but their runtime ordering varies on the other groups. Thus, exploiting correlation can increase computational overhead without consistently improving recommendation quality.

\begin{figure*}[!htbp]
\centering
\includegraphics[width=0.8\textwidth]{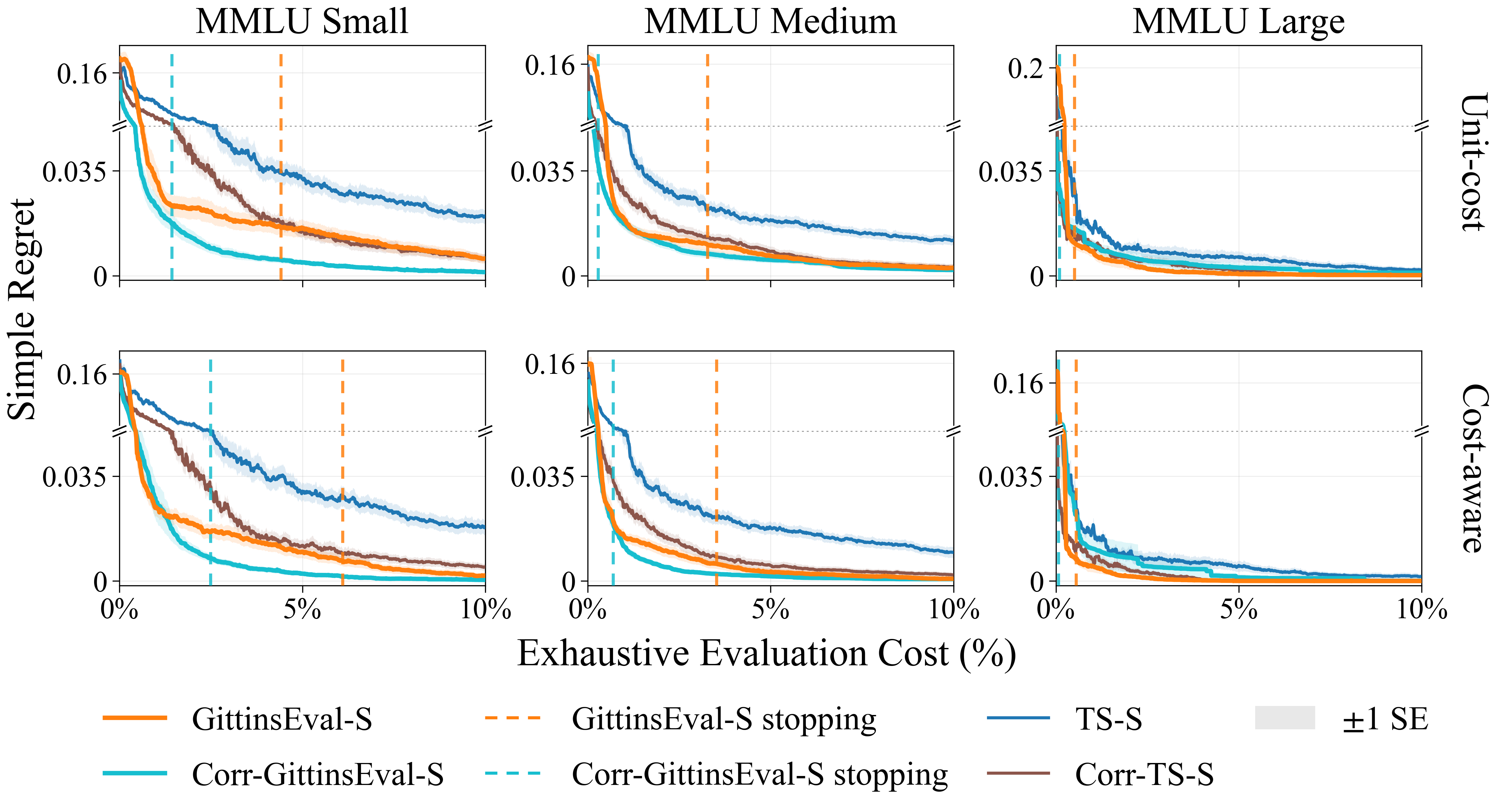}
\caption{
Correlation ablation across the 57 MMLU subjects, grouped as small (22), medium (30), and large (5). All methods use data-specific priors and LCB-style recommendations; TS denotes Gaussian Thompson sampling. The top row uses unit costs and the bottom row uses original costs. Each method uses the same 18--20 matched seeds per subject and cost setting. Curves show equally weighted subject means of simple regret against cumulative evaluation cost as a percentage of exhaustive evaluation; bands denote $\pm1$ subject-level standard error. Dashed vertical lines mark subject-averaged natural-stopping locations for the Gittins variants, conditional on an observed stop within the evaluation budget. The vertical axis uses two linear scales, expanding regret values below 0.05; axis marks indicate the scale transition, with no interval omitted. }
\label{fig:correlation-mmlu}
\end{figure*}

\begin{figure*}[!htbp]
\centering
\includegraphics[width=0.8\textwidth]{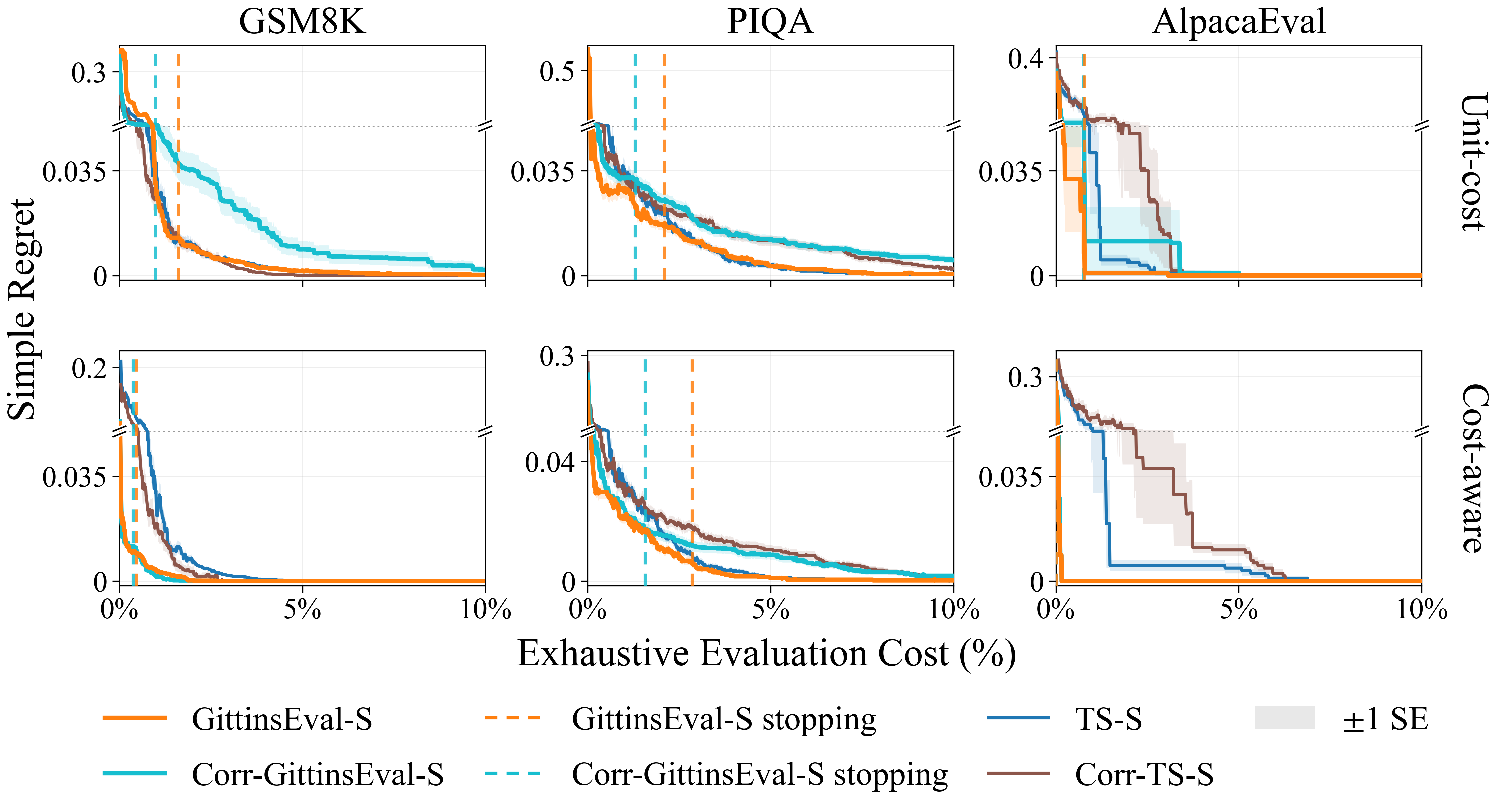}
\caption{
Correlation ablation on GSM8K, PIQA, and AlpacaEval. All methods use data-specific priors and LCB-style recommendations; TS denotes Gaussian Thompson sampling. The top row uses unit costs and the bottom row uses original costs. GSM8K and PIQA each average 100 matched matrix--seed runs across five response matrices; AlpacaEval averages 20 matched runs on the 152-model matrix. Curves show mean simple regret against cumulative evaluation cost as a percentage of exhaustive evaluation; bands denote $\pm1$ standard error across runs. Dashed vertical lines mark mean natural-stopping locations for the Gittins variants, conditional on an observed stop within the evaluation budget. The vertical axis uses two linear scales, expanding regret values below 0.05; axis marks indicate the scale transition, with no interval omitted. }
\label{fig:correlation-benchmarks}
\end{figure*}

\begin{figure*}[!htbp]
\centering
\includegraphics[width=\textwidth]{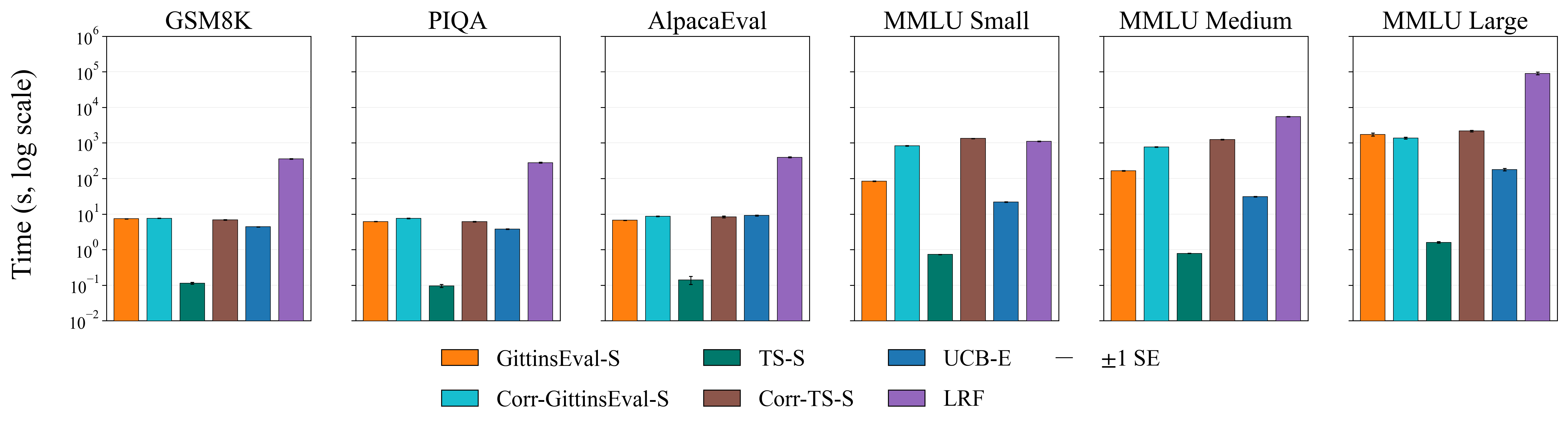}
\caption{Recorded runtime under unit-cost evaluation for GittinsEval-S, Corr-GittinsEval-S, Gaussian-TS-S, and Corr-Gaussian-TS-S, with UCB-E and LRF as references. The four variants use the same matched runs as Figures~\ref{fig:correlation-mmlu} and~\ref{fig:correlation-benchmarks}; UCB-E and LRF use their historical completed-run statistics. Bars show arithmetic means across runs, with error bars denoting $\pm1$ standard error. The vertical axis uses a base-10 logarithmic scale; each major tick interval represents a tenfold increase in runtime. TS denotes Gaussian Thompson sampling. \edit{As in Appendix~\ref{app:computational-overhead}, LLM inference and response-matrix construction are excluded.}}
\label{fig:allocation-correlation-timing-unit}
\end{figure*}

\begin{figure*}[!htbp]
\centering
\includegraphics[width=\textwidth]{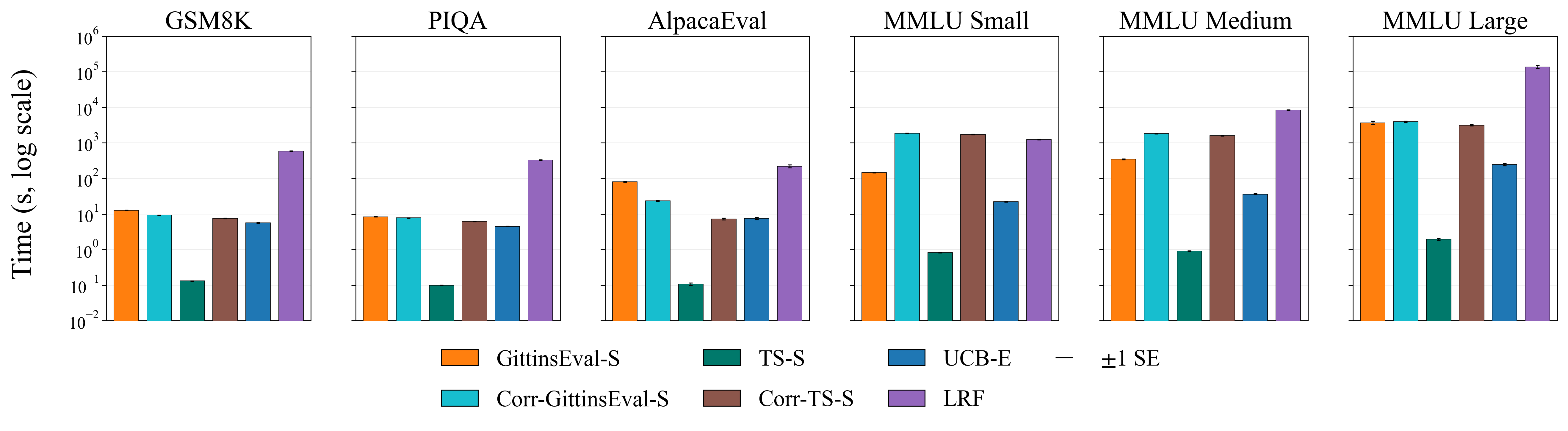}
\caption{Recorded runtime under cost-aware evaluation for GittinsEval-S, Corr-GittinsEval-S, Gaussian-TS-S, and Corr-Gaussian-TS-S, with UCB-E and LRF as references. The four variants use the same matched runs as Figures~\ref{fig:correlation-mmlu} and~\ref{fig:correlation-benchmarks}; UCB-E and LRF use their historical completed-run statistics. Bars show arithmetic means across runs, with error bars denoting $\pm1$ standard error. The vertical axis uses a base-10 logarithmic scale; each major tick interval represents a tenfold increase in runtime. TS denotes Gaussian Thompson sampling. \edit{As in Appendix~\ref{app:computational-overhead}, LLM inference and response-matrix construction are excluded.}}
\label{fig:allocation-correlation-timing-cost}
\end{figure*}

\FloatBarrier

\paragraph{Choice of PromptEval representation.}
We compare PE-OneHot with the concatenation and interaction PE-EmbPT adaptations, using GittinsEval-G and GittinsEval-S as references (Figure~\ref{fig:prompteval-representation}). Concatenation shares prompt effects across models, whereas interaction permits model-specific prompt effects. Both PE-EmbPT variants use balanced question sampling; PE-OneHot uses its native random-question sampler. This comparison evaluates the implemented variants rather than isolating the representation alone. Concatenation has lower mean simple regret than interaction at the five phase boundaries in all three size groups, with the largest gaps on small subjects.

\begin{figure*}[!htbp]
\centering
\includegraphics[width=0.8\textwidth]{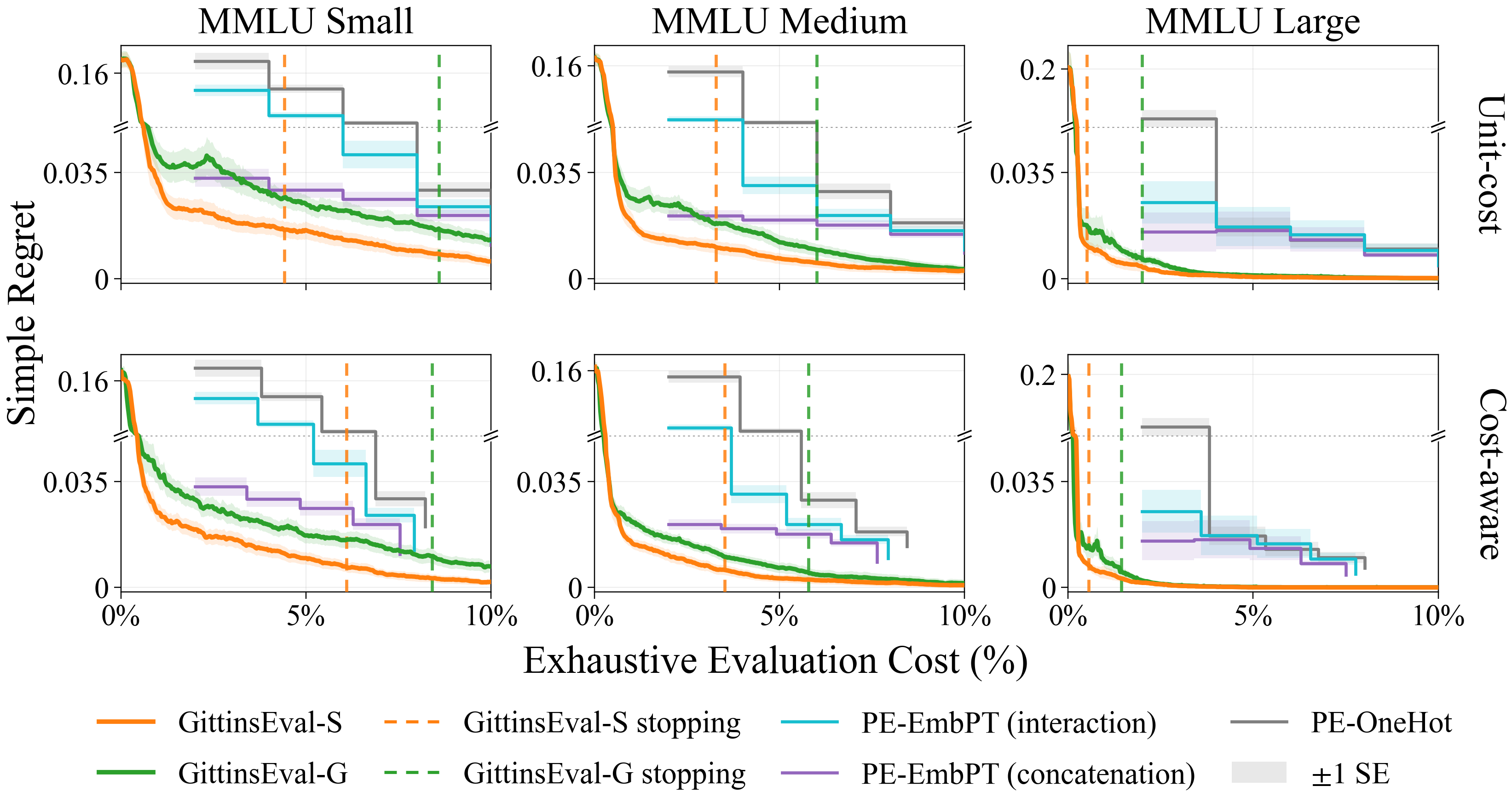}
\caption{
PromptEval representation comparison across the 57 MMLU subjects, grouped as small (22), medium (30), and large (5), with 20 randomized runs per method and subject. The top row shows unit-cost results and the bottom row shows cost-aware results; the PromptEval trajectories are the same in both rows, with their recorded costs used for the cost-aware view. Curves show equally weighted subject means of simple regret; bands denote $\pm1$ subject-level standard error. PromptEval curves connect the five successive-halving phase means with post-step lines at their mean actual evaluation costs. Both GittinsEval variants use LCB-style recommendations. Dashed vertical lines mark subject-averaged natural-stopping locations, conditional on an observed stop within the evaluation budget. The vertical axis expands regret values below 0.05 using two linear scales; axis marks indicate the scale transition, with no interval omitted. }
\label{fig:prompteval-representation}
\end{figure*}
\FloatBarrier

\subsection{Per-Subject MMLU Results}

We provide additional MMLU results under informative priors in Figures~\ref{fig:mmlu-easy-unit-bo5pct}--\ref{fig:mmlu-hard-aware-bo5pct}. \edit{Subjects are grouped as easy, medium, or hard by empirical mean accuracy across configurations (Table~\ref{tab:mmlu-datasets}).} The corresponding prior means are $\mu_0=0.75$, $\mu_0=0.6$, and $\mu_0=0.4$, respectively. For each bucket, we report both unit-cost and cost-aware results, with the horizontal axis showing cumulative evaluation cost as a percentage of the exhaustive-evaluation cost. Across the MMLU subjects, GittinsEval-S generally reaches lower simple regret earlier than UCB-E and GittinsEval-G across difficulty buckets. \edit{These results illustrate the potential value of coarse aggregate prior information for MMLU, especially on the empirically harder subjects.}

\begin{figure}[htbp]
\centering
\ifdefined\iclrsysrs
\includegraphics[width=\textwidth]{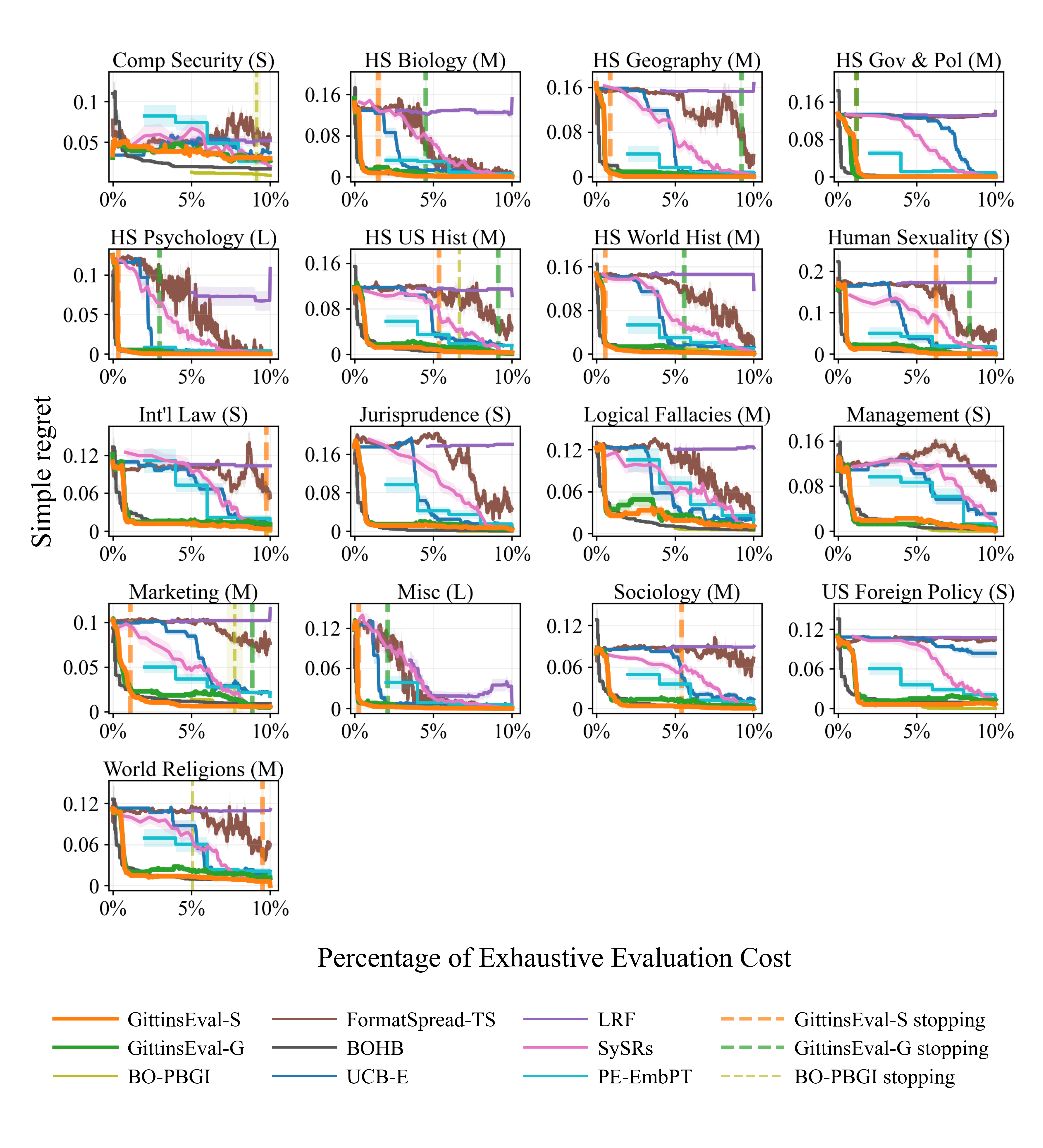}
\else\fi
\caption{
Per-subject MMLU unit-cost results for easy subjects (high-prior bucket). Simple regret is plotted against cumulative evaluation cost as a percentage of exhaustive evaluation. Labels S, M, and L denote matrix size. GittinsEval, UCB-E, and FormatSpread-TS use $B=2,4,8$ for these sizes, respectively; LRF uses $B=32$, and BOHB uses example-subsampling fidelities. Bands denote $\pm1$ standard error, and dashed lines with faint bands mark mean stopping times with $\pm1$ standard error. }
\label{fig:mmlu-easy-unit-bo5pct}
\end{figure}

\begin{figure}[htbp]
\centering
\ifdefined\iclrsysrs
\includegraphics[width=\textwidth]{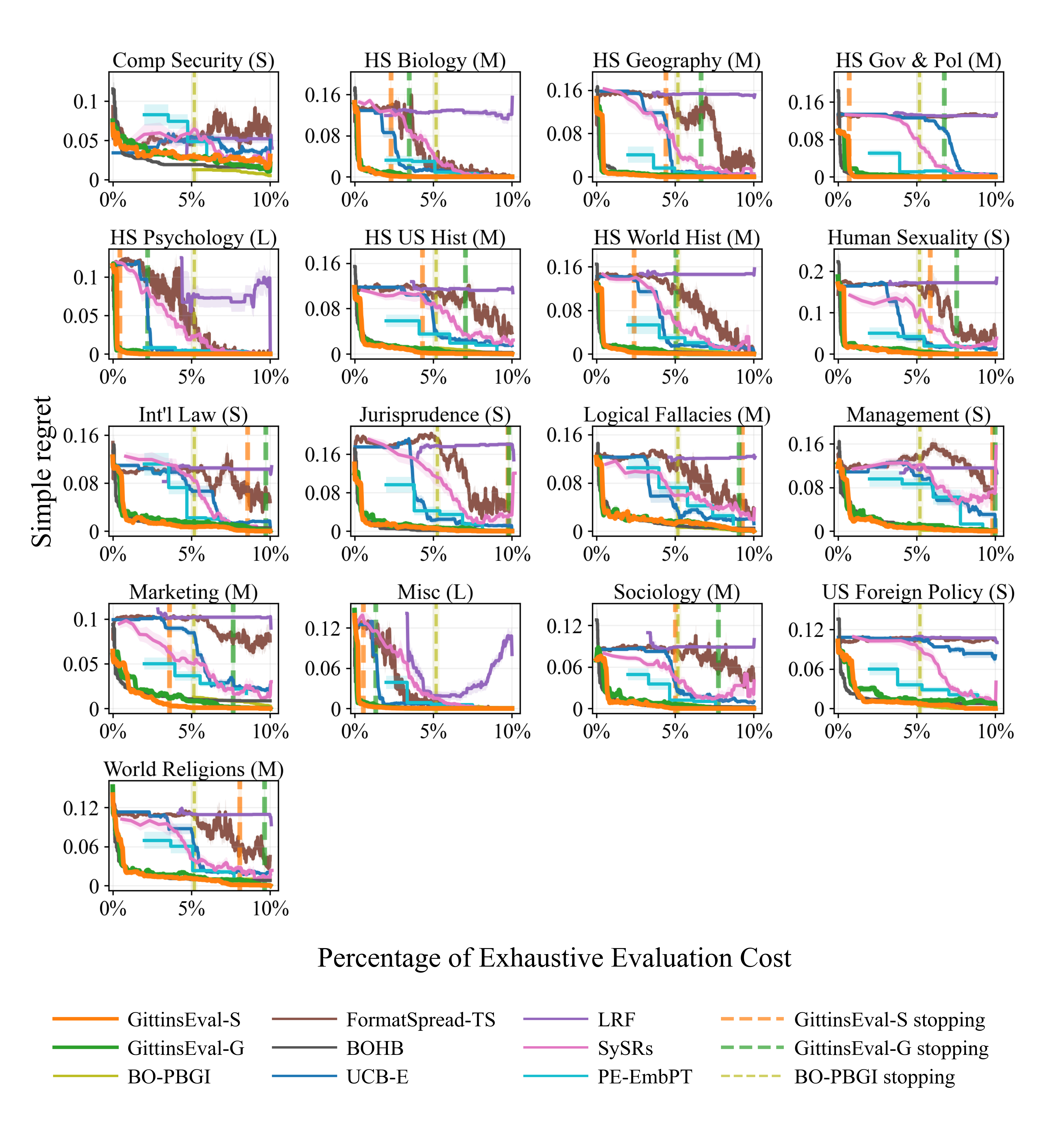}
\else\fi
\caption{
Per-subject MMLU cost-aware results for easy subjects (high-prior bucket). Simple regret is plotted against cumulative evaluation cost as a percentage of exhaustive evaluation. Labels S, M, and L denote matrix size. GittinsEval, UCB-E, and FormatSpread-TS use $B=2,4,8$ for these sizes, respectively; LRF uses $B=32$, and BOHB uses example-subsampling fidelities. Bands denote $\pm1$ standard error, and dashed lines with faint bands mark mean stopping times with $\pm1$ standard error. }
\label{fig:mmlu-easy-aware-bo5pct}
\end{figure}

\begin{figure}[p]
\centering
\ifdefined\iclrsysrs
\includegraphics[width=\textwidth]{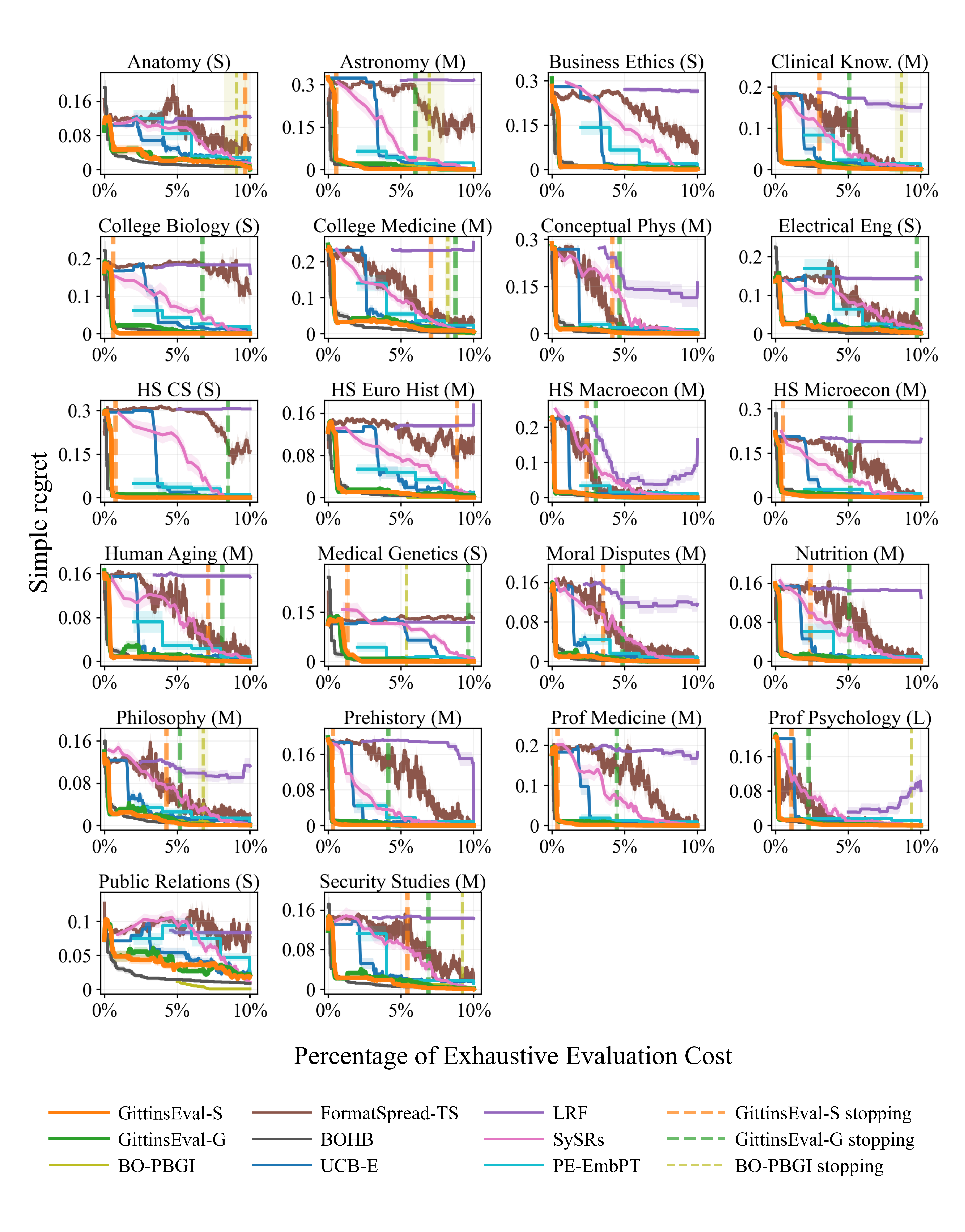}
\else\fi
\caption{
Per-subject MMLU unit-cost results for medium-difficulty subjects (medium-prior bucket). Simple regret is plotted against cumulative evaluation cost as a percentage of exhaustive evaluation. Labels S, M, and L denote matrix size. GittinsEval, UCB-E, and FormatSpread-TS use $B=2,4,8$ for these sizes, respectively; LRF uses $B=32$, and BOHB uses example-subsampling fidelities. Bands denote $\pm1$ standard error, and dashed lines with faint bands mark mean stopping times with $\pm1$ standard error. }
\label{fig:mmlu-medium-unit-bo5pct}
\end{figure}

\begin{figure}[p]
\centering
\ifdefined\iclrsysrs
\includegraphics[width=\textwidth]{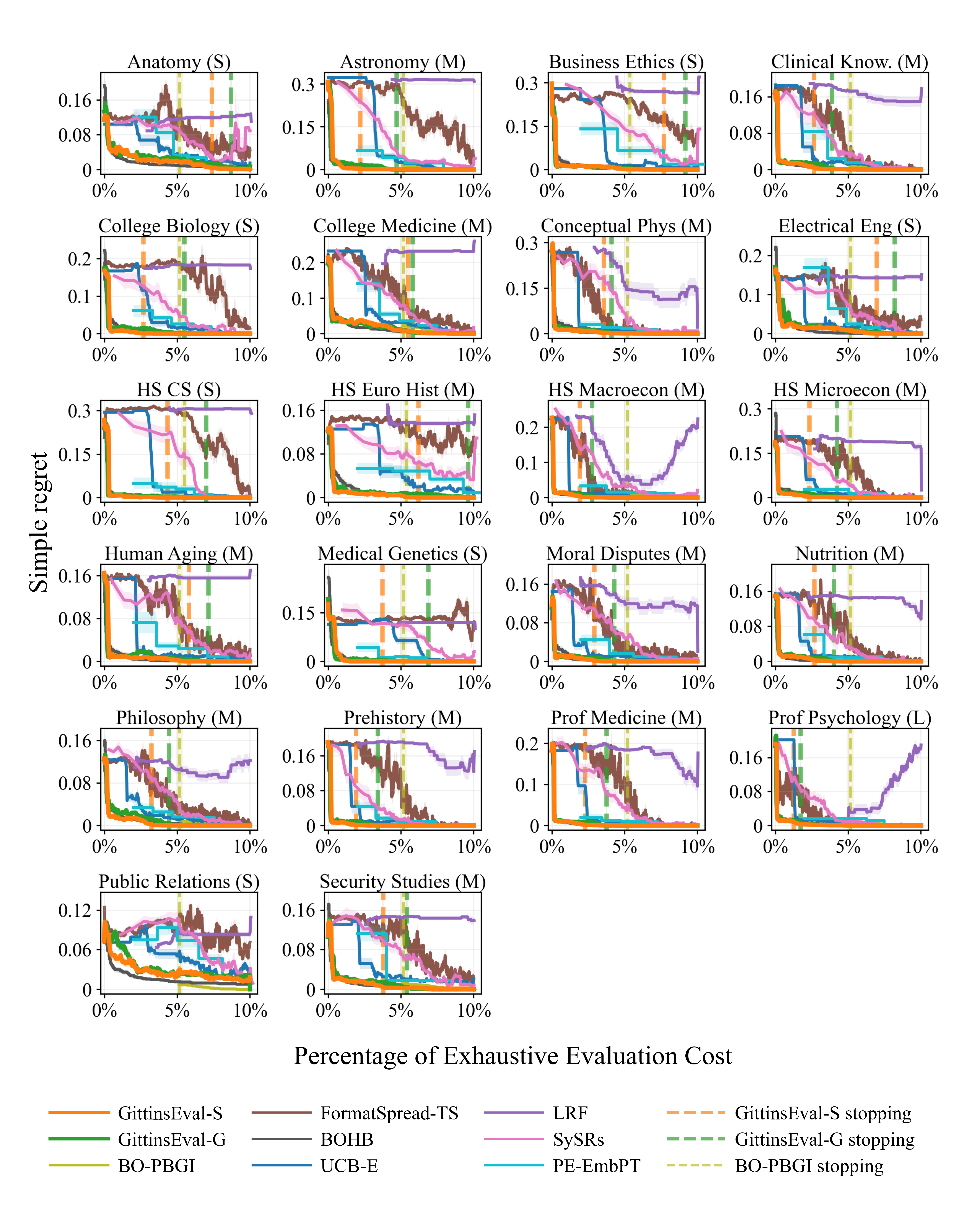}
\else\fi
\caption{
Per-subject MMLU cost-aware results for medium-difficulty subjects (medium-prior bucket). Simple regret is plotted against cumulative evaluation cost as a percentage of exhaustive evaluation. Labels S, M, and L denote matrix size. GittinsEval, UCB-E, and FormatSpread-TS use $B=2,4,8$ for these sizes, respectively; LRF uses $B=32$, and BOHB uses example-subsampling fidelities. Bands denote $\pm1$ standard error, and dashed lines with faint bands mark mean stopping times with $\pm1$ standard error. }
\label{fig:mmlu-medium-aware-bo5pct}
\end{figure}

\begin{figure}[p]
\centering
\ifdefined\iclrsysrs
\includegraphics[width=\textwidth]{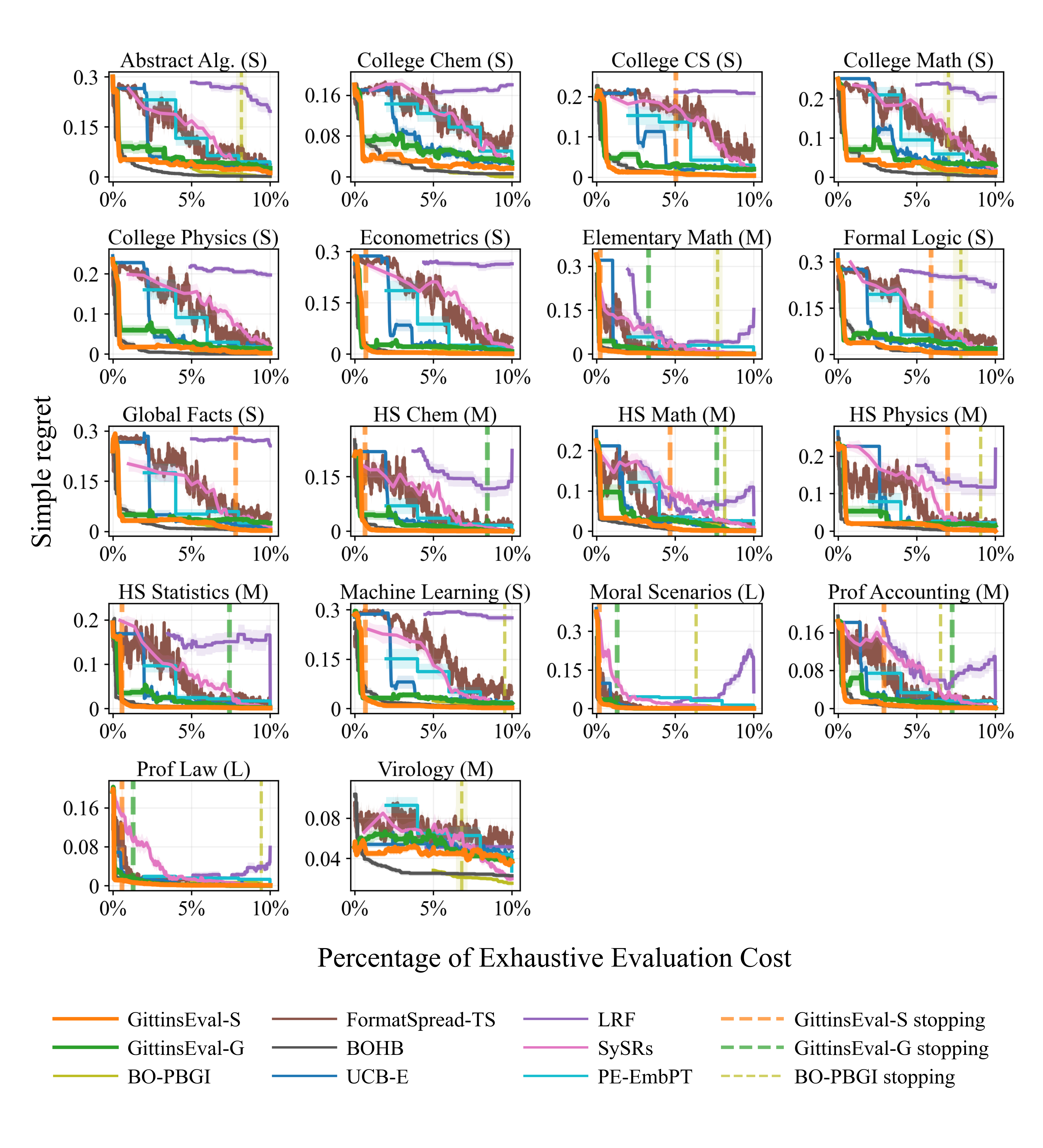}
\else\fi
\caption{
Per-subject MMLU unit-cost results for hard subjects (low-prior bucket). Simple regret is plotted against cumulative evaluation cost as a percentage of exhaustive evaluation. Labels S, M, and L denote matrix size. GittinsEval, UCB-E, and FormatSpread-TS use $B=2,4,8$ for these sizes, respectively; LRF uses $B=32$, and BOHB uses example-subsampling fidelities. Bands denote $\pm1$ standard error, and dashed lines with faint bands mark mean stopping times with $\pm1$ standard error. }
\label{fig:mmlu-hard-unit-bo5pct}
\end{figure}

\begin{figure}[p]
\centering
\ifdefined\iclrsysrs
\includegraphics[width=\textwidth]{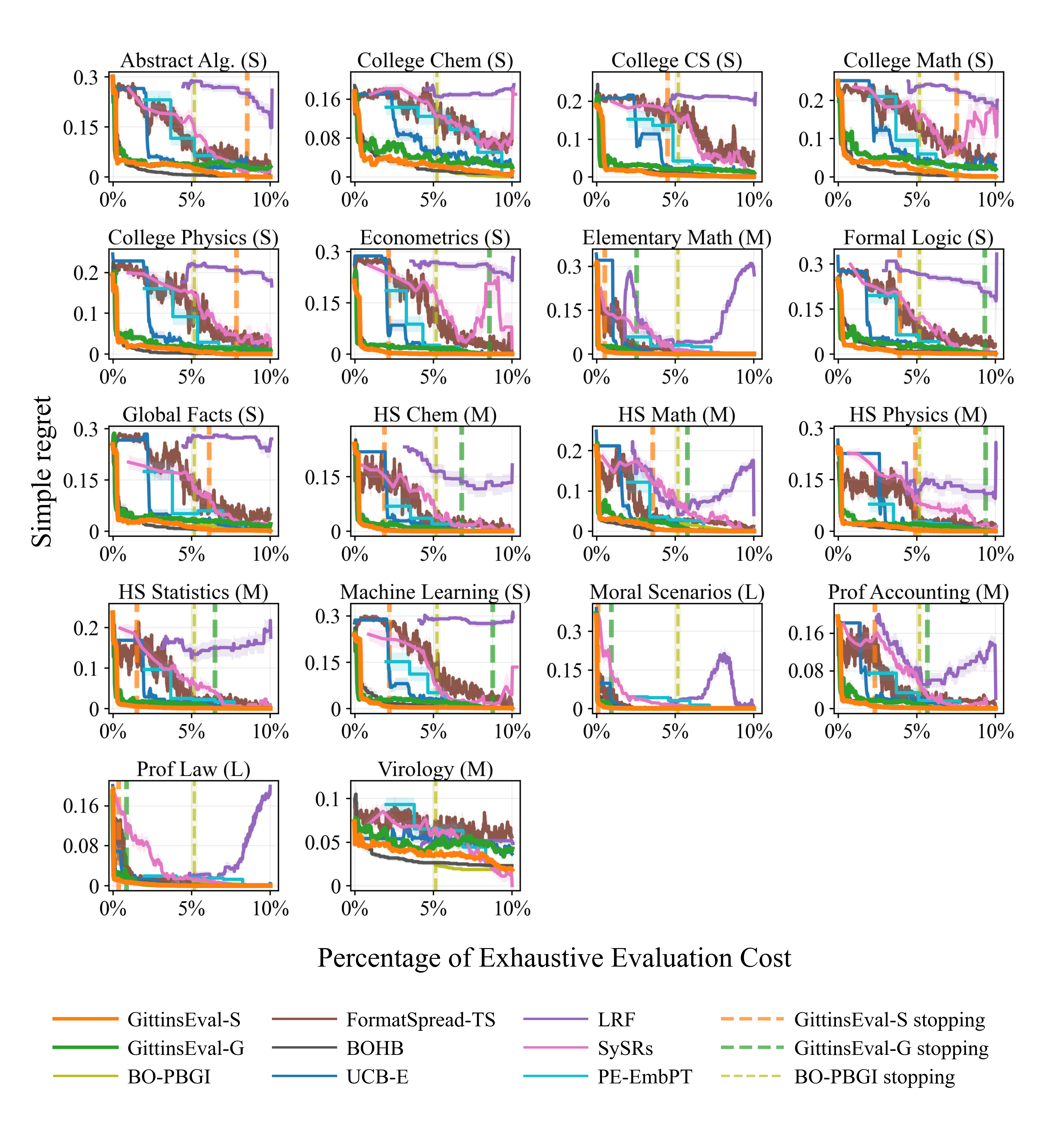}
\else\fi
\caption{
Per-subject MMLU cost-aware results for hard subjects (low-prior bucket). Simple regret is plotted against cumulative evaluation cost as a percentage of exhaustive evaluation. Labels S, M, and L denote matrix size. GittinsEval, UCB-E, and FormatSpread-TS use $B=2,4,8$ for these sizes, respectively; LRF uses $B=32$, and BOHB uses example-subsampling fidelities. Bands denote $\pm1$ standard error, and dashed lines with faint bands mark mean stopping times with $\pm1$ standard error. }
\label{fig:mmlu-hard-aware-bo5pct}
\end{figure}

\FloatBarrier

%% file: BanditGittinsEval.bib
@inproceedings{zhou2024speeding,
  title={On speeding up language model evaluation},
  author={Zhou, Jin Peng and Belardi, Christian K and Wu, Ruihan and Zhang, Travis and Gomes, Carla P and Sun, Wen and Weinberger, Kilian Q},
  booktitle={The Thirteenth International Conference on Learning Representations},
  year={2025},
  url={https://openreview.net/forum?id=3cvwO5DBZn}
}

@inproceedings{habba2025dove,
  title={{DOVE}: A Large-Scale Multi-Dimensional Predictions Dataset Towards Meaningful {LLM} Evaluation},
  author={Habba, Eliya and Arviv, Ofir and Itzhak, Itay and Perlitz, Yotam and Bandel, Elron and Choshen, Leshem and Shmueli-Scheuer, Michal and Stanovsky, Gabriel},
  booktitle={Findings of the Association for Computational Linguistics: ACL 2025},
  pages={11744--11763},
  year={2025},
  publisher={Association for Computational Linguistics},
  doi={10.18653/v1/2025.findings-acl.611},
  url={https://aclanthology.org/2025.findings-acl.611/}
}

@incollection{scully2025gittins,
  title={The Gittins Index: A Design Principle for Decision Making under Uncertainty},
  author={Scully, Ziv and Terenin, Alexander},
  booktitle={Tutorials in Operations Research: Advances in Analytics and Operations Research: Improving Decisions to Secure the Future},
  pages={28--70},
  year={2025},
  publisher={INFORMS}
}

@article{gittins1979bandit,
  title={Bandit Processes and Dynamic Allocation Indices},
  author={Gittins, John C},
  journal={Journal of the Royal Statistical Society: Series B (Methodological)},
  volume={41},
  number={2},
  pages={148--177},
  year={1979},
  doi={10.1111/j.2517-6161.1979.tb01068.x}
}

@book{gittins2011multi,
  title={Multi-armed bandit allocation indices},
  author={Gittins, John and Glazebrook, Kevin and Weber, Richard},
  year={2011},
  publisher={John Wiley \& Sons}
}

@article{cobbe2021training,
  title={Training verifiers to solve math word problems},
  author={Cobbe, Karl and Kosaraju, Vineet and Bavarian, Mohammad and Chen, Mark and Jun, Heewoo and Kaiser, Lukasz and Plappert, Matthias and Tworek, Jerry and Hilton, Jacob and Nakano, Reiichiro and others},
  journal={arXiv preprint arXiv:2110.14168},
  year={2021}
}

@inproceedings{bisk2020piqa,
  title={Piqa: Reasoning about physical commonsense in natural language},
  author={Bisk, Yonatan and Zellers, Rowan and Gao, Jianfeng and Choi, Yejin and others},
  booktitle={Proceedings of the AAAI conference on artificial intelligence},
  volume={34},
  number={05},
  pages={7432--7439},
  year={2020}
}

@inproceedings{audibert2010best,
  title={Best arm identification in multi-armed bandits},
  author={Audibert, Jean-Yves and Bubeck, S{\'e}bastien},
  booktitle={COLT-23th Conference on learning theory-2010},
  pages={13--p},
  year={2010}
}

@article{hendrycks2020measuring,
  title={Measuring massive multitask language understanding},
  author={Hendrycks, Dan and Burns, Collin and Basart, Steven and Zou, Andy and Mazeika, Mantas and Song, Dawn and Steinhardt, Jacob},
  journal={arXiv preprint arXiv:2009.03300},
  year={2020}
}

@article{polo2024efficient,
  title={Efficient multi-prompt evaluation of llms},
  author={Polo, Felipe M and Xu, Ronald and Weber, Lucas and Silva, M{\'\i}rian and Bhardwaj, Onkar and Choshen, Leshem and de Oliveira, Allysson F and Sun, Yuekai and Yurochkin, Mikhail},
  journal={Advances in Neural Information Processing Systems},
  volume={37},
  pages={22483--22512},
  year={2024}
}

@inproceedings{sclar2024quantifying,
  title={Quantifying Language Models' Sensitivity to Spurious Features in Prompt Design or: How I learned to start worrying about prompt formatting},
  author={Sclar, Melanie and Choi, Yejin and Tsvetkov, Yulia and Suhr, Alane},
  booktitle={The Twelfth International Conference on Learning Representations},
  year={2024},
  url={https://openreview.net/forum?id=RIu5lyNXjT}
}

@article{xie2024cost,
  title={Cost-aware bayesian optimization via the pandora's box gittins index},
  author={Xie, Qian and Astudillo, Raul and Frazier, Peter I and Scully, Ziv and Terenin, Alexander},
  journal={Advances in Neural Information Processing Systems},
  volume={37},
  pages={115523--115562},
  year={2024}
}

@article{shi2024efficientbestarm,
  title={Efficient prompt optimization through the lens of best arm identification},
  author={Shi, Chengshuai and Yang, Kun and Chen, Zihan and Li, Jundong and Yang, Jing and Shen, Cong},
  journal={Advances in Neural Information Processing Systems},
  volume={37},
  pages={99646--99685},
  year={2024}
}

@article{dumitriu2003playing,
  title={On playing golf with two balls},
  author={Dumitriu, Ioana and Tetali, Prasad and Winkler, Peter},
  journal={SIAM Journal on Discrete Mathematics},
  volume={16},
  number={4},
  pages={604--615},
  year={2003},
  publisher={SIAM}
}

@article{xie2025cost,
  title={Cost-aware stopping for bayesian optimization},
  author={Xie, Qian and Cai, Linda and Terenin, Alexander and Frazier, Peter I and Scully, Ziv},
  journal={arXiv preprint arXiv:2507.12453},
  year={2025}
}

@article{snoek2012practical,
	author = {Snoek, Jasper and Larochelle, Hugo and Adams, Ryan P},
	journal = {Advances in Neural Information Processing Systems},
	title = {Practical Bayesian optimization of machine learning algorithms},
	year = {2012}}

@misc{li2023alpacaeval,
  title={Alpacaeval: An automatic evaluator of instruction-following models},
  author={Li, Xuechen and Zhang, Tianyi and Dubois, Yann and Taori, Rohan and Gulrajani, Ishaan and Guestrin, Carlos and Liang, Percy and Hashimoto, Tatsunori B},
  year={2023}
}

@article{lyu2026cutting,
  title={Cutting LLM Evaluation Costs with SySRs: A Bandit Algorithm that Provably Exploits Model Similarity},
  author={Lyu, Zifan and Nejma, Chahine and Wegel, Tobias and Yang, Fanny and Dorner, Florian E},
  journal={arXiv preprint arXiv:2606.07726},
  year={2026}
}

@phdthesis{xie2026gittins,
  author={Xie, Qian},
  title={The {Gittins Index} Design Principle for Cost-Aware {Bayesian} Decision-Making under Uncertainty},
  school={Cornell University},
  year={2026},
  month=aug,
  type={{Ph.D.} dissertation},
  number={32792910},
  url={https://www.proquest.com/docview/3385501903},
  note={ProQuest Dissertations \& Theses, ProQuest document ID 3385501903}
}

@inproceedings{klein2017fast,
  title={Fast bayesian optimization of machine learning hyperparameters on large datasets},
  author={Klein, Aaron and Falkner, Stefan and Bartels, Simon and Hennig, Philipp and Hutter, Frank},
  booktitle={Artificial intelligence and statistics},
  pages={528--536},
  year={2017},
  organization={PMLR}
}

@inproceedings{kandasamy2017multi,
  title={Multi-fidelity bayesian optimisation with continuous approximations},
  author={Kandasamy, Kirthevasan and Dasarathy, Gautam and Schneider, Jeff and P{\'o}czos, Barnab{\'a}s},
  booktitle={International conference on machine learning},
  pages={1799--1808},
  year={2017},
  organization={PMLR}
}

@inproceedings{wu2020practical,
  title={Practical multi-fidelity Bayesian optimization for hyperparameter tuning},
  author={Wu, Jian and Toscano-Palmerin, Saul and Frazier, Peter I and Wilson, Andrew Gordon},
  booktitle={Uncertainty in Artificial Intelligence},
  pages={788--798},
  year={2020},
  organization={PMLR}
}

@inproceedings{falkner2018bohb,
  title={{BOHB}: Robust and Efficient Hyperparameter Optimization at Scale},
  author={Falkner, Stefan and Klein, Aaron and Hutter, Frank},
  booktitle={Proceedings of the 35th International Conference on Machine Learning},
  series={Proceedings of Machine Learning Research},
  volume={80},
  pages={1437--1446},
  year={2018},
  publisher={PMLR},
  url={https://proceedings.mlr.press/v80/falkner18a.html}
}

@article{balandat2020botorch,
  title={{BoTorch}: A framework for efficient {Monte-Carlo Bayesian} optimization},
  author={Balandat, Maximilian and Karrer, Brian and Jiang, Daniel and Daulton, Samuel and Letham, Ben and Wilson, Andrew G and Bakshy, Eytan},
  journal={Advances in Neural Information Processing Systems},
  volume={33},
  pages={21524--21538},
  year={2020}
}

@article{tolochinsky2026valid,
  title={Valid Best-Model Identification for LLM Evaluation via Low-Rank Factorization},
  author={Tolochinsky, Elad and Tenzer, Yaniv and Romano, Yaniv},
  journal={arXiv preprint arXiv:2605.10405},
  year={2026}
}

@article{ao2026bestarm,
  title={Best Arm Identification with LLM Judges and Limited Human Audits},
  author={Ao, Ruicheng and Chen, Hongyu and Gao, Siyang and Li, Hanwei and Simchi-Levi, David},
  journal={arXiv preprint arXiv:2601.21471},
  year={2026}
}

@article{xue2026cost,
  title={Cost-Aware Multi-Objective Bandits: Theory and Application to Budgeted {LLM} Configuration Evaluation},
  author={Xue, Bo and Hong, Zhi and Li, Jiayi and Wan, Yuanyu and Cheng, Ji and Qiu, Shuang},
  journal={arXiv preprint arXiv:2608.04333},
  year={2026}
}

@inproceedings{harrasse2026debate,
  title={Debate, Deliberate, Decide ({D3}): A Cost-Aware Adversarial Framework for Reliable and Interpretable {LLM} Evaluation},
  author={Harrasse, Abir and Bandi, Chaithanya and Bandi, Hari},
  booktitle={Proceedings of the 19th Conference of the European Chapter of the Association for Computational Linguistics (Volume 1: Long Papers)},
  pages={8376--8392},
  year={2026},
  month={mar},
  address={Rabat, Morocco},
  publisher={Association for Computational Linguistics},
  doi={10.18653/v1/2026.eacl-long.392}
}
